%% file: paper.tex
\documentclass{article}
\pdfoutput=1

\usepackage{PRIMEarxiv}

\usepackage[utf8]{inputenc} 
\usepackage[T1]{fontenc}    
\usepackage{hyperref}       
\usepackage{url}            
\usepackage{booktabs}       
\usepackage{amsfonts}       
\usepackage{nicefrac}       
\usepackage{microtype}      
\usepackage{lipsum}
\usepackage[pdftex]{graphicx}
\graphicspath{{img/}}     

\title{Learning Symmetric Rules with SATNet}

\author{
  Sangho Lim$^*$ \\
  School of Computing \\
  KAIST \\
  Daejeon, South Korea \\
  \texttt{lim.sang@kaist.ac.kr} \\
  \And
  Eun-Gyeol Oh$^*$ \\
  Graduate School of Information Security \\
  KAIST \\
  Daejeon, South Korea \\
  \texttt{eun-gyeol.oh@kaist.ac.kr} \\
  \And
  Hongseok Yang \\
  School of Computing and Kim Jaechul Graduate School of AI, KAIST\\
  Discrete Mathematics Group, Institute for Basic Science (IBS)\\
  Daejeon, South Korea \\
  \texttt{hongseok.yang@kaist.ac.kr} \\
}

\input{packages}
\input{macros}

\begin{document}
\maketitle
\def\thefootnote{*}\footnotetext{These authors contributed equally to this work.}\def\thefootnote{\arabic{footnote}}

\begin{abstract}
SATNet is a differentiable constraint solver with a custom backpropagation algorithm, which can be used as a layer in a deep-learning system. It is a promising proposal for bridging deep learning and logical reasoning. In fact, SATNet has been successfully applied to learn, among others, the rules of a complex logical puzzle, such as Sudoku, just from input and output pairs where inputs are given as images. In this paper, we show how to improve the learning of SATNet by exploiting symmetries in the target rules of a given but unknown logical puzzle or more generally a logical formula. We present SymSATNet, a variant of SATNet that translates the given symmetries of the target rules to a condition on the parameters of SATNet and requires that the parameters should have a particular parametric form that guarantees the condition. The requirement dramatically reduces the number of parameters to learn for the rules with enough symmetries, and makes the parameter learning of SymSATNet much easier than that of SATNet. We also describe a technique for automatically discovering symmetries of the target rules from examples. Our experiments with Sudoku and Rubik's cube show the substantial improvement of SymSATNet over the baseline SATNet.
\end{abstract}

\input{intro}
\input{background}
\input{result}
\input{algo}
\input{experiment}
\input{conclusion}

\section*{Acknowledgements} This work was supported by the Engineering Research Center Program through the National Research Foundation of Korea (NRF) funded by the Korean Government MSIT (NRF-2018R1A5A1059921) and also by the Institute for Basic Science (IBS-R029-C1).

\bibliography{bib}  
\clearpage

\appendix
\input{appendix-visualisation}
\input{appendix-notation}
\input{appendix-basis}
\input{appendix-symsatnet}
\input{appendix-symfind}
\input{appendix-symfind-complexity}
\input{appendix-sudoku-cube}
\input{appendix-efficiency}
\input{appendix-hyperparameter}
\input{appendix-gpu}
\input{appendix-ablation}
\input{appendix-emergence}
\input{appendix-symsatnet-300aux}
\input{appendix-losses}

\end{document}

%% file: packages.tex
\usepackage{natbib}
\usepackage{mathtools} 
\usepackage{booktabs} 
\usepackage{tikz} 
\usepackage{wrapfig}

\usepackage{microtype}
\usepackage{graphicx}
\usepackage{subcaption}

\usepackage{hyperref}

\usepackage{algorithm}
\usepackage{algorithmic}

\usepackage{amsmath}
\usepackage{amssymb}
\usepackage{amsthm}
\usepackage{dsfont}

\usepackage{makecell}
\usepackage{paralist}

%% file: macros.tex
\allowdisplaybreaks

\theoremstyle{plain}
\newtheorem{theorem}{Theorem}[section]
\newtheorem{claim}{Claim}[section]

\newtheorem{lemma}[theorem]{Lemma}

\theoremstyle{definition}
\newtheorem{definition}[theorem]{Definition}
\newtheorem{assumption}[theorem]{Assumption}
\theoremstyle{remark}

\DeclareMathOperator*{\argmin}{argmin}

\DeclareMathOperator*{\group}{grp}
\DeclareMathOperator*{\ind}{\mathds{1}}
\DeclareMathOperator*{\image}{im}
\DeclareMathOperator*{\ovec}{vec}
\DeclareMathOperator*{\prj}{prj}
\DeclareMathOperator*{\rank}{rank}

\DeclareMathOperator*{\trace}{tr}

\DeclareMathOperator*{\wreath}{wr}

\newcommand*{\id}{\mathrm{id}}
\newcommand*{\kronsplit}{\textsc{ProdFind}}
\newcommand*{\symfind}{\textsc{SymFind}}
\newcommand*{\sumsplit}{\textsc{SumFind}}
\newcommand*{\idg}{\mathcal{I}}
\newcommand*{\cyclic}{\mathbb{Z}}
\newcommand*{\perm}{\mathcal{S}}

\newcommand*{\N}{\mathbb{N}}
\newcommand*{\R}{\mathbb{R}}

\newcommand*{\cA}{\mathcal{A}}
\newcommand*{\cB}{\mathcal{B}}

\newcommand*{\cE}{\mathcal{E}}
\newcommand*{\cF}{\mathcal{F}}
\newcommand*{\cI}{\mathcal{I}}

\newcommand*{\cN}{\mathcal{N}}
\newcommand*{\cO}{\mathcal{O}}

\newcommand*{\cR}{\mathcal{R}}
\newcommand*{\cS}{\mathcal{S}}
\newcommand*{\cV}{\mathcal{V}}
\newcommand*{\cX}{\mathcal{X}}
\newcommand*{\cY}{\mathcal{Y}}

%% file: intro.tex
\section{Introduction}
\label{sec:intro}

Bringing the ability of reasoning to the deep-learning systems has been the aim of a large amount of recent research efforts~\citep{YangYC17,EvansG18,CingilliogluR19,WangDWK19,Topan21}.
One notable outcome of these endeavours is SATNet~\citep{WangDWK19},
a differentiable constraint solver with an efficient custom backpropagation algorithm.
SATNet can be used as a component of a deep-learning system and make the system capable of learning and reasoning about sophisticated logical rules.
Its potential has been demonstrated successfully with the tasks of learning the rules of complex logical puzzles, such as Sudoku,
just from input-output examples where the inputs are given as images.

We show how to improve the rule (or constraint) learning of SATNet, when the target rules have permutation symmetries.
By having symmetries, we mean the solutions of the rules are closed under the permutations of those symmetries.
For example, in Sudoku, if a completed $9\times 9$ Sudoku board is a solution,
permuting the numbers $1$ to $9$ in the board, the first three rows, or the last three columns always gives rise to another solution.
Thus, these permutations are symmetries of Sudoku.

Our improvement is a variant of SATNet, called SymSATNet, which abbreviates symmetry-aware SATNet.
SymSATNet assumes that some symmetries of the target rules are given a priori although the rules themselves are unknown.
It then translates these symmetries into a condition on the parameter matrix $C \in \R^{n\times n}$ of SATNet (or our minor generalisation),
and requires that the parameters have a particular parametric form that guarantees the condition. Concretely,
the translated condition says that the matrix $C$ regarded as a linear map should be equivariant with respect to the group $G$ determined by the given symmetries,
and the requirement is that $C$ should be a linear combination of elements in a basis for the space of $G$-equivariant symmetric matrices.
The coefficients of this linear combination are the parameters of our SymSATNet,
and their number is often substantially smaller than that of the parameters of SATNet.\footnote{SATNet assumes that $C$ is of the form $S^TS$ for some $S \in \R^{m \times n}$ for $m < n$,
so that the number of parameters is $mn < n^2$. But often it is still substantially larger than the number of parameters of SymSATNet.}
For Sudoku, the former is $18$, while the latter is $729^2$ or $k\cdot729$ for some $k \in \N$ at best.
The reduced number of parameters implies that SymSATNet has to tackle an easier learning problem than SATNet, and has a potential to learn faster and generalise better than SATNet.

Who provides symmetries for SymSATNet? The default answer is domain experts, but
a better alternative is possible.
We present an automatic algorithm to discover symmetries. Our algorithm is based on empirical observation that
symmetries emerge in the parameter matrix $C$ of SATNet in the early phase of training, as clusters of similar entries.
Our algorithm takes a snapshot of
$C$ at some training epoch of SATNet, and finds a group $G$ such that
(i) specific entries of $C$ share similar values by $G$-equivariance condition,
and (ii) the number of SymSATNet parameters under $G$ is minimised.

We empirically evaluate SymSATNet and our symmetry-discovering algorithm with Sudoku and a problem related to Rubik's cube.
For both problems, our algorithm discovered nontrivial symmetries,
and SymSATNet with manually specified or automatically found symmetries outperformed the baseline SATNet in learning the rules,
in terms of both efficiency and generalisation.

\paragraph{Related work}
There have been multiple studies on discovering symmetries present in conjunctive normal form (CNF) formulas in order to reduce the search space of satisfiability (SAT) solvers.
Crawford~\citep{Crawford92} proved that the symmetry-detection problem is equivalent to the graph isomorphism problem, and showed how to reduce the complexity of pigeonhole problems using symmetries.
Crawford et al.~\citep{Crawford96} proposed symmetry-breaking predicates (SBPs), and Aloul et al.~\citep{Aloul03, Aloul06} developed SBPs with more efficient constructions.
For automatic symmetry detection, Darga et al.~\citep{Darga04} presented a method that improves the partition refinement procedure introduced by McKay~\citep{Mckay76, Mckay81},
and Darga et al.~\citep{Darga08} proposed an algorithm that achieve efficiency by exploiting the sparsity of symmetries.
In contrast to these global and static methods, Benhamou et al.~\citep{Benhamou10} and Devriendt et al.~\citep{Devriendt17} handled local symmetries that dynamically arise during search.
The use of symmetries also appears in NeuroSAT~\citep{Selsam19}, which learns how to solve SAT problems from examples.
NeuroSAT solves a given SAT formula by message passing over a graph constructed from the formula, and in so doing,
its learnable solver can exploit symmetries in the formulas.
All of these techniques use symmetries to help solve given formulas, whereas our approach uses symmetries to help learn such formulas.
Another difference is that those techniques find hard symmetries of given formulas, whereas our approach discovers soft or approximate symmetries in a given SATNet parameter matrix.
In the context of deep learning, Basu et al.~\citep{Basu21} and Dehmamy et al.~\cite{Dehmamy21} described algorithms that find and exploit symmetries via group decompositions and Lie algebra convolutions.
But these techniques are not designed to find symmetries in logical formulas.

Our work is related to the studies on learning logical rules from examples using gradients.
Yang et al.~\citep{YangYC17} proposed neural logic programming, an end-to-end differentiable system which learns first-order logical rules,
Evans and Grefenstette~\citep{EvansG18} proposed a differentiable inductive logic programming system which is robust to noise of training data,
and Cingillioglu and Russo~\citep{CingilliogluR19} introduced an RNN-based model to learn logical reasoning tasks end-to-end. Want et al.~\citep{Wang18, WangDWK19} presented SATNet using the mixing method, and Topan et al.~\cite{Topan21} further improved SATNet by solving the symbol grounding problem, a key challenge of SATNet. 
Our work extends these lines of work by proposing how to discover and exploit symmetries from examples when learning logical rules with SATNet.

%% file: background.tex
\section{Background}
\label{sec:background}

We review SATNet, the formalisation of symmetries using groups, and equivariant 
maps. For a natural number $n$, let $[n] = \{1, \ldots, n\}$, and for a matrix $M$, let $M_{i,j}$ be the $(i, j)$-th entry of $M$.

\subsection{SATNet}
\label{sec:satnet}
A good starting point for learning about SATNet is to look at its origin, the mixing method~\citep{Wang18}, which is an efficient algorithm for solving semidefinite programming problems with diagonal constraints. Let $n,k \in \N$ and $C$ be a real-valued symmetric matrix in $\R^{n\times n}$. The mixing method aims at solving the following optimisation problem:
\begin{equation}
\label{eqn:mixing-method-objective}
    \begin{aligned}
        \argmin_{V \in \R^{k \times n}}\, \langle C, V^T V \rangle \quad \text{subject to}\ \Vert v_i \Vert = 1 \ \text{for}\ i \in [n]
    \end{aligned}
\end{equation}
where $v_i$ is the $i$-th column of the matrix $V$, and $\Vert v_i \Vert$ is $L_2$ norm of $v_i$.
The mixing method solves \eqref{eqn:mixing-method-objective} by coordinate descent, where each column $v_i$ of $V$ is repeatedly updated as follows:
$g_i \leftarrow \sum_{\substack{j \in [n], j\neq i}} C_{i,j} v_j$ and $v_i \leftarrow - \frac{g_i}{\Vert g_i \Vert}$.
This always finds a fixed point of the equations.
In fact, it is shown that almost surely this fixed point attains a global optimum of the optimisation problem.

An example of the above optimisation problem most relevant to us is a continuous relaxation of MAXSAT.
MAXSAT is a problem of finding truth assignments to $n$ boolean variables $b_1,\ldots, b_n$.
It assumes $m$ clauses of those variables, $F_1,\ldots,F_m$, where $F_\ell$ is the disjunction of some variables with or without negation:
$F_\ell = b_{i_1} \vee \ldots \vee b_{i_p} \vee \neg b_{i_{p +1}} \vee \ldots \neg b_{i_{p+q}}$.
Then, MAXSAT asks for a truth assignment on the variables that maximises the number of true clauses $F_i$ under the assignment.
The rules of many problems, including Sudoku, can be expressed as an instance of MAXSAT. 

To apply the mixing method to MAXSAT, we introduce relaxed vectors $v_1,\ldots,v_n \in \R^k$ that encode the boolean variables,
and construct the matrix $S \in \R^{m \times n}$ that encodes the $m$ clauses of MAXSAT:
the $(\ell,j)$-th entry of $S$ has $1$ if
$F_\ell$ contains $b_j$;
and $-1$ if $F_\ell$ includes $\neg b_j$; and $0$ if neither of these cases holds.
Then, the problem in \eqref{eqn:mixing-method-objective} is formed with $C = -S^T S$, and solved by the mixing method.

SATNet is a variant of the mixing method where some of the columns of $V$ are fixed and the optimisation is over the rest of the columns.\footnote{The original SATNet assumes that $C$ has the form $S^T S$ for some $m \times n$ matrix $S$.
We drop this assumption and adjust the forward and backward computations of SATNet accordingly.
The main steps of derivations of the formulas for the forward and backward computations are from the work on SATNet~\citep{WangDWK19}.}
Concretely, it assumes that the column indices in $[n]$ are split into two disjoint sets,
$\cI$ and $\cO$ (i.e., $\cI \cup \cO = [n]$ and $\cI \cap \cO = \emptyset$).
The inputs of SATNet are the columns $v_i$ of $V$ with $i \in \cI$, and the outputs are the rest of the columns (i.e., the $v_o$'s with $o \in \cO$).
The symmetric matrix $C$ is the parameter of SATNet.
Given the input vectors, SATNet repeatedly executes the coordinate descent updates on each output column, until it converges.

One important feature of SATNet is that it has a custom algorithm for backpropagation.
Let $V_\cI$ be the matrix of the input columns to SATNet, and $V_\cO$ be that of the output columns computed by SATNet on the input $V_\cI$ under the parameter $C$.
Assume that $l$ is a loss of the output $V_\cO$.
In this context, SATNet provides formulas and algorithms for computing the derivatives $\partial l/\partial V_\cI$ and $\partial l/\partial C$.

We recall the formulas for the derivatives. Let $o_1 < o_2 < \ldots < o_{|\cO|}$ be the sorting of the indices in $\cO$.
Assume that SATNet was run until convergence, so that the output columns in $V_\cO$ are the fixed point of the coordinate descent updates:
for all $o \in \cO$, $g_o = \sum_{\substack{j \in [n],  j\neq o}} C_{o,j} v_j$ and $v_o = - \frac{g_o}{\Vert g_o \Vert}$.
The formulas for 
$\partial l/\partial V_\cI$ and $\partial l/\partial C$ at $(V_\cI,V_\cO,C)$ are defined in terms of the next quantities:
\begin{align*}
C', D' \in \R^{|\cO| \times |\cO|}, \qquad
P \in \R^{|\cO|k \times |\cO|k}, \qquad
U \in \R^{|\cO|k \times 1}, \qquad
W \in \R^{|\cO|k \times n^2}.
\end{align*}
They have the following definitions: for $i,j \in [|\cO|]$, $p,q \in [k]$, and $r,s \in [n]$,
\begin{align*}
\begin{aligned}
(C')_{i,j} &=
\begin{cases}
0 & \text{if } i \,{=}\, j
\\
C_{o_i,o_j}  & \text{if } i\,{\neq}\, j,
\end{cases}
\\
(D')_{i,j} &= 
\begin{cases}
\Vert g_{o_i}\Vert & \text{if } i \,{=}\, j
\\
0 & \text{if } i\,{\neq}\, j,
\end{cases}
\\
P_{\substack{(i-1)k + p, \\ (j-1)k+q}} &= 
\begin{cases}
(I_k - v_{o_i}v_{o_i}^T)_{p,q} & \text{if } i = j
\\
0 & \text{if } i \neq j,
\end{cases}
\end{aligned}
\qquad
\begin{aligned}
& U = \left(P((D' + C') \otimes I_k)\right)^\dagger \left(\frac{\partial l}{\partial \ovec(V_\cO)}\right)^T,
\\
& W_{\substack{(i-1)k+p, \\ (r-1)n+s}} = 
\begin{cases}
0 & \text{if } r = o_i \text{ and } s = o_i
\\
V_{p,s}  & \text{if } r = o_i \text{ and } s \neq o_i
\\
V_{p,s} & \text{if } r \neq o_i \text{ and } s = o_i
\\
0 & \text{if } r \neq o_i \text{ and } s \neq o_i.
\end{cases}
\end{aligned}
\end{align*}
Here $\otimes$ is the Kronecker product, $\_^\dagger$ is Moore-Penrose inverse (also known as pseudo inverse),
and $\ovec(V_\cO)$ is the vector obtained by stacking the columns of $V_\cO$.
Let $C_{\cO,\cI} \in \R^{|\cO|\times |\cI|}$ be obtained by restricting $C$ to the indices $(o,i)$ with $o \in \cO$ and $i \in \cI$. Then,
\begin{align}
\label{eqn:SATNet-derivative}
\partial l/\partial \ovec(C) = - U^T W, \qquad
\partial l/\partial \ovec(V_\cI) = - U^T (C_{\cO,\cI} \otimes I_k).
\end{align}
SATNet computes the above derivative formulas efficiently by iterative algorithms.

\subsection{Symmetries and equivariant maps}
\label{subsec:symmetries}

By symmetries on a set $\cX$, we mean
a group $G$ that acts on $\cX$.
The acting here refers to
a function $\_ \cdot \_$ from $G \times \cX$ to $\cX$, called \emph{group action},
such that (i) $e \cdot x = x$ for the unit $e \in G$ and any $x \in \cX$, and (ii) $(g \circ g') \cdot x = g \cdot (g' \cdot x)$ for all $g, g' \in G$ and $x \in \cX$,
where $\_ \circ \_$ is the
group operator of $G$.

We use symmetries
of permutations on a finite set.
The set $\cX$ in our case is $\R^{k \times n}$,
the space of the matrix $V$ in \eqref{eqn:mixing-method-objective}, and $G$ is a subgroup of the group $\perm_n$ of all permutations on $[n]$.
The group action $g \cdot V$ is then defined by permuting the columns of $V$ by $g$: for all $i,j \in [n]$,
$(g \cdot V)_{i,j} = V_{i,g^{-1}(j)}$.
This group action can be expressed compactly with the $n\times n$ permutation matrix $P_{g^{-1}}$ where
$(P_{g^{-1}})_{i,j} = \ind_{\{i = g^{-1}(j)\}}$ and $g \cdot V = V P_{g^{-1}}$ for the indicator function $\ind$.
Throughout the paper, we often equate each element $g$ of $G$ with its permutation matrix $P_g$, and view $G$ itself as the group of permutation matrices $P_g$ for $g \in G$ with the standard matrix multiplication.

One important reason for considering symmetries is to study maps that preserve these symmetries, called equivariant maps. Let $G$ be a group that acts on sets $\cX$ and $\cY$. 
\begin{definition}
A function $f : \cX \to \cY$ is \emph{$G$-equivariant} or \emph{equivariant} if 
$f(g \cdot x) = g \cdot (f(x))$ for all $g \in G$ and $x \in \cX$.
It is \emph{$G$-invariant} or \emph{invariant} if $f(g \cdot x) = f(x)$ for all $g \in G$ and $x \in \cX$. 
\end{definition}

The forms of equivariant maps have been studied extensively in the work on equivariant neural networks and group representation theory~\citep{Cohen16, Maron19,Wang20}.
In particular, when
$f$ is linear,
various representation theorems for different $G$'s describe the matrix form of $f$. We use
permutation groups defined inductively by the following three
operations.
\begin{definition}
\label{def:group-constructors}
Let $G$ and $H$ be permutation groups on $[p]$ and $[q]$, with each group element 
viewed as a $p\times p$ or $q \times q$ permutation matrix. The \emph{direct sum} $G\oplus H$, 
the \emph{direct product} $G\otimes H$, and the \emph{wreath product} $H \wr G$ are the following groups of 
 $(p+q)\times (p+q)$ or $pq \times pq$ permutation matrices with matrix multiplication as their composition:
\begin{align*}
    G \oplus H & = \{ g \oplus h : g \in G,\, h \in H \};
    &
    G \otimes H & = \{ g \otimes h : g \in G,\, h \in H \};
    \\
    H \wr G & = \{ \wreath(\vec{h},g) : g \in G, \vec{h} \in H^p \}.
\end{align*}
Here $g \oplus h$ is the block diagonal matrix with $p\times p$ matrix $g$ as its upper-left corner and $q\times q$ matrix $h$ as the lower-right corner,
$g \otimes h$ is the Kronecker product of two matrices $g$ and $h$, and $\wreath(\vec{h},g)$ is the $pq \times pq$ permutation matrix defined by
$\wreath(\vec{h},g)_{(i-1)q+j,(i'-1)q+j'} = \ind_{\{ g_{i,i'} = (h_i)_{j,j'} = 1 \}}$ for all $i,i' \in [p]$ and $j,j' \in [q]$.
\end{definition}
The next theorem specifies the representation of $G$-equivariant linear maps for an inductively-constructed $G$,
by describing a basis
of those linear maps.
For a permutation group $G$ on $[m]$, let $\cE(G) = \left\{ M \in \R^{m \times m} : Mg = gM, \, \forall g \in G \right\}$,
the vector space of $G$-equivariant linear maps,
where each $g \in G$ is regarded as a permutation matrix.
See Appendix~\ref{sec:equivariant-basis} for the proof of Theorem~\ref{thm:equivariant-basis}.
\begin{theorem}
\label{thm:equivariant-basis}
Let $G,H$ be permutation groups on $[p]$ and $[q]$, and $\cB(G),\cB(H)$ be some bases of $\cE(G)$ and $\cE(H)$, respectively. Then, the following sets form bases for $G \oplus H$, $G \otimes H$, and $H \wr G$:
\begin{align*}
    \cB(G \,{\oplus}\, H) & =
    \{ A \,{\oplus}\, \mathbf{0}_q, \mathbf{0}_p \,{\oplus}\, B \,{:}\, A \,{\in}\, \cB(G), B \,{\in}\, \cB(H)\}
    \\
    &\qquad \cup \{ \mathbf{1}_{O \times (p+O')}, \mathbf{1}_{(p+O')\times O}\,{:}\, O \,{\in}\, \cO(G), O' \,{\in}\, \cO(H)\};
    \\
    \cB(G \,{\otimes}\, H) & = \{ A \,{\otimes}\, B \,{:}\, A \,{\in}\, \cB(G), B \,{\in}\, \cB(H) \};\
    \\
    \cB(H \,{\wr}\, G) & {} = \{ A \,{\otimes}\, \mathbf{1}_{O' \times O''} \,{:}\, A \,{\in}\, \cB(G), A_{i,i} = 0\ \text{ for } i \in [p] \text{ and } O', O'' \,{\in}\, \cO(H)\} \\
    &\qquad {} \cup \{ I_O \,{\otimes}\, B \,{:}\, B \,{\in}\, \cB(H), O \,{\in}\, \cO(G) \}.
\end{align*}
Here $\mathbf{0}_m$ is an everywhere-zero matrix in $\R^{m\times m}$,
and $\mathbf{1}_{R \times S}$, $I_R$ are the matrices defined by
$(\mathbf{1}_{R\times S})_{i, j} = \ind_{\{ i \in R, \, j \in S \}}$, $(I_R)_{i, j} = \ind_{\{ i = j, \, i \in R \}}$
whose shapes are defined by the context in which they are used.
Here, $\mathbf{1}_{O \times (p + O')}, \mathbf{1}_{(p + O') \times O} \in \R^{(p+q) \times (p+q)}$,
$\mathbf{1}_{O' \times O''} \in \R^{q \times q}$,
$I_O \in \R^{p \times p}$.
Also, $\cO(G) = \{\{g(i) : g \in G\} : i \in [p]\}$ (i.e., the set of $G$-orbits), and $p+O = \{p + i : i \in O\}$.
\end{theorem}

%% file: result.tex
\section{Symmetry-aware SATNet}
\label{sec:SymSATNet}

In this section, we present SymSATNet, which abbreviates symmetry-aware SATNet.
This variant is designed to operate when symmetries of a learning task are known a priori (via an algorithm or a domain expert).
The proofs of the theorem and the lemma in the section are in Appendix~\ref{sec:appendix-symsatnet}.

SymSATNet solves the optimisation problem of SATNet, but under the following assumptions:
\begin{assumption}
    \label{asm:symsatnet}
    The optimisation objective $\langle C, V^T V \rangle$ in \eqref{eqn:mixing-method-objective} as a map on $V = \R^{k \times n}$ is invariant under a permutation group $G$,
    whose action is of type in Section~\ref{subsec:symmetries} (i.e., each $g \in G$ acts as a permutation on the columns of $V$).
\end{assumption}
Continuing our convention, we denote $P_g$ by $g$.
One immediate consequence of Assumption~\ref{asm:symsatnet} is
\begin{align*}
\langle C,V^TV \rangle 
& {} = \langle C, (g \cdot V)^T (g \cdot V) \rangle 
 {} = \langle C, (V g^{-1})^T (V g^{-1})\rangle
 \quad \text{for all } g \in G \text{ and } V \in \R^{k \times n}.
\end{align*}
The next theorem re-phrases this property of the optimisation objective as equivariance of $C$:
\begin{theorem} 
\label{thm:objective-invariance-equivariance}
Let $C$ be a symmetric $n\times n$ matrix. Then, 
\begin{equation}
\label{eqn:objective-invariance-equivariance0}
\langle C,V^TV \rangle = \langle C, (V g^{-1})^T (V g^{-1})\rangle 
\end{equation}
for all $V \in \R^{k \times n}$ and $g \in G$ if $C$ as a linear map on $\R^n$ is $G$-equivariant, that is, $C g = g C$ for all $g \in G$. Furthermore, if $k = n$, the converse also holds.
\end{theorem}

This theorem lets us incorporate the symmetries into the objective of SATNet,
and leaves the handling of the diagonal constraints of SATNet.
The next lemma says that those constraints require no special treatment, though,
since they are already preserved by the action of any $g \in G$.
\begin{lemma}
\label{lem:constraint-invariance-equivariance} 
Let $V \in \R^{k\times n}$ and $g \in G$. Every column of $V$ has the $L_2$-norm $1$ 
if and only if every column of $V g^{-1}$ has the $L_2$-norm $1$.
\end{lemma}

Recall $\cE(G) = \{ M \in \R^{n \times n} : Mg = gM, \, \forall g \in G\}$ is the vector space of $G$-equivariant matrices.
Let $\cE(G)_s$ be the subset of $\cE(G)$ containing only symmetric matrices.
When $G$ is a permutation group constructed by direct sum, direct product, and wreath product,
we can generate a basis $\cB(G)$ of $\cE(G)$ automatically using Theorem~\ref{thm:equivariant-basis}.
Then, we can convert $\cB(G)$ to an orthogonal basis of $\cE(G)_s$ by applying the Gram-Schmidt orthogonalisation to $\{B + B^T : B \in \cB(G)\}$.
Let $\cB(G)_s = \{ B_1, \, \ldots, \, B_d \}$ be such an orthogonal basis of $\cE(G)_s$.
  
SymSATNet is SATNet where the matrix $C$ in the optimisation objective has the form:
\begin{equation}
\label{eqn:SymSATNet-equivariance}
C = \sum_{\alpha = 1}^d \theta_\alpha B_\alpha
\end{equation}
for some scalars $\theta_1,\ldots,\theta_d \in \R$.
Note that by this condition on the form of $C$, SymSATNet has only $d$ parameters $\theta_1,\ldots,\theta_d$, instead of
$n\times m$ for some $m$ in the original formulation of SATNet.
When the learning target has enough symmetries, $d$ is usually far smaller than $n^2$ or even $n$, and this reduction brings speed-up and improved generalisation.

The forward computation of SymSATNet is precisely that of SATNet, the repeated coordinate-wise updates until convergence,
and the backward computation is the one of SATNet extended (by the chain rule) with a step backpropagating the derivatives
$\partial l / \partial C$ to each $\partial l / \partial \theta_\alpha$ for $\alpha \in [d]$.\footnote{SymSATNet is implemented based on the SATNet code~\citep{WangDWK19} available under the MIT License.}

We summarise SymSATNet below using the usual notation of SATNet ($\cI$, $\cO$, $V_\cI$, $V_\cO$, and $V$):
\begin{compactitem}
\item The input is $V_\cI$, the matrix of the input columns of $V$. 
\item The parameters are $(\theta_1,\ldots,\theta_d) \in \R^d$. They define the matrix $C$ by \eqref{eqn:SymSATNet-equivariance}.
\item The forward computation solves the following optimisation problem using coordinate descent, and returns $V_\cO$, the matrix of the output columns of $V$:
\begin{align*}
\argmin_{V_\cO \in \R^{k \times |\cO|}} \,\langle C, V^TV\rangle 
\quad \text{subject to}\ \Vert v_o \Vert = 1 \ \text{for}\ o \in \cO.
\end{align*}
\item The backward computation computes $\partial l / \partial V_\cI$ and $\partial l / \partial \theta_\alpha$ by \eqref{eqn:SATNet-derivative}
and the chain rule:
\begin{align}
    \label{eqn:SymSATNet-chainrule}
    \partial l/\partial \theta_\alpha & = (\partial l/\partial \ovec(C)) \ovec(B_\alpha) = - U^T W \ovec(B_\alpha).
\end{align}
\end{compactitem}

%% file: algo.tex
\section{Discovery of Symmetries}
\label{sec:symfind}

One obstacle for using SymSATNet is that the user has to specify symmetries.
We now discuss how to alleviate this issue by presenting an algorithm for discovering candidate symmetries automatically.

The goal of our algorithm, denoted by $\symfind$, is to find a permutation group $G$ that captures the symmetries of an unknown learning target and is expressible by the following grammar:
$ G ::= \, \idg_m \,\mid\, \cyclic_m \,\mid\, \perm_m \,\mid\, G \oplus G \,\mid\, G \otimes G \,\mid\, G \wr G\ $ for $m \in \N$.
The $\idg_m$ denotes the trivial group containing only the identity permutation on $[m]$, and $\cyclic_m$ denotes the group of cyclic permutations on $[m]$,
each of which maps $i \in [m]$ to $(i + n)\,\mathrm{mod}\, m$ for some~$n$.
The $\perm_m$ is the group of all the permutations on $[m]$.
The last three cases are direct sum, direct product, and wreath product (see Definition~\ref{def:group-constructors}).
They describe three ways of decomposing a group $G$ into smaller parts.
Having such a decomposition of $G$ brings the benefit
to recursively and efficiently compute a basis of $G$-equivariant linear maps.

The design of $\symfind$ is based on our empirical observation that a softened version of symmetries often emerges in the parameter matrix $C$ of the original SATNet during training.
Even in the early part of training, many entries of $C$ share similar values,
and there is a large-enough group $G$ with $Cg \approx gC$ for all $g \in G$,
which intuitively means that $G$ captures symmetries of $C$.
Furthermore, we observed, such $G$ often consists of symmetries of the learning target.
This observation suggests an algorithm that takes $C$ as input and finds such $G$ in our grammar or its slight extension. 

\begin{wrapfigure}{r}{0.5\textwidth}
\vskip -2.2em
\begin{minipage}{0.5\textwidth}
\begin{algorithm}[H]
    \caption{$\symfind$ with a threshold $\lambda > 0$}
    \label{alg:symfind}
\begin{algorithmic}[1]
    \STATE {\bfseries Input:} $M \in \R^{m \times m}$\qquad {\bfseries Output:} $(G, \sigma)$
    \IF {$\Vert \prj(\group(\perm_m, \id_m),M) - M \Vert_F \le \lambda $}
        \RETURN $(\perm_m,\id_m)$
    \ENDIF
    \STATE $\cA \gets \{(\idg_m, \id_m)\}$
    \IF {$\Vert \prj(\group(\cyclic_m,  \id_m),M) - M \Vert_F \le \lambda $}
           \STATE $\cA \gets \cA \cup \{(\cyclic_m, \id_m)\}$
    \ENDIF
    \STATE $(G',\sigma') \gets \sumsplit(M)$;
    \STATE $\cA \gets \cA \cup \{(G',\sigma')\}$
    \FOR{every divisor $p$ of $m$}
            \STATE $(G'',\sigma'')\, {\gets}\, \kronsplit(M,p)$;
            \STATE $\cA\, {\gets}\, \cA \,{\cup}\, \{(G'',\sigma'')\}$
    \ENDFOR
    \STATE $(G, \sigma) \gets \argmin_{(G,\sigma) \in \cA} \dim(\cE(\group(G,\sigma)))$
    \\
    \STATE {\bfseries return} $(G,\sigma)$
\end{algorithmic}
\end{algorithm}
\end{minipage}
\vskip -2em
\end{wrapfigure}

The input of $\symfind$ is a matrix $M \in \R^{m\times m}$.
As previously explained, when $\symfind$ is called at the top level, it receives as input the parameter $C$ of SATNet learnt by a fixed number of
training steps.
However, subsequent recursive calls to $\symfind$ may have input $M$ different from $C$.
Then, $\symfind$ returns a group $G$ in our grammar and a permutation $\sigma$ on $[m]$, together defining a permutation group on $[m]$:
\begin{align*}
    & \symfind(M) = (G, \sigma); \\
    &\group(G,\sigma) = \{ \sigma \circ g \circ \sigma^{-1} : g \in G\},
\end{align*}
where $\circ$ is the composition of permutations.
When $G$ is decomposed into, say $G_1 \oplus G_2$, the $\sigma$ specifies which indices in $[m]$ get permuted by $G_1$ and
$G_2$.
Once
top-level $\symfind$ returns $(G, \sigma)$,
we construct $\cB(\group(G, \sigma))_s$, as
in Section~\ref{sec:SymSATNet} with a minor adjustment with $\sigma$.\footnote{We construct $\cB(\group(G,\sigma)) = \{ \sigma B \sigma^T : B \in \cB(G)\}$, which is an orthogonal basis for $\cE(\group(G,\sigma))$.}

Algorithm~\ref{alg:symfind} describes $\symfind$, where $\id_m$ is the identity permutation on $[m]$, $\Vert {} \cdot {} \Vert_F$ is the Frobenius norm,
$\dim(\cV)$ is the dimension of a vector space $\cV$,
and the Reynolds operator $\prj$ projects a matrix $M \in \R^{m\times m}$ orthogonally to the subspace of $G$-equivariant $m \times m$ matrices:
\begin{equation*}
\prj(G,M) = \frac{1}{|G|} \sum_{g \in G} g M g^T,
\end{equation*}
so that $\Vert \prj(G, M) - M \Vert_F$ computes the $L^2$ distance between the matrix $M$ and the space $\cE(G)$.

In the lines 2-4, the algorithm first checks whether $\cS_m$
models symmetries of the input $M$ accurately. If so, the algorithm returns $(\cS_m, \id_m)$.
Otherwise, it assumes that an appropriate group for $M$'s symmetries is one of the remaining cases in the grammar,
and constructs a list $\cA$ of candidates intially containing the trivial group $(\idg_m, \id_m)$.
In the lines 6-8, the algorithm adds a pair $(\cyclic_m, \id_m)$
to $\cA$ if it approximates $M$'s symmetries well.
In the lines 9-10,
the algorithm calls the subroutine $\sumsplit$
which finds $\group(G', \sigma')$ with $G' = \bigoplus_i G_i'$ that approximates $M$'s symmetries well.
In the lines 11-14, the algorithm calls the other subroutine $\kronsplit$ for every divisor $p$ of $m$.
For each $p$, $\kronsplit$ finds $\group(G'', \sigma'')$ with $G'' = G''_1 \otimes G''_2$ (or $G''_2 \wr G''_1$) that approximates $M$'s symmetries well,
where $G''_1$ and $G''_2$ are permutation groups on $[p]$ and $[m/p]$.
Finally, in the line 15, $\symfind$ picks a pair $(G,\sigma)$ from the candidates $\cA$ with the strongest level of symmetries in the sense that
the basis of $\group(G, \sigma)$-equivariant matrices
has the fewest elements.

The subroutine $\sumsplit$ clusters entries of $M$ as blocks since block-shaped clusters commonly arise in matrices equivariant with respect to a direct sum of groups.
The other subroutine $\kronsplit$ uses a technique~\citep{Van93} to exploit a typical pattern of Kronecker product of matrices,
and detects the presence of the pattern in $M$ by applying SVD to a reshaped version of $M$.
Each subroutine may call $\symfind$ recursively.
See Appendix~\ref{sec:symfind-appendix} for the details.

%% file: experiment.tex
\section{Experimental Results}
\label{sec:experiments}

We experimentally evaluated SymSATNet and the $\symfind$ algorithm on the tasks of learning rules of two problems, Sudoku and the completion problem of Rubik's cube.
The original SATNet was used as a baseline, and both ground-truth and automatically-discovered symmetries were used for SymSATNet.
For $\symfind$, we also tested its ability to recover known symmetries given randomly generated equivariant matrices.
We observed significant improvement of SymSATNet over SATNet in various learning tasks, and also the promising results and limitation of $\symfind$.

\paragraph{Sudoku problem}
In Sudoku, we are asked to fill in the empty cells of a $9 \times 9$ board such that every row, every column, and each of nine $3 \times 3$ blocks have all numbers $1-9$.
Let $A \in \{0,1\}^{9 \times 9 \times 9}$ be the encoding of a full number assignment for the board where the $(i, j, k)$-th entry of $A$ is $1$
if the $(i, j)$-th cell of the board contains $k$.
In SATNet, we flatten $A$ to the assignment on the $n = 9^3$ boolean variables, and relax each variable
into $\R^k$, resulting $V \in \R^{k \times n}$ in the objective of SATNet.

\begin{wrapfigure}{r}{0.48\columnwidth}
    \begin{subfigure}{0.48\columnwidth}
    \includegraphics[width=0.95\columnwidth]{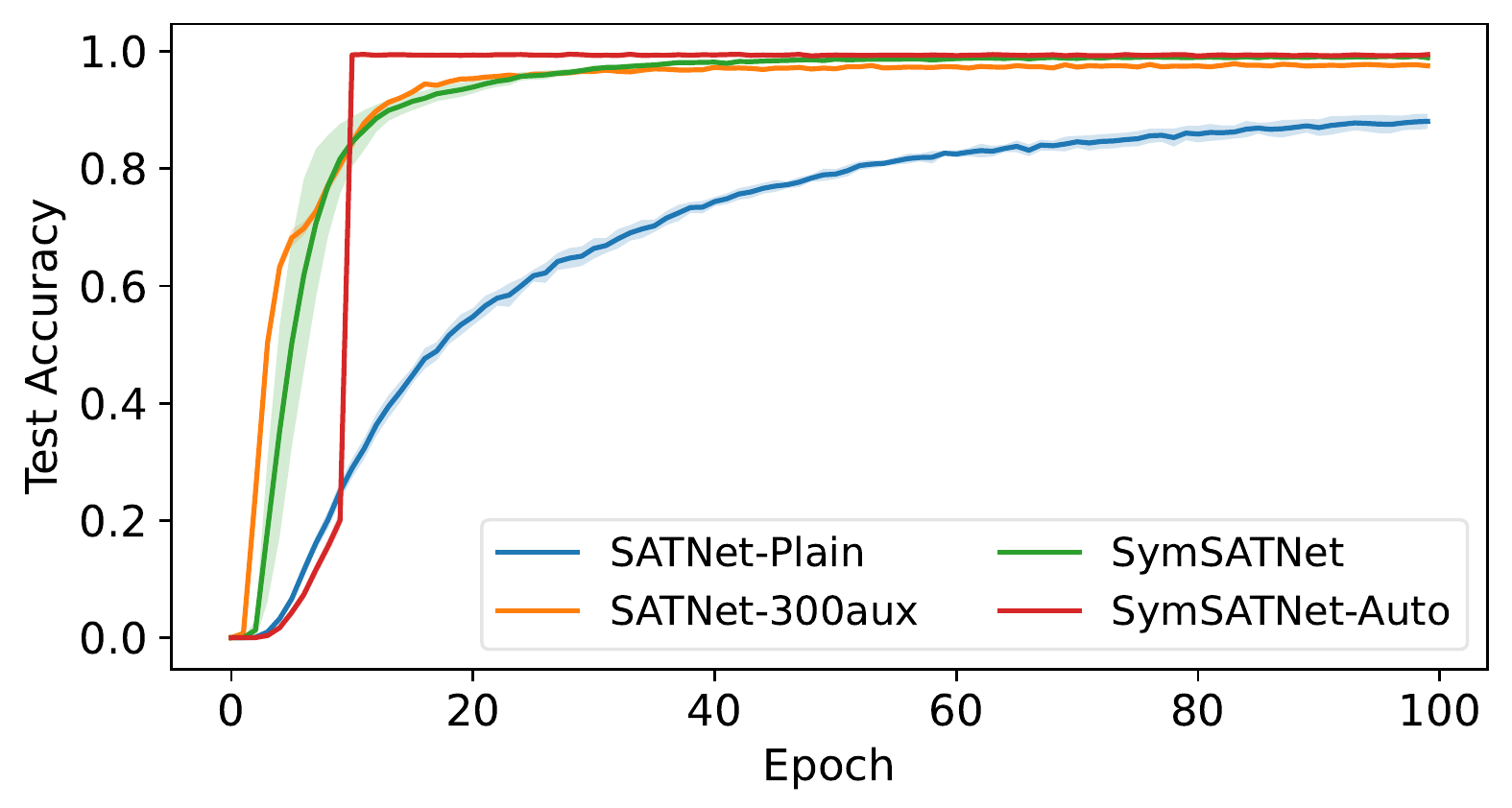}
    \vskip -0.5em
    \subcaption{Sudoku}
    \end{subfigure}
    \begin{subfigure}{0.48\columnwidth}
    \includegraphics[width=0.95\columnwidth]{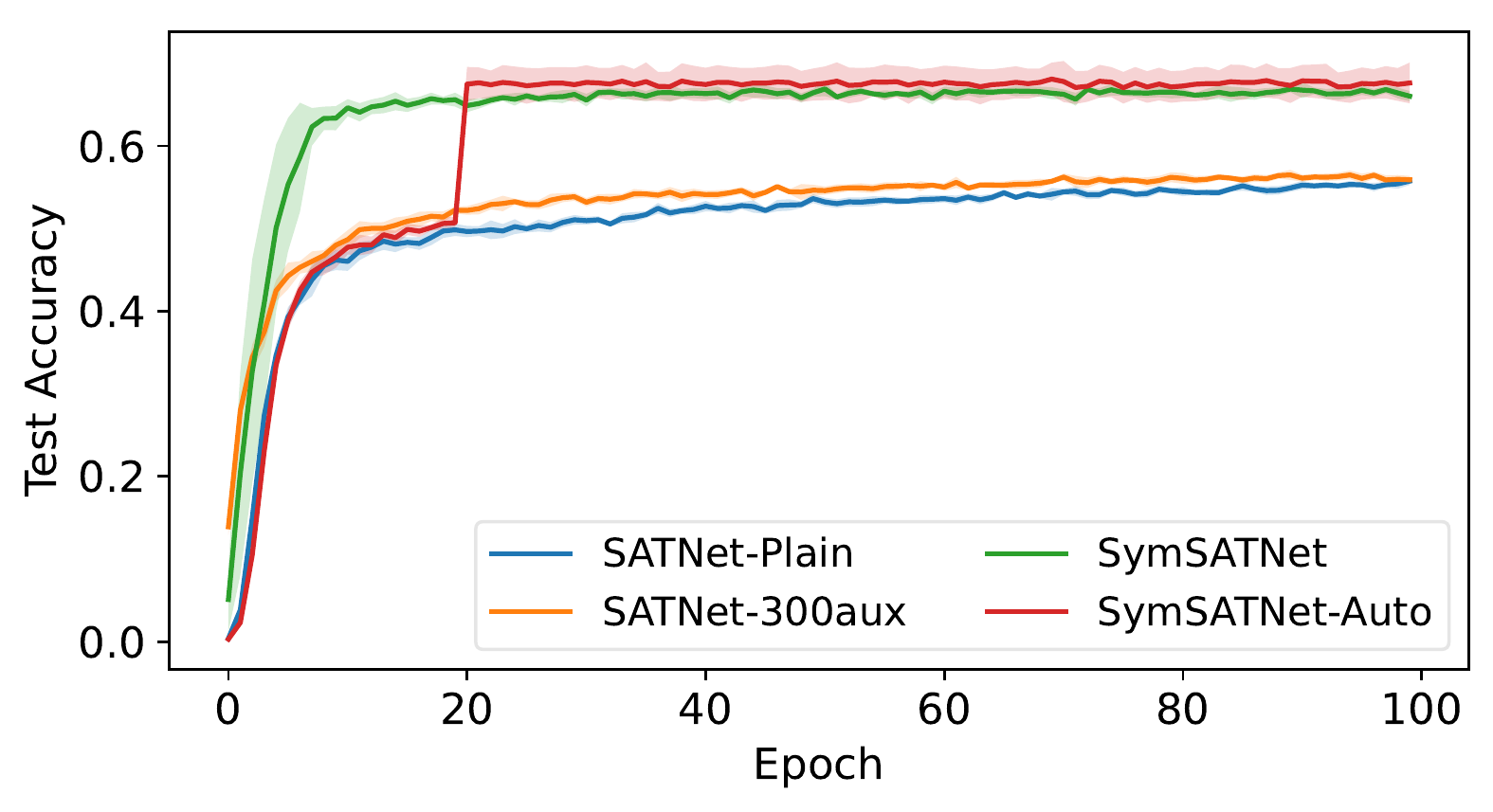}
    \vskip -0.5em
    \subcaption{Rubik's cube}
    \end{subfigure}
    \caption{Test accuracies over training epochs.}
    \label{fig:accuracy}
    \vskip -2em
\end{wrapfigure}

The rules of Sudoku have symmetries formalised by
$G = (\perm_3 \wr \perm_3) \otimes (\perm_3 \wr \perm_3) \otimes \perm_9$.
Each of two \linebreak $\perm_3 \wr \perm_3$ refers to solution-preserving permutations for rows and columns in Sudoku.
The last $\perm_9$ refers to permutations of the assigned numbers $1-9$ in each cell.
See Appendix~\ref{sec:symmetries-sudoku-cube} for more information about the symmetry group for Sudoku.

To learn the rules of Sudoku using SymSATNet, we constructed a basis $\cB(G)_s$
as explained in Section~\ref{sec:SymSATNet}.
It has 18 elements, thus SymSATNet has 18 parameters to learn.

We used 9K training and 1K test examples generated by the Sudoku generator~\citep{Park18}.
Each example is a pair $(V_\cI,V_\cO)$ where the input $V_\cI$ assigns 31-42 cells (out of 81 cells) and the output $V_\cO$ specifies the remaining cells.
SymSATNet was compared with
SATNet-Plain without auxiliary variables, and SATNet-300aux with 300 auxiliary variables.
We used
binary cross entropy loss and Adam optimizer~\citep{Kingma15},
with the learning rate $\eta=2 \times 10^{-3}$ for SATNet-Plain and SATNet-300aux as the original work
and $\eta=4 \times 10^{-2}$ for SymSATNet.
We measured test accuracy, the rate of the correctly-solved Sudoku instances
by the forward computations.
We reported the average over 10 runs with $95\%$ confidence interval.

The results are in Figure~\ref{fig:accuracy} and Table~\ref{table:accuracy}.
(SymSATNet-Auto refers to a variant of SymSATNet that
uses $\symfind$ as a subroutine to find symmetries automatically, and will be described later in this section.)
Our SymSATNet outperformed SATNet-Plain and SATNet-300aux.
On average, its 100 epochs finished $2$-$4\times$ faster in terms of wall clock than two alternatives.
due to the reduced number of iterations and avoidance of matrix operations. See Appendix~\ref{sec:algorithm-satnet} for the efficiency of SymSATNet.
Despite the speed up, the best test accuracy of SymSATNet ($99.2\%$) was significantly better than SATNet-Plain ($88.1\%$) and slightly better than SATNet-300aux ($97.9\%$).

\begin{wrapfigure}{r}{0.5\columnwidth}
    \vskip -3em
    \begin{minipage}{0.5\columnwidth}
    \begin{table}[H]
        \begin{center}
        \begin{small}
        \begin{sc}
        \caption{Best test accuracies during 100 epochs and average train times ($10^2$ sec).
        Additional times for automatic symmetry detection are also reported after $+$.}
        \label{table:accuracy}
        \vskip 0.5em
        \begin{tabular}{lrrrr}
        \toprule
        Model & \multicolumn{2}{c}{Sudoku} & \multicolumn{2}{c}{Cube} \\
        & Acc. & Time & Acc. & Time \\
        \midrule \midrule
        SATNet\scriptsize-Plain   & 88.1\%                 & 48.0                 & 55.7\%                 & 1.8      \\
                                  & \scriptsize $\pm$1.8\% & \scriptsize $\pm$0.17 & \scriptsize $\pm$0.7\% & \scriptsize $\pm$0.01 \\
        SATNet\scriptsize-300aux  & 97.9\%                 & 90.3                 & 56.5\%                 & 14.0     \\
                                  & \scriptsize $\pm$0.3\% & \scriptsize $\pm$0.68 & \scriptsize $\pm$0.9\% & \scriptsize $\pm$0.12 \\
        \midrule
        SymSATNet                 & 99.2\%                           & 25.6                          & 66.9\%                 & \textbf{1.1}       \\
                                  & \scriptsize $\pm$0.2\%           & \scriptsize $\pm$0.14          & \scriptsize $\pm$1.2\% & \scriptsize \textbf{$\pm$0.00}  \\
        SymSATNet\scriptsize-Auto & \textbf{99.5\%}                  & \textbf{22.7}                & \textbf{68.1\%}                          & 3.4             \\
                                  & \scriptsize \textbf{$\pm$0.2\%}  & \scriptsize \textbf{\makecell{$+$0.14 \\ $\pm$0.35}} & \scriptsize \textbf{$\pm$2.8\%}          & \scriptsize \makecell{$+$0.66 \\ $\pm$0.19}  \\
        \bottomrule
        \end{tabular}
        \end{sc}
        \end{small}
        \end{center}
    \end{table}
    \end{minipage}
    \vskip -2.7em
\end{wrapfigure}

\paragraph{Completion problem of Rubik's cube}
The Rubik's cube is composed of $6$ faces, each of which has $9$ facelets.
We considered a constraint satisfaction problem where
we are asked to complete the missing facelets of the Rubik's cube such that the resulting cube is solvable;
by moving the cube, we can make all facelets in each face have the same colour, and no same colours appear in two faces.
Let $A \in \{0,1\}^{6 \times 9 \times 6}$ be a colour assignment of Rubik's cube where the $(i, j, k)$-th entry has $1$ if and only if the $j$-th facelet of the $i$-th face has colour $k$. 
We formulate the optimisation objective of SATNet for Rubik's cube using the relaxation of $A$ to $V \in \R^{k \times n}$ for $n = 6 \times 9 \times 6$. 

This problem has symmetries formalised by $G = \cR_{54} \otimes \cR_6$ on $[n]$.
Here $\cR_{54}$ and $\cR_6$ are permutation groups on $[54]$ and $[6]$, each of which captures the allowed moves of facelets, and the rotations of the whole cube.
If a colour assignment $A$ is solvable, so is the transformation of $A$ by any permutations in $G$.
See Appendix~\ref{sec:symmetries-sudoku-cube} for more information about the symmetries of this problem. 

We generated a basis $\cB(G)_s$ in three steps.
We first created $\cB(\cR_{54})$ and $\cB(\cR_6)$ using the generators of each group~\citep{Finzi21a}.
Next, we combined them using Theorem~\ref{thm:equivariant-basis} to get $\cB(G)$,
which was converted to a symmetric orthogonal basis $\cB(G)_s$ via Gram-Schmidt.
The final result has $48$ basis elements.

We used a dataset of 9K training and 1K test examples generated by randomly applying moves to the solution of the cube.
Each example is a pair $(V_\cI,V_\cO)$ where $V_\cI$ assigns colours to facelets except for two corner facelets,
two edge facelets, and one center facelet, and $V_\cO$ specifies the colours of those five missing facelets.
In the test examples, only $V_\cI$ is used.
We trained SymSATNet, SATNet-Plain, and SATNet-300aux for 100 epochs, under the same configuration as in the Sudoku case.

The results appear in Figure~\ref{fig:accuracy} and Table~\ref{table:accuracy}.
On average, the $100$-epoch training of SymSATNet completed faster in the wall-clock time than those of SATNet-Plain and SATNet-300aux.
Also, it achieved better test accuracies ($66.9\%$) than these alternatives ($55.7\%$ and $56.5\%$).
Note that unlike Sudoku, the test accuracy of SATNet-300aux was only marginally better than that of SATNet-Plain,
which indicates that both suffered from the overfitting issue.
Note also the sharp increase in the training time of SATNet-300aux.
These two indicate that adding auxiliary variables is not so effective for the completion problem for Rubik's cube,
while exploiting symmetries is still useful.

\paragraph{Automatic discovery of symmetries}
\label{sec:limitation}
To test the effectiveness of $\symfind$, we tested whether $\symfind$ could find proper symmetries in Sudoku and Rubik's cube.
We applied $\symfind$ to the parameter $C$ of SATNet-Plain in $T$-th training epoch, where $T = 10$ for Sudoku and $T = 20$ for Rubik's cube.
For Sudoku, $\symfind$ always recovered the full symmetries with $G = (\perm_3 \wr \perm_3) \otimes (\perm_3 \wr \perm_3) \otimes \perm_9$ in our 10 trials.
For Rubik's cube, the group of full symmetries is $\group(G,\sigma)$ for $G = ((\perm_2 \wr \perm_3) \oplus (\perm_3 \wr \perm_8) \oplus (\perm_2 \wr \perm_{12})) \otimes (\perm_2 \wr \perm_3)$.
$\symfind$ recovered all the parts except $\perm_2 \wr \perm_{12}$.
Instead of this, the algorithm found $\perm_{12} \otimes \perm_2$ or $\perm_3 \otimes \perm_8$, or $\perm_4 \otimes \perm_6$ in our 10 trials.
We manually observed that the entries of $C$
in the corresponding part were difficult to be clustered, violating the assumption of $\symfind$.
This illustrates the fundamental limitation of $\symfind$.

\begin{wrapfigure}{r}{0.48\columnwidth}
    \begin{subfigure}{0.48\columnwidth}
    \includegraphics[width=0.95\columnwidth]{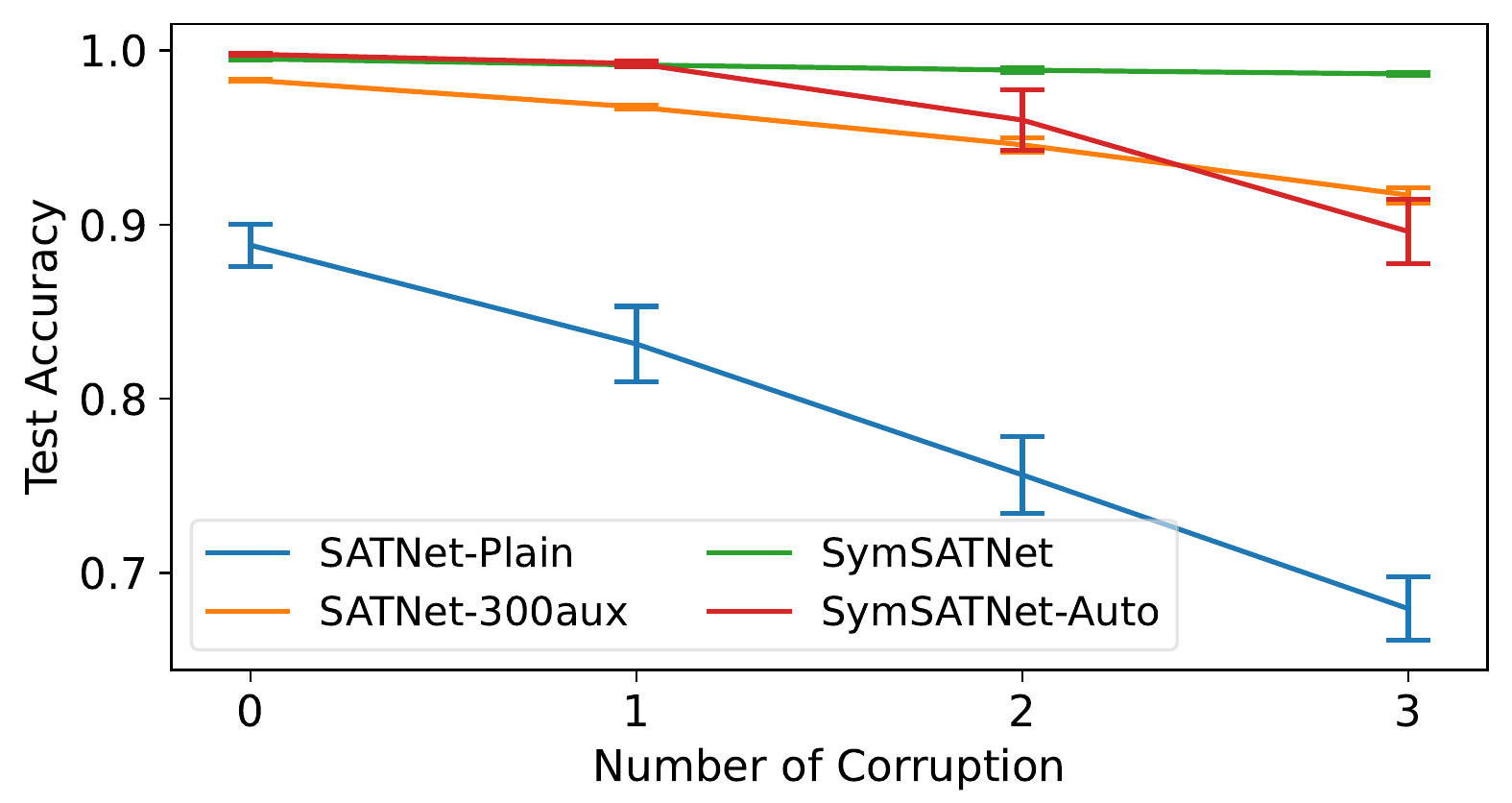}
    \vskip -0.5em
    \subcaption{Sudoku}
    \end{subfigure}
    \begin{subfigure}{0.48\columnwidth}
    \includegraphics[width=0.95\columnwidth]{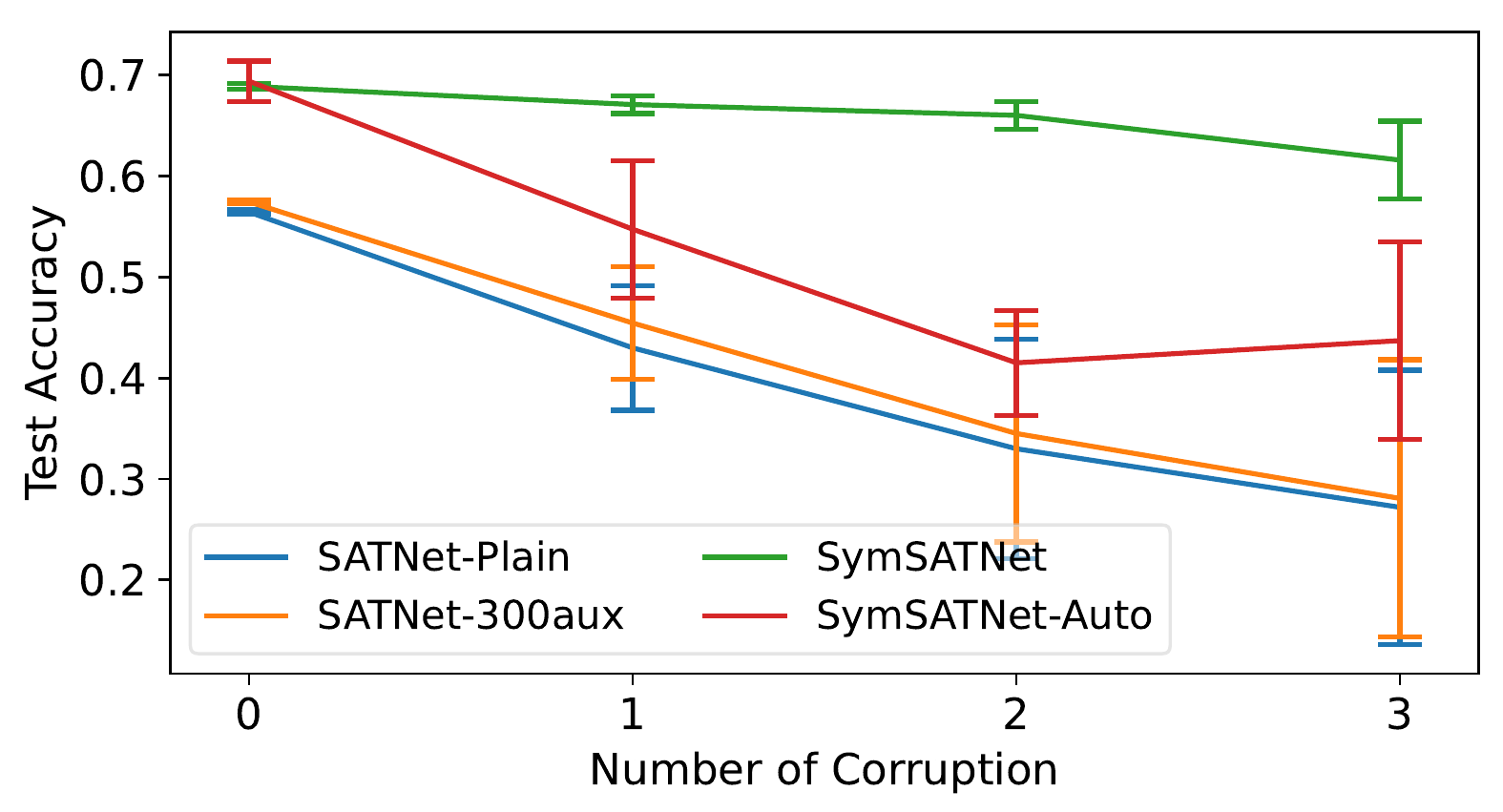}
    \vskip -0.5em
    \subcaption{Rubik's cube}
    \end{subfigure}
    \caption{Best test accuracies for noisy Sudoku and Rubik's cube datasets.}
    \label{fig:noise-robustness}
    \vskip -1em
\end{wrapfigure}

To account for the limitation of $\symfind$, we refined the group $G$ of detected symmetries to a subgroup in an additional validation step, before training SymSATNet with those symmetries.
In the validation step, we checked the usefulness of each part $G_i$ of the expression of $G$ in our grammar.
Concretely, we rewrote $G$ only with the part $G_i$ in concern, where all the other parts of $G$ were masked by the trivial groups $\idg_k$.
After projecting $C$ with the masked groups using Reynolds operator, we measured the improvement of accuracy of SATNet over validation examples.
Finally, we assembled only the parts $G_i$ that led to sufficient improvement.
For example, if $G = (\cyclic_3 \otimes \perm_4) \oplus \perm_5$ is discovered by $\symfind$, we consider the parts $G_1 = \cyclic_3$, $G_2 = \perm_4$, and $G_3 = \perm_5$.
Then, we construct three masked groups $G_1' = (\cyclic_3 \otimes \idg_4) \oplus \idg_5$, $G_2' = (\idg_3 \otimes \perm_4) \oplus \idg_5$, and $G_3' = (\idg_3 \otimes \idg_4) \oplus \perm_5$,
and measure the accuracy of SATNet with $C$ projected by each $G_i'$ over validation examples. 
If $G_1'$ and $G_3'$ show accuracy improvements greater than a threshold, we combine $G_1$ and $G_3$ to form $G' = (\cyclic_3 \otimes \idg_4) \oplus \perm_5$, which is then used to train SymSATNet.

We used 8K training, 1K validation, and 1K test examples to train SymSATNet with symmetries found by $\symfind$ and the validation step.
We denote these runs by SymSATNet-Auto.
We took a group $G$ discovered by $\symfind$ in $T$-th training epoch (with the same $T$ above)
and constructed its subgroup $G'$ via the validation step.
SymSATNet was then trained after being initialised by the projection of $C$ with $G'$.
The other configurations are the same as before.

As shown in Figure~\ref{fig:accuracy} and Table~\ref{table:accuracy},
SymSATNet-Auto performed the best for Sudoku ($99.5\%$) and Rubik's cube ($68.1\%$) better than even SymSATNet.
During the 10 trials with Sudoku, SymSATNet-Auto was always given the full symmetries in Sudoku.
For Rubik's cube, when SymSATNet-Auto was given correct subgroups
(e.g., $((\perm_2 \wr \perm_3) \oplus (\perm_3 \wr \perm_8) \oplus (\idg_4 \otimes \idg_6)) \otimes (\perm_2 \wr \perm_3)$,
$((\perm_2 \wr \perm_3) \oplus (\perm_3 \wr (\perm_2 \wr \perm_4)) \oplus (\idg_8 \otimes \idg_3)) \otimes (\perm_2 \wr \perm_3)$),
then it performed even better than SymSATNet.
In two of the 10 trials, slightly incorrect symmetries were exploited, but it outperformed SATNet-Plain and SATNet-300aux.
These results show the partial symmetries of subgroups derived by the validation step are still useful, even when they are slightly inaccurate.

\paragraph{Robustness to noise}
We tested robustness of SymSATNet and SymSATNet-Auto to noise by noise-corrupted datasets.
We generated noisy Sudoku and Rubik's cube datasets where each training example is corrupted with noise;
it alters the value of a random cell or the colour of a random facelet to a random value other than the original.
We used noisy datasets with 0-3 corrupted instances to measure the test accuracy, and tried 10 runs for each dataset to report the average and $95\%$ confidence interval.
All the other setups are the same as before.
Figure~\ref{fig:noise-robustness} shows the results.
In both problems, SymSATNet was the most robust, showing remarkably consistent accuracies.
SymSATNet-Auto showed comparable robustness to SATNet-300aux in noisy Sudoku, but outperformed the two baselines in noisy Rubik's cube.

Next, to show the robustness of $\symfind$, we applied it to restore permutation groups $G$ from noise-corrupted $G$-equivariant symmetric matrices $M$.
We picked $(G_i, \, \sigma_i)$ for $i \in [4]$ where $\sigma_1$, $\sigma_2$, $\sigma_3$ are random permutations on $[15]$, $[30]$, $[12]$, and $\sigma_4$ is the identity permutation on $[8]$, and 
\begin{align*}
G_1 &= \cyclic_3 \oplus \cyclic_3 \oplus \cyclic_3 \oplus \cyclic_3 \oplus \cyclic_3,
&
G_2 & = \perm_3 \wr \perm_{10},
&
G_3 &= (\perm_3 \wr \perm_3) \oplus \cyclic_3,
&
G_4 & = \perm_2 \otimes \perm_2 \otimes \perm_2.
\end{align*}
Then, we generated $\group(G_i,\sigma_i)$-equivariant symmetric matrices $M_i$
by projecting random matrices with standard normal entries into the space $\cE(\group(G_i,\sigma_i))_s$.
Then, Gaussian noises from $\cN(0,\omega^2)$ for $\omega = 5 \times 10^{-3}$ are added to $M_i$'s entries, and the resulting matrix $M'_i$ is given to $\symfind$. 

\begin{wrapfigure}{r}{0.4\columnwidth}
    \vskip -3.2em
    \begin{minipage}{0.4\columnwidth}
    \begin{table}[H]
    \begin{center}
    \begin{small}
    \begin{sc}
    \caption{Full accuracies and partial accuracies of $\symfind$ for given groups over 1K runs.}
    \vskip 0.5em
    \label{table:accuracy-symfind}
    \begin{tabular}{lrrrr}
    \toprule
    Group & Full Acc. & Partial Acc. \\
    \midrule
    $\bigoplus_{i=1}^5 \cyclic_3$              & 76.6\%     & 79.2\%    \\
    $\perm_3 \wr \perm_{10}$                   & 60.3\%     & 79.9\%    \\
    $(\perm_3 \wr \perm_3) \oplus \cyclic_3$   & 77.5\%     & 87.0\%   \\
    $\perm_2 \otimes \perm_2 \otimes \perm_2$  & 93.5\%     & 94.3\%    \\
    \bottomrule
    \end{tabular}
    \vskip 1em
    \end{sc}
    \end{small}
    \end{center}
    \end{table}
    \end{minipage}
    \vskip -2em
\end{wrapfigure}

For each $(G_i,\sigma_i)$, we repeatedly generated $M'_i$ and ran $\symfind$ on $M'_i$ for 1K times,
and measured the portion where $\symfind$ recovered $(G_i,\sigma_i)$ exactly (full accuracy),
and also the portion of cases where $\symfind$ returned a subgroup of $(G_i,\sigma_i)$ which is not the trivial group $\idg_m$ (partial accuracy).
As Table~\ref{table:accuracy-symfind} shows, the measured full accuracies were in the range of $60.3-93.5\%$, and the partial accuracies were in the range of $79.2-94.3\%$.
These results show the ability of $\symfind$ to recover meaningful and sometimes full symmetries.

\begin{figure}
    \begin{subfigure}{0.48\columnwidth}
    \includegraphics[width=0.95\columnwidth]{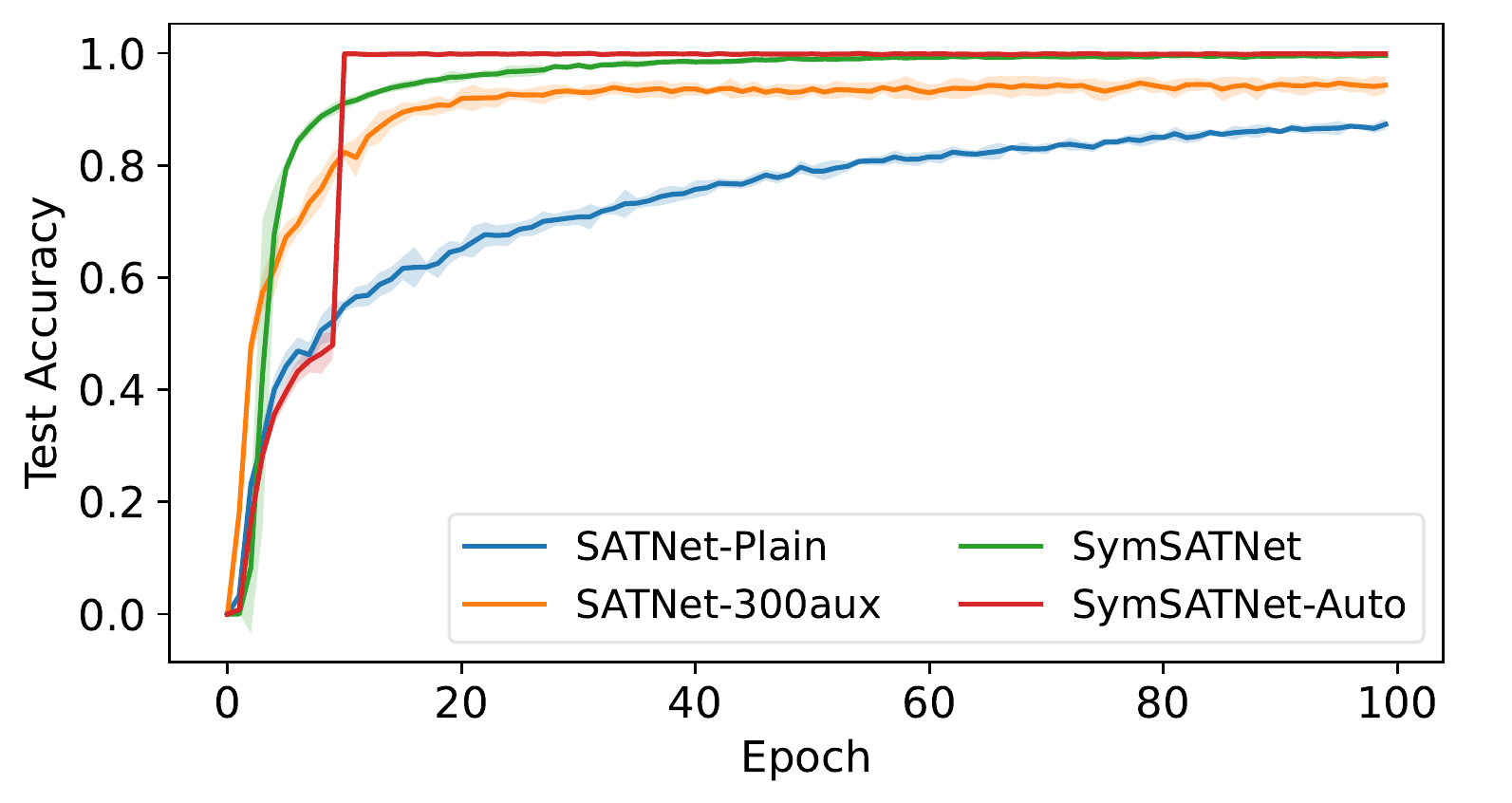}
    \subcaption{Normal Sudoku $\rightarrow$ Hard Sudoku}
    \label{fig:transfer-sudoku-1}
    \end{subfigure}
    \begin{subfigure}{0.48\columnwidth}
    \includegraphics[width=0.94\columnwidth]{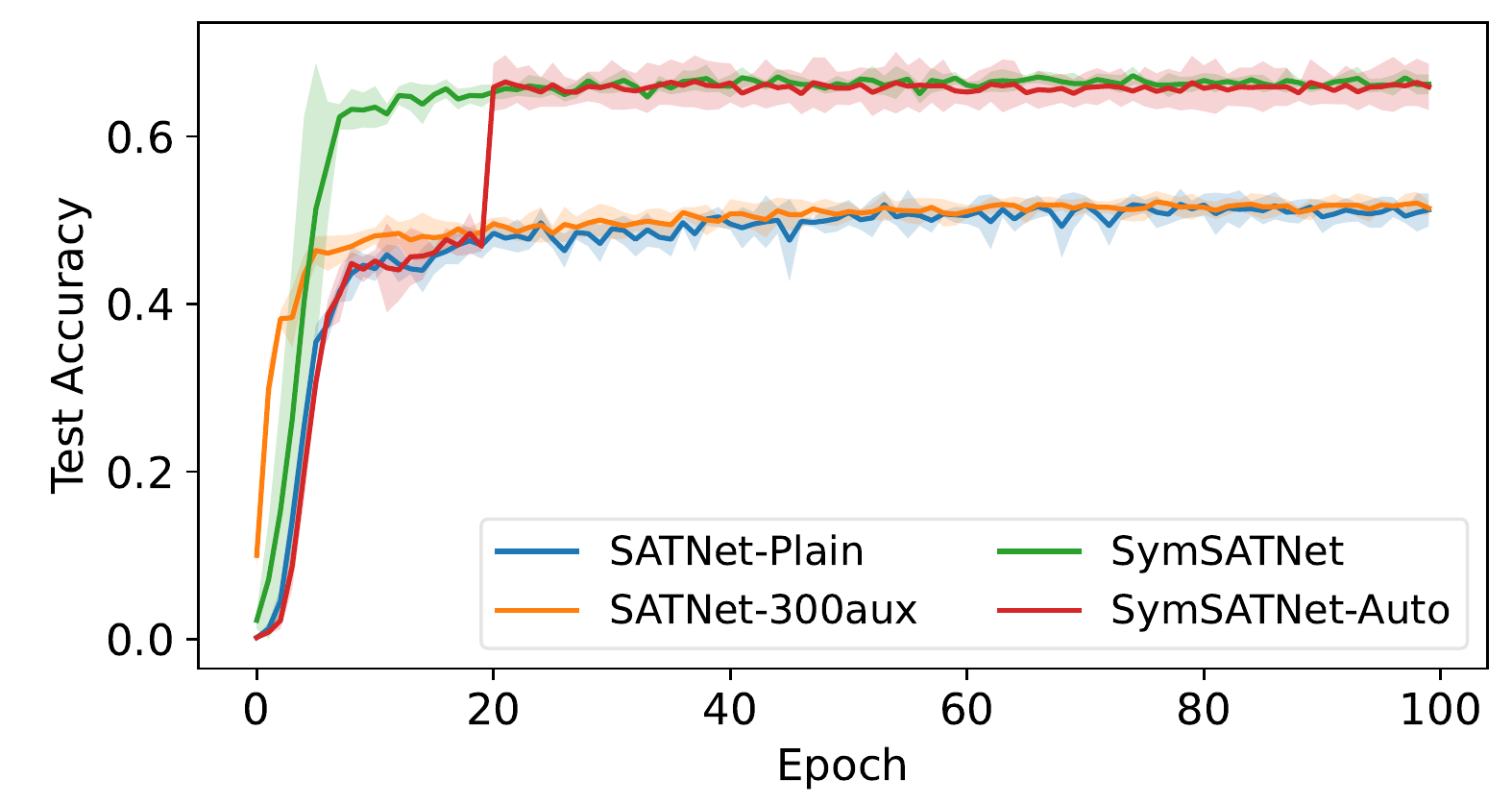}
    \vskip -0.1em
    \subcaption{Normal Rubik's cube $\rightarrow$ Hard Rubik's cube}
    \label{fig:transfer-cube-1}
    \end{subfigure}
    \begin{subfigure}{0.48\columnwidth}
    \includegraphics[width=0.96\columnwidth]{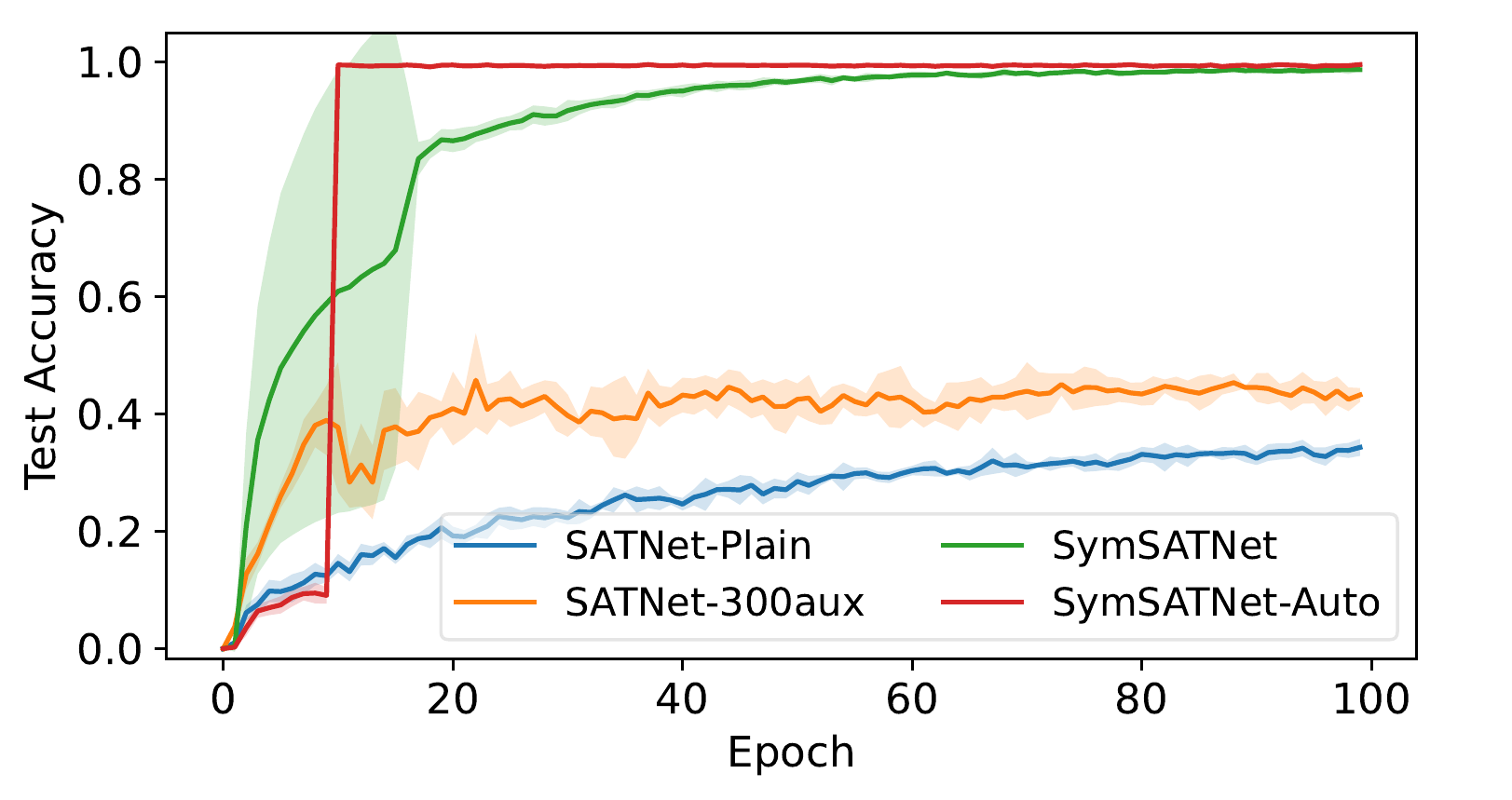}
    \subcaption{Easy Sudoku $\rightarrow$ Hard Sudoku}
    \label{fig:transfer-sudoku-2}
    \end{subfigure}
    \hskip 1.3em
    \begin{subfigure}{0.48\columnwidth}
    \vskip 0.1em
    \includegraphics[width=0.94\columnwidth]{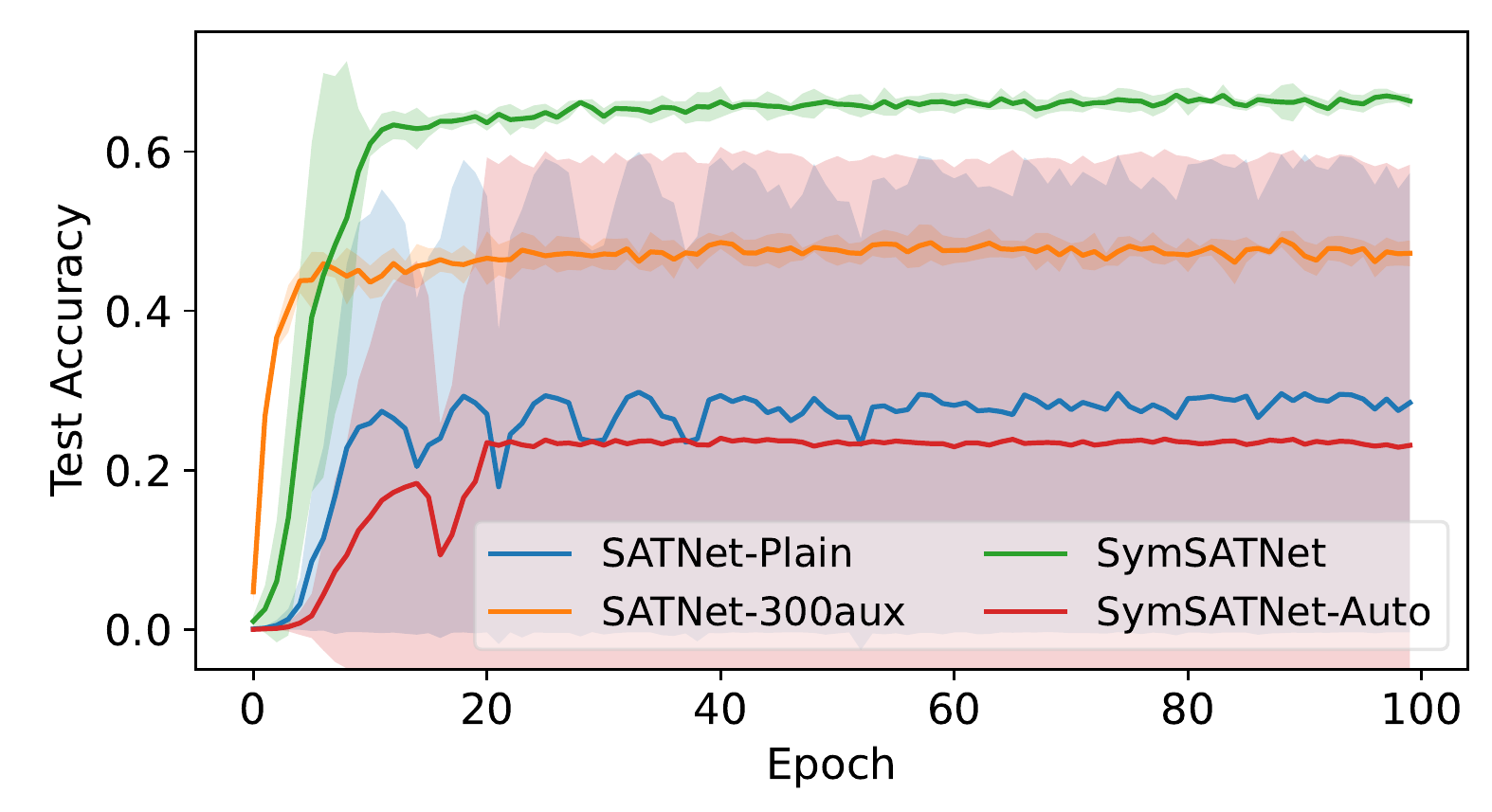}
    \vskip -0.1em
    \subcaption{Easy Rubik's cube $\rightarrow$ Hard Rubik's cube}
    \label{fig:transfer-cube-2}
    \end{subfigure}
    \caption{Transfer learning with various difficulties of training and test examples.
    For both problems, normal and easy examples were used to train, and hard examples were used to test each model.
    }
    \label{fig:transfer}
\end{figure}

\paragraph{Transfer learning}
To test the transferability of SymSATNet, we generated Sudoku and Rubik's cube datasets with varying difficulties, where each dataset consisted of 9K training 
and 1K test examples.  For SymSATNet-Auto, we split the 9K training examples into 8K training and 1K validation examples.
We used three levels of difficulties for Sudoku and Rubik's cube (easy, normal, hard), based on the number of missing cells for Sudoku or missing facelets for Rubik's cube.
The input part of Sudoku examples was generated with 21 (easy), or 31 (normal), or 41 masked cells (hard), and the input part of Rubik's cube examples was generated with 3 (easy), or 4 (normal), or 5 missing facelets (hard).
For both problems, we used the training examples of easy or normal datasets for training, and the test examples of hard datasets for testing. 
We repeated every task in this experiment five times. Here we report the average test accuracies and $95\%$ confidence interval.

Figure~\ref{fig:transfer} shows test accuracies throughout 100 epochs in the four types of transfer learning tasks.
SymSATNet achieved the best result in the whole tasks; it succeeded in solving hard problems after learning from easier examples.
As Figures~\ref{fig:transfer-sudoku-1} and \ref{fig:transfer-sudoku-2} indicate, SymSATNet-Auto exploited the full group symmetries in Sudoku.
For Rubik's cube, Figure~\ref{fig:transfer-cube-1} shows SymSATNet-Auto achieved better performance over the baselines by finding partial symmetries. 
These results show the promise of SymSATNet and SymSATNet-Auto to learn transferable rules even from easier examples.

Note that for the easy Rubik's cube dataset, SymSATNet-Auto showed poor performance (Figure~\ref{fig:transfer-cube-2}). The poor performance comes from the violation of the assumption of $\symfind$; the group symmetries sometimes did not emerge in SATNet in this case. In three out of five trials
with the easy Rubik's cube dataset,
SATNet learnt nothing while producing the $0\%$ test accuracy,
and $\symfind$ returned the trivial group which equated SymSATNet-Auto with SATNet-Plain.
In the remaining two trials, SATNet learnt correct rules, and $\symfind$ and the validation step found correct partial symmetries, which led to the improved performance. These results exhibit the fundamental limitation of SymSATNet-Auto, whose performance strongly depends on the original SATNet.

%% file: conclusion.tex
\section{Conclusion}
\label{sec:conclusion}
We presented SymSATNet, that is capable of exploiting symmetries of the rules or constraints to be learnt by SATNet.
We also described the $\symfind$ algorithm for automatically discovering symmetries from the parameter $C$ of the original SATNet at a fixed training epoch,
which is based on our empirical observation that symmetries emerge during training as duplicated or similar entries of $C$.
Our experimental evaluations with two rule-learning problems related to Sudoku and Rubik's cube show the benefit of SymSATNet and the promise and limitation of $\symfind$. Although components of $\symfind$ are motivated by the theoretical analysis of the space of equivariant matrices, such as Theorem~\ref{def:group-constructors}, $\symfind$ lacks a theoretical justification on its overall performance. One interesting future direction is to fill in this gap by identifying when symmetries emerge during the training of SATNet and proving probabilistic guarantees on when $\symfind$ returns correct group symmetries.

%% file: appendix-visualisation.tex
\begin{figure}[t]
    \centering
    \begin{minipage}{.95\columnwidth}
    \centering
    \includegraphics[width=.95\columnwidth]{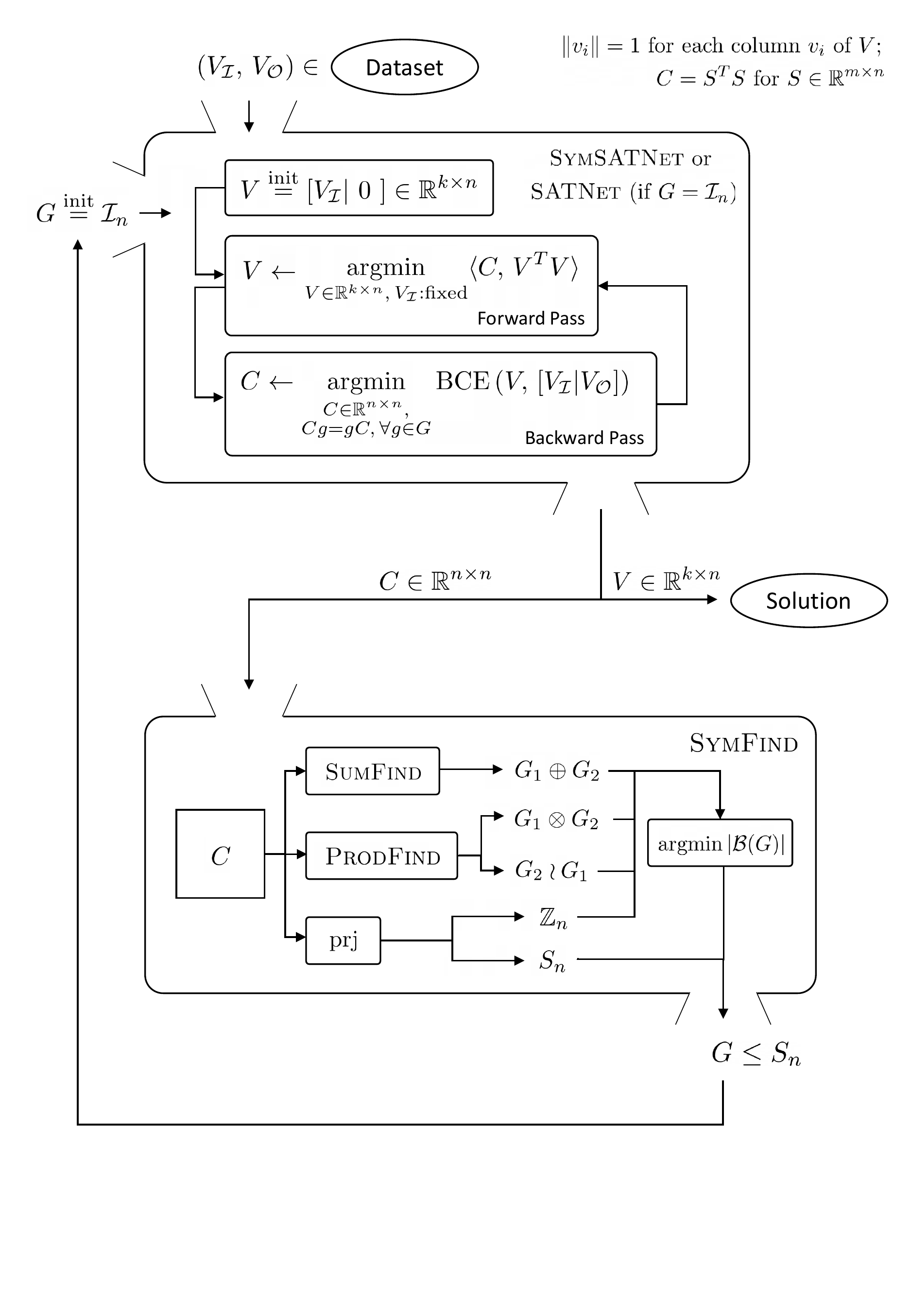}
    \end{minipage}
    \vskip -6em
    \caption{The overall procedure of SATNet, SymSATNet, and $\symfind$.
    The original SATNet takes as input the assignment pairs in the dataset, and learns the parameter matrix $C$ which describes the logical rules to solve.
    Our $\symfind$ algorithm receives the learnt parameter $C$ of SATNet,
    and uses the subroutine algorithms $\sumsplit$ and $\kronsplit$ and the Reynolds operator $\prj$ to find the candidate groups,
    and returns the strongest group $G$ among the candidates.
    Finally, SymSATNet exploits the groups symmetries in logical rules provided by domain experts, or automatically detected group $G$ by $\symfind$.}
    \label{fig:visualisation}
\end{figure}

\clearpage

%% file: appendix-notation.tex
\section{Notation}
For natural numbers $n_1$ and $n_2$ with $n_1 \leq n_2$, we write $[n_1:n_2]$ for the set $\{n_1, n_1+1,\ldots, n_2\}$.
For an $m\times m$ matrix $M$, $1 \leq a \leq b \leq m$, and $1 \leq c \leq d \leq m$,
we define $M[a:b,c:d]$ by a $(b-a+1) \times (d-c+1)$ matrix $M$:
\begin{align*}
M[a:b, c:d]_{i,j} = M_{a+i-1, c+j-1} \quad \text{for } i \in [b-a+1] \text{ and } j \in [d-c+1].
\end{align*}
Also, if $H$ is a subgroup of $G$, we denote it by $H \le G$.

%% file: appendix-basis.tex
\section{Proof of Theorem~\ref{thm:equivariant-basis}}
\label{sec:equivariant-basis}
In this section, we prove Theorem~\ref{thm:equivariant-basis}.
Here we prove the direct sum, direct product, and wreath product cases by an argument similar to the previous work~\citep{Wang20}.
For the wreath product case, we slightly generalise the previous results~\citep{Wang20}, which considered only transitive group actions.
We will use the notation from Theorem~\ref{thm:equivariant-basis}.
Let $G$ and $H$ be permutation groups over $[p]$ and $[q]$, respectively.
Also, let $A \in \cB(G)$, $B \in \cB(H)$, $O \in \cO(G)$, and $O', O''\in \cO(H)$.

\begin{claim}
\label{claim:equivariant-basis-1}
The matrices $A \oplus \mathbf{0}_q$, $\mathbf{0}_p \oplus B$, $\mathbf{1}_{O \times (p + O')}$, and $\mathbf{1}_{(p + O) \times O'}$ are $(G \oplus H)$-equivariant.
Also, the matrices of these types for all possible choices of $A$, $B$, $O$, and $O'$ form a linearly independent set.
\end{claim}

\begin{proof}
Pick arbitrary group elements $g \in G$ and $h \in H$. 
Let $P_g$ and $P_h$ be the permutation matrices corresponding to $g$ and $h$, respectively. 

We prove the claimed equivariance below:
\begin{align*}
(P_g \oplus P_h) (A \oplus \mathbf{0}_p) &= (P_g A) \oplus \mathbf{0}_p = (A P_g) \oplus \mathbf{0}_p = (A \oplus \mathbf{0}_p) (P_g \oplus P_h), \\
(P_g \oplus P_h) (\mathbf{0}_q \oplus B) &= \mathbf{0}_q \oplus (P_h B) = \mathbf{0}_q \oplus (B P_h) = (\mathbf{0}_q \oplus B) (P_g \oplus P_h), \\
(P_g \oplus P_h) \mathbf{1}_{O \times (p + O')} &= \mathbf{1}_{O \times (p + O')} = \mathbf{1}_{O \times (p + O')} (P_g \oplus P_h), \\
(P_g \oplus P_h) \mathbf{1}_{(p + O') \times O} &= \mathbf{1}_{(p + O') \times O} = \mathbf{1}_{(p + O') \times O} (P_g \oplus P_h).
\end{align*}
The third and fourth lines use the fact that $\mathbf{1}_{O \times (p + O')}$ and $\mathbf{1}_{(p + O') \times O}$ are invariant under the left or right multiplication of the permutation matrix $P_g \oplus P_h$. This holds because the orbits of $G$ are preserved by any permutation in $G$, and those of $H$ are preserved by all permutations in $H$, so that 
\begin{align}
\label{eqn:orbit-invariance}
P_g \mathbf{1}_{O \times O'} = \mathbf{1}_{O \times O'} = \mathbf{1}_{O \times O'} P_h, \qquad
P_h \mathbf{1}_{O' \times O} = \mathbf{1}_{O' \times O} = \mathbf{1}_{O' \times O} P_g.
\end{align}

Next, we show the claimed linear independence by analysing the indices of the nonzero entries of the four types of matrices in the claim.
\begin{compactenum}
\item The matrices of the type $A \oplus \mathbf{0}_q$ for some $A \in \cB(G)$ are linearly independent, since their $A$ parts are linearly independent.
\item The matrices of the type $\mathbf{0}_p \oplus B$ for some $B \in \cB(H)$ are also linearly independent by similar reason.
\item Different matrices of the type $\mathbf{1}_{O \times (p + O')}$ for some $O$ and $O'$ do not share an index of a nonzero entry, since different orbits of a group are disjoint. Thus, the matrices of this type are linearly independent.
\item Different matrices of the type $\mathbf{1}_{(p + O) \times O'}$ for some $O$ and $O'$ do not share an index of a nonzeron entry. Thus, the matrices of this form are linearly independent.
\end{compactenum}
Also, any matrices of the above four types form a linearly independent set because those linear combinations do not share any indices of nonzero entries.
From this and the above reasoning for each of the four types of matrices it follows that the matrices of those four types are linearly independent, as claimed.
\end{proof}

\begin{claim}
\label{claim:equivariant-basis-2}
The matrix $A \otimes B$ is $(G \otimes H)$-equivariant.
Also, the matrices of this shape for all possible choices of $A$ and $B$ form a linearly independent set.
\end{claim}

\begin{proof}
Pick arbitrary group elements $g \in G$ and $h \in H$.  Let $P_g$ and $P_h$ be the permutation matrices corresponding to $g$ and $h$, respectively. Then,
\begin{align*}
(P_g \otimes P_h) (A \otimes B) &= (P_g A) \otimes (P_h B) = (A P_g) \otimes (B P_h) = (A \otimes B) (P_g \otimes P_h).
\end{align*}
For linear independence, we can prove it using the fact $\langle A \otimes B, \, A' \otimes B' \rangle = \langle A, \, A' \rangle \cdot \langle B, \, B' \rangle$.
\end{proof}

\begin{claim}
\label{claim:equivariant-basis-3}
Recall $A \in \cB(G)$ and $B \in \cB(H)$. The matrices $A \otimes \mathbf{1}_{O' \times O''}$ such that $A_{i,i} = 0$ for $i \in [p]$ and $I_O \otimes B$ are $(H \wr G)$-equivariant.
Also, the matrices of these types for all possible choices of $A$, $B$, $O$, $O'$, and $O''$ form a linearly independent set.
\end{claim}

\begin{proof}
Pick arbitrary group elements $g \in G$ and $\vec{h} = (h_1, \, \ldots, \, h_p) \in H^p$.
Let $P_{g^{-1}}$ and $P_{h_i}$ be the permutation matrices corresponding to $g^{-1}$ and $h_i$ for $i \in [p]$.
We can express $\wreath(\vec{h}, g)$ and $P_{g^{-1}}$ by
\begin{align*}
\wreath(\vec{h}, g) = \sum_{i=1}^p \mathbf{1}_{\{i\} \times \{g(i)\}} \otimes P_{h_i}, \qquad P_{g^{-1}} = \sum_{i=1}^p \mathbf{1}_{\{i\} \times \{g(i)\}}.
\end{align*}
Then, we prove the following equivariance:
\begin{align}
\wreath(\vec{h}, g) \left( A \otimes \mathbf{1}_{O' \times O''} \right)
\nonumber
&= \left( \sum_{i=1}^p \mathbf{1}_{\{i\} \times \{g(i)\}} \otimes P_{h_i} \right) \left( A \otimes \mathbf{1}_{O' \times O''} \right) \\
\nonumber
&= \sum_{i=1}^p \left( \mathbf{1}_{\{i\} \times \{g(i)\}} A \otimes P_{h_i} \mathbf{1}_{O' \times O''} \right) \\
\label{eqn:wreath-basis-1}
&= \left( \sum_{i=1}^p \mathbf{1}_{\{i\} \times \{g(i)\}} \right) A \otimes \mathbf{1}_{O' \times O''} \\
\nonumber
&= P_{g^{-1}} A \otimes \mathbf{1}_{O' \times O''} = A P_{g^{-1}} \otimes \mathbf{1}_{O' \times O''} \\
\nonumber
&= A \left( \sum_{i=1}^p \mathbf{1}_{\{i\} \times \{g(i)\}} \right) \otimes \mathbf{1}_{O' \times O''} \\
\label{eqn:wreath-basis-2}
&= \sum_{i=1}^p \left( A \, \mathbf{1}_{\{i\} \times \{g(i)\}} \otimes \mathbf{1}_{O' \times O''} P_{h_i} \right) \\
\nonumber
&= \left( A \otimes \mathbf{1}_{O' \times O''} \right) \left( \sum_{i=1}^p \mathbf{1}_{\{i\} \times \{g(i)\}} \otimes P_{h_i} \right) \\
\nonumber
&= \left( A \otimes \mathbf{1}_{O' \times O''} \right) \wreath(\vec{h}, g),
\end{align}
\begin{align}
\wreath(\vec{h}, g) \left( I_O \otimes B \right)
\nonumber
&= \left( \sum_{i=1}^p \mathbf{1}_{\{i\} \times \{g(i)\}} \otimes P_{h_i} \right) \left( I_O \otimes B \right) \\
\nonumber
&= \sum_{i=1}^p \left( \mathbf{1}_{\{i\} \times \{g(i)\}} I_O \otimes P_{h_i} B \right) \\
\nonumber
&= \sum_{i=1}^p \left( \mathbf{1}_{\{i\} \times (\{g(i)\} \cap O)} \otimes P_{h_i} B \right) \\
\label{eqn:wreath-basis-3}
&= \sum_{i=1}^p \left( \mathbf{1}_{(\{i\} \cap O) \times \{g(i)\}} \otimes B P_{h_i} \right) \\
\nonumber
&= \sum_{i=1}^p \left( I_O \mathbf{1}_{\{i\} \times \{g(i)\}} \otimes B P_{h_i} \right) \\
\nonumber
&= \left( I_O \otimes B \right) \left( \sum_{i=1}^p \mathbf{1}_{\{i\} \times \{g(i)\}} \otimes P_{h_i} \right) \\
\nonumber
&= \left( I_O \otimes B \right) \wreath(\vec{h}, g).
\end{align}
Here, \eqref{eqn:wreath-basis-1} and \eqref{eqn:wreath-basis-2} use the same argument in \eqref{eqn:orbit-invariance},
and \eqref{eqn:wreath-basis-3} uses the fact that $i \in O \iff g(i) \in O$ for any $g \in G$ and $O \in \cO(G)$.
For linear independence, we again use the fact $\langle A \otimes B, \, A' \otimes B' \rangle = \langle A, \, A' \rangle \cdot \langle B, \, B' \rangle$ to show the orthogonality of all possible matrices of types $A \otimes \mathbf{1}_{O' \times O''}$ and $I_O \otimes B$.
\end{proof}

\begin{claim}
\label{claim:equivariant-basis-4}
The number of basis elements of $G \oplus H$, $G \otimes H$, and $H \wr G$ can be computed as follows:
\begin{align*}
|\cB(G \oplus H)| &= |\cB(G)| + |\cB(H)| + 2 |\cO(G)| |\cO(H)|, \\
|\cB(G \otimes H)| &= |\cB(G)| |\cB(H)|, \\
|\cB(H \wr G)| &= |\cO(G)||\cB(H)| + \left( |\cB(G)| - |\cO(G)| \right) |\cO(H)|^2.
\end{align*}
\end{claim}

\begin{proof}
First, we derive some general fact on a finite permutation group.
Let $\mathbf{G}$ be a permutation group on $[r]$ for some $r$.
Consider the action of $\mathbf{G}$ on the space of $r$-dimensional vectors $\cX = \R^r$ by the row permutation action $g \cdot v = P_g v$ for any $g \in \mathbf{G}$.
Define $\cF(\mathbf{G}) = \{v \in \cX : g \cdot v = v, \, \forall g \in \mathbf{G}\}$.
Consider the following linear operator 
\begin{align*}
\phi_\mathbf{G} & : \cX \to \cX,
&
\phi_\mathbf{G}(v) & = \dfrac{1}{|\mathbf{G}|} \sum\limits_{g \in \mathbf{G}} g \cdot v,
\end{align*}
which can also be represented as the following matrix:
\[
\phi_\mathbf{G} = \dfrac{1}{|\mathbf{G}|} \sum\limits_{g \in \mathbf{G}} P_g.
\]
The operator $\phi_\mathbf{G}$ is a projection map, since
\begin{align*}
\phi_\mathbf{G}(\phi_\mathbf{G}(v))
&= \dfrac{1}{|\mathbf{G}|^2} \sum\limits_{g_1 \in \mathbf{G}} \sum\limits_{g_2 \in \mathbf{G}} g_1 \cdot (g_2 \cdot v) \\
&= \dfrac{1}{|\mathbf{G}|^2} \sum\limits_{g_1 \in \mathbf{G}} \sum\limits_{g \in \mathbf{G}} g \cdot v \\
&= \dfrac{1}{|\mathbf{G}|} \sum\limits_{g \in G} g \cdot v.
\end{align*}
Also, 
we have $\image(\phi_\mathbf{G}) = \cF(\mathbf{G})$ where $\image(f)$ is the image of the function $f$.
Now, by noting that a linear projection map has only eigenvalues $0$ and $1$, the dimension of $\cF(\mathbf{G})$ can be computed by the sum of eigenvalues of $\phi_\mathbf{G}$, i.e., $\trace(\phi_\mathbf{G})$.
Also, using Burnside's lemma, we can count the orbits of $\mathbf{G}$ (acting on $[r]$) by
\begin{align*}
|\cO(\mathbf{G})| &= \dfrac{1}{|\mathbf{G}|} \sum\limits_{g \in \mathbf{G}} \trace(P_g) \\
&= \trace \left( \dfrac{1}{|\mathbf{G}|} \sum\limits_{g \in \mathbf{G}} P_g \right) \\
&= \trace (\phi_\mathbf{G}).
\end{align*}
Putting together, we get 
\begin{equation}
\label{eqn:dim-tr-orbit}
\dim \cF(\mathbf{G}) = \trace(\phi_\mathbf{G}) = |\cO(\mathbf{G})|.
\end{equation}

Next, we instantiate what we have just shown above for the following case that $\mathbf{G}$ is the following group:
\[
\mathbf{G}_0^{\otimes 2} = \{g \otimes g : g \in \mathbf{G}_0\}
\]
for some permutation group $\mathbf{G}_0$ on $[n]$. Note that $\mathbf{G}$ is a permutation group on $[n^2]$. As explained above, $\mathbf{G}_0^{\otimes 2}$ can act on the space of $n^2$-dimensional vectors $\R^{n^2}$. By vectorizing matrices, we can express the space $\cE(\mathbf{G}_0)$ of $\mathbf{G}_0$-equivariant linear maps on $\R^n$
by
\begin{align*}
\ovec(\cE(\mathbf{G}_0)) &= \left\{ \ovec(M) : P_g M P_g^T = M, \, \forall g \in \mathbf{G}_0 \right\} \\
&= \left\{ \ovec(M) : (P_g \otimes P_g) \ovec(M) = \ovec(M), \, \forall g \in \mathbf{G}_0 \right\} \\
&= \left\{ \ovec(M) : (g \otimes g) \cdot \ovec(M) = \ovec(M), \, \forall g \in \mathbf{G}_0 \right\} \\
&= \cF(\mathbf{G}_0^{\otimes 2}).
\end{align*}
Thus, $|\cB(\mathbf{G}_0)| = \dim \ovec(\cE(\mathbf{G}_0)) = \dim \cF(\mathbf{G}_0^{\otimes 2})$. We now calculate the dimension of
$\cF(\mathbf{G}_0^{\otimes 2})$ using the relationship in \eqref{eqn:dim-tr-orbit}:
\begin{align}
\nonumber
|\cB(\mathbf{G}_0)| &= \dim \cF(\mathbf{G}_0^{\otimes 2})
= \trace(\phi_{\mathbf{G}_0^{\otimes 2}}) \\ 
&= \dfrac{1}{|\mathbf{G}_0^{\otimes 2}|} \sum\limits_{{g \otimes g} \in \mathbf{G}_0^{\otimes 2}} \trace(P_{g \otimes g})
= \frac{1}{|\mathbf{G}_0|} \sum\limits_{g \in \mathbf{G}_0} \trace(P_g)^2. 
\label{eqn:basis-trace}
\end{align}

Finally, we complete the proof by calculating \eqref{eqn:basis-trace} for $\mathbf{G}_0 = G \oplus H$, $\mathbf{G}_0 =  G \otimes H$, and $\mathbf{G}_0 = H \wr G$:
\pagebreak
\begin{align*}
|\cB(G \oplus H)| &= \dfrac{1}{|G \oplus H|} \sum\limits_{g \oplus h \in G \oplus H} \trace(P_{g \oplus h})^2 \\
&= \dfrac{1}{|G||H|} \sum\limits_{g \in G} \sum\limits_{h \in H} (\trace(P_g) + \trace(P_h))^2 \\
&= \dfrac{1}{|G|} \sum\limits_{g \in G} \trace(P_g)^2 + \dfrac{1}{|H|} \sum\limits_{h \in H} \trace(P_h)^2 
+ \dfrac{2}{|G||H|} \left( \sum\limits_{g \in G} \sum\limits_{h \in H} \trace(P_g) \trace(P_h) \right) \\
&= |\cB(G)| + |\cB(H)| + 2 |\cO(G)| |\cO(H)|, 
\\
\\
|\cB(G \otimes H)| &= \dfrac{1}{|G \otimes H|} \sum\limits_{g \otimes h \in G \otimes H} \trace(P_{g \otimes h})^2 \\
&= \dfrac{1}{|G||H|} \sum\limits_{g \in G} \sum\limits_{h \in H} \trace(P_g)^2 \trace(P_h)^2 \\
&= |\cB(G)| |\cB(H)|,
\\
\\
|\cB(H \wr G)| &= \dfrac{1}{|H \wr G|} \sum\limits_{\wreath(\vec{h}, g) \in H \wr G} \trace(\wreath(\vec{h}, g))^2 \\
&= \dfrac{1}{|G||H|^p} \sum\limits_{g \in G, (h_1, \ldots, h_p) \in H^p} \trace \left( \sum\limits_{i=1}^p \mathbf{1}_{\{i\} \times \{g(i)\}} \otimes P_{h_i} \right)^2 \\
&= \dfrac{1}{|G||H|^p} \sum\limits_{g \in G, (h_1, \ldots, h_p) \in H^p} \left( \sum\limits_{i=1}^p \trace \left( \mathbf{1}_{\{i\} \times \{g(i)\}} \otimes P_{h_i} \right) \right)^2 \\
&= \dfrac{1}{|G||H|^p} \sum\limits_{g \in G, (h_1, \ldots, h_p) \in H^p} \left( \sum\limits_{i=1}^p \ind\nolimits_{\{i = g(i)\}} \trace(P_{h_i}) \right)^2 \\
&= \dfrac{1}{|G||H|^p} \sum\limits_{g \in G, (h_1, \ldots, h_p) \in H^p} \left( \sum\limits_{i=1}^p \ind\nolimits_{\{i = g(i)\}} \trace(P_{h_i})^2 + \sum\limits_{i \neq j} \ind\nolimits_{\{i = g(i), j = g(j)\}} \trace(P_{h_i}) \trace(P_{h_j}) \right) \\
&= \dfrac{1}{|G||H|^p} \sum\limits_{g \in G} \left\{ \sum\limits_{i=1}^p \ind\nolimits_{\{i = g(i)\}} \left( \sum\limits_{(h_1, \ldots, h_p) \in H^p} \trace(P_{h_i})^2 \right) \right. \\
&\hspace{22mm} + \left. \sum\limits_{i \neq j} \ind\nolimits_{\{i = g(i), j = g(j)\}} \left( \sum\limits_{(h_1, \ldots, h_p) \in H^p} \trace(P_{h_i}) \trace(P_{h_j}) \right) \right\} \\
&= \dfrac{1}{|G||H|^p} \sum\limits_{g \in G} \left\{ \left( \sum\limits_{i=1}^p \ind\nolimits_{\{i = g(i)\}} \right) |H|^p |\cB(H)| + \left( \sum\limits_{i \neq j} \ind\nolimits_{\{i = g(i), j = g(j)\}} \right) |H|^p |\cO(H)|^2 \right\} \\
&= \dfrac{1}{|G|} \sum\limits_{g \in G} \left\{ \left( \sum\limits_{i=1}^p \ind\nolimits_{\{i = g(i)\}} \right) |\cB(H)| + \left( \sum\limits_{i=1}^p \ind\nolimits_{\{i=g(i)\}} \sum\limits_{j=1}^p \ind\nolimits_{\{j=g(j)\}} - \sum\limits_{i=1}^p \ind\nolimits_{\{i=g(i)\}} \right) |\cO(H)|^2 \right\} \\
&= \dfrac{1}{|G|} \sum\limits_{g \in G} \left( \trace(P_g) |\cB(H)| + \left( \trace(P_g)^2 - \trace(P_g) \right) |\cO(H)|^2 \right) \\
&= |\cO(G)||\cB(H)| + \left( |\cB(G)| - |\cO(G)| \right) |\cO(H)|^2.
\end{align*}
\end{proof}

\paragraph{Proof of Theorem~\ref{thm:equivariant-basis}}
Claims~\ref{claim:equivariant-basis-1}, \ref{claim:equivariant-basis-2}, and \ref{claim:equivariant-basis-3} show that
each set of matrices on the right-hand side of the equations in Theorem~\ref{thm:equivariant-basis} consists of linearly independent matrices,
and it is contained in the corresponding space of equivariant linear maps.
Furthermore, Claim~\ref{claim:equivariant-basis-4} shows that the matrices in the set span the whole of the space of equivariant linear maps,
since their number coincides with the dimension of the space.
Hence, these three claims complete the proof of the theorem.

%% file: appendix-symsatnet.tex
\section{Proofs of Theorem~\ref{thm:objective-invariance-equivariance}
and Lemma~\ref{lem:constraint-invariance-equivariance}} 
\label{sec:appendix-symsatnet}

\paragraph{Proof of Theorem~\ref{thm:objective-invariance-equivariance}}
Assume that $C$ is $G$-equivariant. Pick any $g \in G$. Then, $C$ preserves the standard action of $g$ on $\R^n$ via its permutation matrix. Thus, $C g = g C$, where $g$ denotes the permutation matrix $P_g$. This equation is equivalent to
\begin{equation}
\label{eqn:objective-invariance-equivariance1}
g^T C g = C,
\end{equation}
since $g^T = g^{-1}$. Using this equality, we show that the optimisation objective is invariant with respect to $g$'s action:
\begin{align*}
\langle C, (V g^{-1})^T (V g^{-1})\rangle & 
{} = \langle C, (g^{-1})^T (V^TV) g^{-1} \rangle
\\
&
{} = \trace(C^T g (V^TV) g^T))
\\
& 
{} = \trace((g^T C^T g) (V^TV))
\\
&
{} = \trace((g^T C g)^T (V^TV))
\\ 
&
{} = \langle g^T C g, V^TV \rangle
\\
&
{} = \langle C, V^TV \rangle.
\end{align*}
Here $\trace$ is the usual trace operator on matrices, the third equality uses the cyclic property of $\trace$,
and the last equality uses \eqref{eqn:objective-invariance-equivariance1}.

We move on to the proof of the converse. Assume $k = n$ and the equation \eqref{eqn:objective-invariance-equivariance0}
holds for every $V \in \R^{k\times n}$ and $g \in G$. Pick any $g \in G$. Then, for all $V \in \R^{k \times n}$, 
\begin{align*}
\langle C, V^T V \rangle 
& 
{} = \langle C, (V g^{-1})^T (V g^{-1})\rangle 
{} =  \langle g^T C g, V^TV \rangle.
\end{align*}
The first equality follows from \eqref{eqn:objective-invariance-equivariance0}, and the second from our derivation above. But the space of matrices of the form $V^T V$ is precisely that of $n\times n$ symmetric positive semi-definite matrices, which is known to contain $n(n+1)/2$ independent elements. Since $n(n+1)/2$ is precisely the dimension of the space of $n \times n$ symmetric matrices, the above equality for every $V \in R^{k\times n}$ implies that $C = g^T C g$, that is, $C g = g C$. This gives the $G$-equivariance of $C$, as desired.

\paragraph{Proof of Lemma~\ref{lem:constraint-invariance-equivariance}}
Pick arbitrary $V \in R^{k \times n}$ and $g \in G$. Let $v_1,\ldots,v_n$ be the columns of $V$ in order, and
$v'_1,\ldots,v'_n$ be those of $V g^{-1}$ in order.
Then, $\Vert v_i \Vert = 1$ for every $i \in [n]$ if and only if
$\Vert v_{g^{-1}(i)} \Vert = 1$ for all $i \in [n]$. The latter is equivalent to the statement that $\Vert v'_i \Vert = 1$ for every $i \in [n]$.

%% file: appendix-symfind.tex
\section{Subroutines of $\symfind$}
\label{sec:symfind-appendix}
In this section, we describe the subroutines $\sumsplit$ and $\kronsplit$ of our $\symfind$ algorithm.

\subsection{Detection of direct sum}
We recall the distributivity law of group constructors \citep{Dummit99, Majumdar12}:
\begin{align*}
G \otimes (H \oplus H') &= (G \otimes H) \oplus (G \otimes H'),
&
(G \oplus G') \otimes H &= (G \otimes H) \oplus (G' \otimes H), \\
G \wr (H \oplus H') &= (G \wr H) \oplus (G \wr H'), 
&
(G \oplus G') \wr H &= (G \wr H) \oplus (G' \wr H).
\end{align*}
By this law and the fact that $\idg_m = \bigoplus\limits_{i = 1}^{m} \cyclic_1$ for all $m$, any permutation group $G$ definable in our grammar can be expressed as a direct sum
\begin{align}
\label{eqn:fully-factorised-sum-product-representation}
G = \bigoplus_{i = 1}^{N} {H_{(i, 1)} \diamond \cdots \diamond H_{(i, n_i)}}
\end{align}
where each $H_{(i, j)} \in \{ \cyclic_{m(i, j)}, \, \perm_{m(i, j)} \}$ and $\diamond$ is either $\otimes$ or $\wr$. Let 
\[
H_i = {H_{(i, 1)} \diamond \cdots \diamond H_{(i, n_i)}},
\]
and $p_i$ be the natural number such that $H_i$ is a permutation group on $[p_i]$. Since both $\cyclic_m$ and $\perm_m$ have only one orbit for all $m$, each $H_i$ also has only one orbit.
Hence, by Theorem~\ref{thm:equivariant-basis},
\[
\cB(G) = \bigcup_{i,j \in [N]} B_{ij},
\]
where
\begin{align*}
B_{ii} & = \{ \mathbf{0}_{p_1 + \cdots + p_{i-1}} \oplus A \oplus \mathbf{0}_{p_{i+1} + \cdots + p_N} : A \in \cB(H_i) \} 
\text{ for } i \in [N], \text{ and}
\\
B_{ij} & = \{ \mathbf{1}_{(p_1 + \cdots + p_{i-1} + [p_i]) \times (p_1 + \cdots + p_{j-1} + [p_j])} \} \text{ for } i,j \in [N] \text{ with } i \neq j.
\end{align*}
This means that all off-diagonal blocks of any $G$-equivariant matrices are constant matrices, just like the matrix shown in (b) of Figure~\ref{fig:block-pattern}.

Given a matrix $M \in \R^{m\times m}$, the subroutine $\sumsplit$ aims at finding a permutation $\sigma : [m] \to [m]$ and the split $m = p_1 + \ldots + p_N$ for $p_1,\ldots,p_N > 0$ such that $P_{\sigma}^TMP_{\sigma}$ is approximately $G$-equivariant for some permutation group $G$ on $[m]$ and this group $G$ has the form in \eqref{eqn:fully-factorised-sum-product-representation} where
$H_i = {H_{(i, 1)} \diamond \cdots \diamond H_{(i, n_i)}}$ is a permutation
group on $[p_i]$. The result of the subroutine is $(G,\sigma)$.

\begin{figure}
    \centering
    \begin{subfigure}{0.48\columnwidth}
    \centering
    \includegraphics[width=0.6\linewidth]{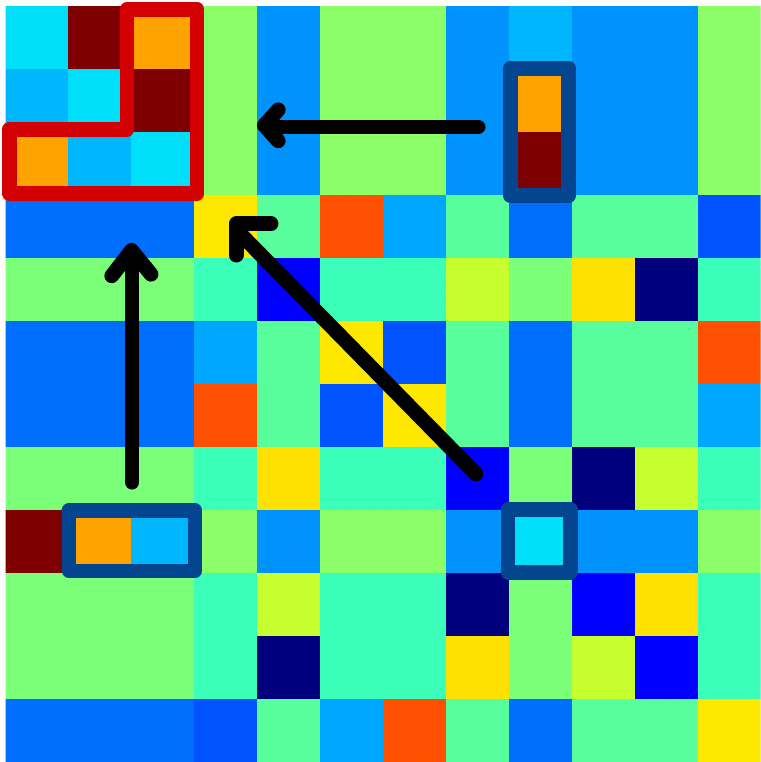}
    \caption{Original matrix $M$}
    \end{subfigure}
    \begin{subfigure}{0.48\columnwidth}
    \centering
    \includegraphics[width=0.6\linewidth]{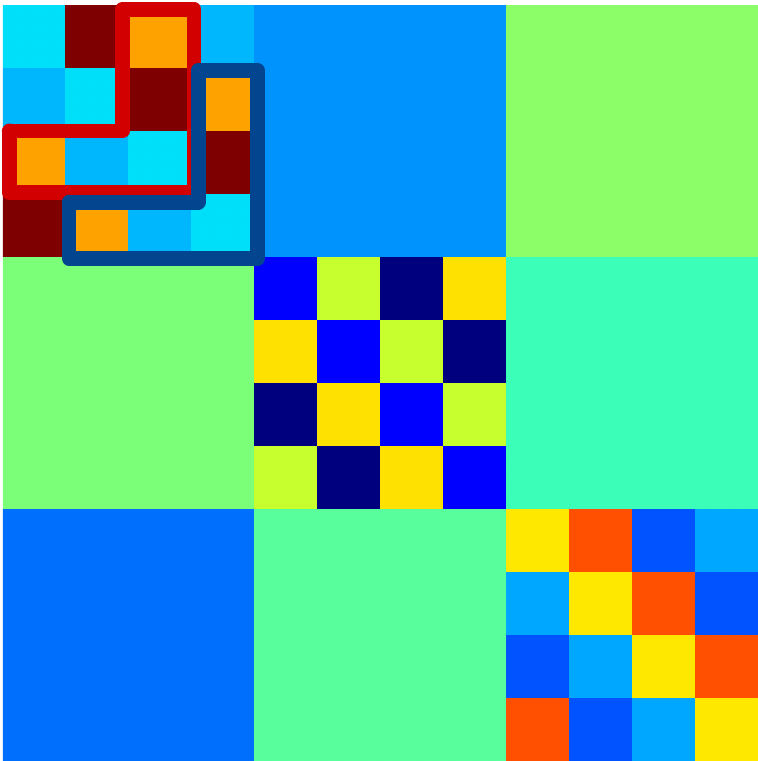}
    \caption{Permuted matrix $M_{\sigma^{-1}} = P_\sigma^TM P_\sigma$}
    \end{subfigure}
    \caption{Block pattern which commonly appears in the equivariant matrices under a direct-sum group. In this example, $M_{\sigma^{-1}} \in \cE(\cyclic_4 \oplus \cyclic_4 \oplus \cyclic_4)$.
    The red and blue "L-shaped" clusters have the same values.}
    \label{fig:block-pattern}
\end{figure}

$\sumsplit$ achieves its aim using two key observations.
The first is an important implication of the $G$-equivariance condition on $M_{\sigma^{-1}} = P_{\sigma}^TM P_\sigma$ that we mentioned above:
when the split $m = p_1 + \ldots + p_N$ is used to group entries of $M_{\sigma^{-1}}$ into blocks,
the $G$-equivalence of $M_{\sigma^{-1}}$ implies that the off-diagonal blocks of $M_{\sigma^{-1}}$ are constant matrices.
This implication suggests one approach:
for every permutation $\sigma$ on $[m]$, construct the matrix $M_{\sigma^{-1}}$,
and find a split $m = p_1 + \ldots + p_N$ for some $N$ such that off-diagonal blocks of $M_{\sigma^{-1}}$ from this split are constant matrices.
Note that this approach is not a feasible option in practice, though, since there are exponentially many permutations $\sigma$.
The second observation suggests a way to overcome this exponential blow-up issue of the approach.
It is that when $M_{\sigma^{-1}}$ is $G$-equivariant, in many cases its diagonal blocks are cyclic in the sense that the $(i,j)$-th entry of a block is the same as the $(i+1,j+1)$-th entry of the block.
This pattern can be used to search for a good permutation $\sigma$ efficiently. 

$\sumsplit$ works as follows. It first finds a permutation
$\sigma: [m] \to [m]$ and a split $m = p_1 + \ldots + p_N$ using the process that we will explain shortly. Figure~\ref{fig:block-pattern} illustrates the input matrix $M$, and its permuted $M_{\sigma^{-1}}$ that has nine blocks induced by the split of $m = p_1 + p_2 + p_3$. Then, $\sumsplit$ 
calls $\symfind$ recursively on each diagonal block of $M_{\sigma^{-1}}$ and gets, for every $i \in [N]$,
\[
(H_i, \sigma_i) = \symfind(M_{\sigma^{-1}}[p_{i-1}+1:p_i, p_{i-1}+1:p_i])
\]
where $p_0 = 0$. Finally, $\sumsplit$ returns
\begin{align*}
\left( \bigoplus_{i=1}^N H_i, \, \sigma \circ \left( \bigoplus_{i=1}^N \sigma_i \right) \right)
\end{align*}
where 
\begin{align*}
\sigma_i \oplus \sigma_j & : [p_i + p_j] \to [p_i + p_j], \\
(\sigma_i \oplus \sigma_j)(a) & =
\begin{cases}
    \sigma_i (a) & \text{if} \ a \in [p_i], \\
    \sigma_j (a - p_i) + p_i & \text{otherwise}.
\end{cases}
\end{align*}


We now explain the first part of $\sumsplit$ that finds a permutation $\sigma$ and a split $m = p_1+\ldots + p_N$. For simplicity, we ignore the issue of noise, and present a simpler version that uses equality instead of approximate equality (i.e., being close enough). We start by
initialising $\sigma = \sigma_I$, the identity permutation, and
$M_{\sigma^{-1}} = M$.  Then, we pick an index $k_1$ from the set $ K = \{ k_1 \in [2:m] : (M_{\sigma^{-1}})_{k_1, k_1} = (M_{\sigma^{-1}})_{1,1} \} $.
Then, we locate $(M_{\sigma^{-1}})_{1,k_1}$ in the $(1, \, 2)$-th entry by swapping the indices $2$ and $k_1$, which can be done by updating
\begin{align*}
M_{\sigma^{-1}} &\gets P_{\sigma_{(2,k_1)}} M_{\sigma^{-1}} P_{\sigma_{(2,k_1)}}^T, \\
\sigma &\gets \sigma \circ \sigma_{(2,k_1)}^{-1}.
\end{align*}
Here  $\sigma_{(a,b)}(a) = b$, $\sigma_{(a,b)}(b) = a$, and $\sigma_{(a,b)}(c) = c$ for $c \neq a, \, b$.
Next, we find a new index $k_2 \in [3:m]$ which is suitable to swap with the index $3$ in the updated $M_{\sigma^{-1}}$.
As Figure~\ref{fig:block-pattern} shows, we need to find $k_2$ such that $(M_{\sigma^{-1}})_{2,k_2} = (M_{\sigma^{-1}})_{1,2}$, $(M_{\sigma^{-1}})_{k_2,k_2} = (M_{\sigma^{-1}})_{2,2}$, and $(M_{\sigma^{-1}})_{k_2,2} = (M_{\sigma^{-1}})_{2,1}$.
Once we find such $k_2$, we swap the indices $3$ and $k_2$ by updating
\begin{align*}
M_{\sigma^{-1}} &\gets P_{\sigma_{(3,k_2)}} M_{\sigma^{-1}} P_{\sigma_{(3,k_2)}}^T, \\
\sigma &\gets \sigma \circ \sigma_{(3,k_2)}^{-1}.
\end{align*}
We repeat this process to find an index $k_l \in [l+1:m]$ and the "L-shaped" entries which preserve the cyclic pattern.
If we cannot find such $k_l$, we stop the process, and check
(i) whether all rows in $(M_{\sigma^{-1}})[1:l, l+1:m]$ are the same
and also (ii) whether all columns in $(M_{\sigma^{-1}})[l+1:n, 1:l]$ are the same. If the answers for both (i) and (ii) are yes, we move on to find the next diagonal block in $M_{\sigma^{-1}}$ in the same manner.

If (i) or (ii) has a negative answer, we go back to the step right before choosing $k_1$ from $K$ and resetting $\sigma$ and $M$ to the values at that step. 
Then, we repeat the above process with a new choice of $k_1$ from $K$.
If no choice of $k_1 \in K$ leads to the situation where both (i) and (ii)
have positive answers, then we conclude that $p_1 = 1$, and move on to find the next diagonal block of $M$ starting from the index $2$.

\subsection{Detection of direct product}

Let $M$ be a $m\times m$ matrix. Assume that we are given $p,q \in \N$ with
$pq = m$. If for some permutation groups $H \le \perm_p$ and $K \le \perm_q$,
 the matrix $M$ lies in $\cE(H \otimes K)$ (i.e., $M$ is $H\otimes K$-equivariant), then $M$ can be represented as a linear combination of Kronecker products:
\begin{align}
\label{eqn:kronecker-linear-combination}
M = \sum\limits_{i = 1}^\gamma (X_i \otimes Y_i)
\end{align}
where $X_i \in \cE(H)$, $Y_i \in \cE(K)$, and $\gamma$ is a natural number. 

The representation suggests the following strategy of finding a group $G$ of symmetries of $M$ that is the direct product of two permutation groups on $[p]$ and $[q]$. First, we express $M$ as a linear combination of Kronecker products $X_i \otimes Y_i$ for $X_i \in \R^{p\times p}$ and $Y_i \in \R^{q \times q}$. Second, we pick a pair $X_i$ and $Y_i$ in the linear combination. Third, we call $\symfind$ recursively on $X_i$ and $Y_i$ to get group-permutation pairs $(H,\sigma_H)$ and $(K,\sigma_K)$. Finally, we return $(H \otimes K, \sigma_H \otimes \sigma_K)$ where 
\begin{gather*}
\sigma_H \otimes \sigma_K : [m] \to [m],
\\
(\sigma_H \otimes \sigma_K)((a-1)q + b) = (\sigma_H(a)-1)q + \sigma_K(b)
\quad
\text{for all $a \in [p]$ and $b \in [q]$}.
\end{gather*}

Our $\kronsplit$ is the implementation of the strategy just described. The only non-trivial steps of the strategy are the first two, namely, to find the representation of $M$ in \eqref{eqn:kronecker-linear-combination}, and to pick good $X_i$ and $Y_i$.
$\kronsplit$ also has to account for the fact that the representation in \eqref{eqn:kronecker-linear-combination} holds only approximately at best in our context.

For the representation finding, $\kronsplit$ uses the technique~\citep{Van93} that solves the following optimisation problem for given $\gamma$:
\begin{align}
\label{eqn:kronecker-approx}
\argmin_{X_i \in \R^{p\times p}, Y_i \in \R^{q \times q}} \left\lVert M - \sum\limits_{i = 1}^\gamma (X_i \otimes Y_i) \right\rVert_F.
\end{align}
The technique performs the singular value decomposition of the rearranged matrix $\hat{M} \in \R^{p^2 \times q^2}$ of $M$ defined by
\begin{align}
\label{eqn:c-hat}
\hat{M}_{(i'-1)p+i, (j'-1)q+j} = M_{(i-1)q+j, (i'-1)q+j'}
\end{align}
for $i, i' \in [p]$ and $j, j' \in [q]$. To understand why it does so, note that the optimisation problem in \eqref{eqn:kronecker-approx} is equivalent to the following problem:
\begin{align}
\label{eqn:kronecker-approx-svd}
\argmin_{X_i \in \R^{p\times p}, Y_i \in \R^{q \times q}} \left\Vert \hat{M} - \sum\limits_{i = 1}^\gamma \ovec(X_i) \ovec(Y_i)^T \right\Vert_F.
\end{align}
This new optimisation problem can be solved using SVD. Concretely,
if we have the SVD of $M$, namely,
\begin{align}
\label{eqn:rank1-decomposition}
\hat{M} = \sum\limits_{i = 1}^{\rank(\hat{M})} s_i (u_i v_i^T)
\end{align}
where $\rank(\hat{M})$ is the rank of $\hat{M}$, $s_i$ is the $i$-th largest singular value, and $u_i$ and $v_i$ are the corresponding left and right singular vectors,
we can solve the minimisation problem~\eqref{eqn:kronecker-approx-svd} by reshaping each term of \eqref{eqn:rank1-decomposition} corresponding to the $\gamma$ largest singular values, i.e.,
\begin{align*}
\hat{M} \approx \sum\limits_{i = 1}^{\gamma} \ovec(X_i) \ovec(Y_i)^T \quad \text{where } \ovec(X_i) = \sqrt{s_i} u_i \text{ and } \ovec(Y_i) = \sqrt{s_i} v_i
\text{ for every $i \in [\gamma]$}.
\end{align*}

The first step of $\kronsplit$ is to apply the above technique~\citep{Van93}. $\kronsplit$ performs the SVD of the matrix $\hat{M}$ and gets 
$s_i,u_i,v_i$ in \eqref{eqn:rank1-decomposition} where $s_1 \geq s_2 \geq \ldots \geq s_l$ for some $l$. Then, it sets $L = s_1 / 5$, and picks $\gamma$ so that all $i$'s with $s_i \geq L$ are included when we approximate $\hat{M}$.

The second step of $\kronsplit$ is to pick $X_i$ and $Y_i$. $\kronsplit$ simply picks $X_1$ and $Y_1$ that correspond to matrices (reshaped from $u_1$ and $v_1$) for the largest singular value $s_1$. We empirically observed that the matrices $X_1$ and $Y_1$ are most informative about the symmetries of $M$, and are least polluted by noise from the training of SATNet. 

The third and fourth steps of $\kronsplit$ are precisely the last two steps of the strategy that we described above.

Our description of $\symfind$ in the main text says that $\kronsplit$ gets called for all the divisors $p$ of $n$ with $q$ being $n/p$, since we do not know the best divisior $p$ a priori, and that each invocation generates a new candidate in the candidate list $\cA$. However, in practice, we ignore some divisors that violates our predefined criteria. In particular, we set an upper bound $U$ on $\gamma$, and use a divisor $p$ only when $\kronsplit$ for $p$ picks $\gamma$ with $\gamma \le U$, i.e., not too many terms are considered in the linear combination of Kronecker products. This criterion can be justified by the following lemma:
\begin{lemma}
Let $G = H \otimes K$ be the direct product of permutation groups $H$ and $K$ on $[p]$ and $[q]$, respectively.
Consider a $G$-equivariant matrix $M \in \cE(G)$ and its rearranged version $\hat{M}$ in \eqref{eqn:c-hat}.
Then, $\rank(\hat{M}) \le \min \left( |\cB(H)|, \, |\cB(K)| \right)$.
\end{lemma}

\begin{proof}
Let $M \in \cE(G)$. The matrix $M$ can be expressed as a linear combination of the basis elements in $\cB(G)$:
\begin{align}
\nonumber
M &= \sum\limits_{i = 1}^{|\cB(H)|} \sum\limits_{j = 1}^{|\cB(K)|} \alpha_{ij} (A_i \otimes B_j) \\
\label{eqn:linear-combination-1}
&= \sum\limits_{j = 1}^{|\cB(K)|} \left( \sum\limits_{i = 1}^{|\cB(H)|} \alpha_{ij} A_i \right) \otimes B_j \\
\label{eqn:linear-combination-2}
&= \sum\limits_{i = 1}^{|\cB(H)|} A_i \otimes \left( \sum\limits_{j = 1}^{|\cB(K)|} \alpha_{ij} B_j \right)
\end{align}
where $A_i \in \cB(H)$ and $B_j \in \cB(K)$.
We can also rewrite \eqref{eqn:linear-combination-1} and \eqref{eqn:linear-combination-2} with equations for $\hat{M}$:
\begin{align*}
\begin{aligned}
\hat{M} &= \sum\limits_{j = 1}^{|\cB(K)|} \ovec(A_j') \ovec(B_j)^T 
= \sum\limits_{i = 1}^{|\cB(H)|} \ovec(A_i) \ovec(B_i')^T \\
\\
&
\qquad\qquad\qquad
\text{where } A_j' = \left( \sum\limits_{i = 1}^{|\cB(H)|} \alpha_{ij} A_i \right) \text{ and }
B_i' = \left( \sum\limits_{j = 1}^{|\cB(K)|} \alpha_{ij} B_j \right).
\end{aligned}
\end{align*}
Since the summands $\ovec(A_j') \ovec(B_j)^T$ and $\ovec(A_i) \ovec(B_i')^T$ are rank-$1$ matrices, by the subadditivity of rank, i.e., $\rank(X + Y) \le \rank(X) + \rank(Y)$, we have $\rank(\hat{M}) \le \min (|\cB(H)|, \, |\cB(K)|)$.
\end{proof}

Having fewer basis elements is generally better because it leads to a small number of parameters to learn in SymSATNet. The above lemma says that a divisor $p$ with large $\gamma$ (which roughly corresponds to the large rank of $\hat{M}$ in the lemma) leads to large $|\cB(H)|$ and $|\cB(K)|$. Our criterion is designed to avoid such undesirable cases.

\subsection{Detection of wreath product}
Assume that the permutation groups $G$ and $H$ act transitively on $[p]$ and $[q]$
(i.e., for all $i,j \in [p]$ and $i',j' \in [q]$, there are $g \in G$ and $h \in H$ such that $g(i) = j$ and $h(i') = j'$).
We recall that in this case, every equivariant matrix $M$ under the wreath product of groups $H \wr G$ can be expressed as follows~\citep{Wang20}:
\begin{align}
\label{eqn:wreath-split}
M = A \otimes \mathbf{1}_q + I_p \otimes B
\end{align}
for some $A \in \cE(G)$ and $B \in \cE(H)$, where $\mathbf{1}_m$, $I_m$ are everywhere-one, identity matrices in $\R^{m \times m}$.
We use this general form of $\cB(H \wr G)$-equivariant matrices and make $\kronsplit$ detect the case that symmetries are captured by a wreath product. 

Our change in $\kronsplit$ is based on a simple observation that the form in \eqref{eqn:wreath-split} is exactly the one in \eqref{eqn:kronecker-linear-combination} with $\gamma = 2$. We change $\kronsplit$ such that if we get $\gamma = 2$ while runnning $\kronsplit(M, p)$, we check whether $M$ can be written as the form in \eqref{eqn:wreath-split}. This checking is as follows. Using given $p$ and $q$, we create blocks 
\begin{align*}
M^{(i,j)} = M[(i-1) \times q + 1 : i \times q, (j-1) \times q + 1 : j \times q]
\end{align*}
of $q\times q$ submatrices in $M$ for all $i,j \in [p]$. Then, we test whether all $M^{(i,i)} - M^{(j,j)}$ and $M^{(i,j)}$ for $i \neq j$ are
constant matrices (i.e., matrices of the form $\alpha \mathbf{1}_q$ for some $\alpha \in \R$). If this test passes, $M$ has the desired form. In that case, we compute $A$ and $B$ from $M$, and recursively call $\symfind$ on $A$ and $B$.

We can also make $\sumsplit(M)$ detect wreath product. The required change is to perform a similar test on each diagonal block of the final $M_{\sigma^{-1}}$ computed by $\sumsplit(M)$. The only difference is that when the test fails, we look for a rearrangement of the rows and columns of the diagonal block that makes the test succeed. If the test on a diagonal block succeeds (after an appropriate rearrangement), the group corresponding to this block, which is one summand of a direct sum, becomes a wreath product.

%% file: appendix-symfind-complexity.tex
\section{Computational Complexity of $\symfind$}
In this section, we analyse the computational complexity of our $\symfind$ algorithm in terms of the dimension of the input matrix.
Suppose that an $n \times n$ matrix is given to SymFind as an input.
The complexity of SymFind shown in Algorithm~\ref{alg:symfind} is $O(n^{3 + \epsilon})$ for any arbitrarily small $\epsilon > 0$.

We can show the claimed complexity of $\symfind$ using induction.
When SumFind is called in the line 9 of Algorithm~\ref{alg:symfind}, it first spends $O(n^3)$ time for clustering, and then makes recursive calls to SymFind with $n_i \times n_i$ submatrices for $\sum n_i = n$.
Each of these calls takes $O(n_i^{3 + \epsilon})$ time by induction, and the total cost of all the calls is $\sum_i O(n_i^{3 + \epsilon}) = O(n^{3 + \epsilon})$.
Also, when ProdFind is called in the lines 11-12 of Algorithm~\ref{alg:symfind}, for each divisor $p$ of $n$, ProdFind performs SVD in $O(n^3)$ steps, and makes recursive calls to SymFind.
The recursive calls here together cost $O(p^{3 + \epsilon / 2} + (n/p)^{3 + \epsilon / 2}) = O(n^{3 + \epsilon / 2})$ by induction.
Thus, the loop in lines 11-14 costs $O(n^{3 + \epsilon / 2} d(n)) = O(n^{3 + \epsilon})$, where $d(n)$ is the number of divisors of $n$, since $d(n) = O(n^{\epsilon / 2})$ for any arbitrarily small $\epsilon > 0$.
It remains to show that applying the Reynolds operator (in the lines 2 and 6) and counting the number of basis elements (in the line 15) take $O(n^{3 + \epsilon})$ steps.
The former costs $O(n^2)$ since it is actually an average pooling operation when we consider permutation groups.
The latter costs at most $O(n)$ (when it is given $\mathcal{B}(\bigoplus_{i=1}^n \mathbb{Z}_1)$).

%% file: appendix-sudoku-cube.tex
\section{Symmetries in Sudoku and Rubik's Cube Problems}
\label{sec:symmetries-sudoku-cube}
In this section, we describe the group symmetries in Sudoku and the completion problem of Rubik's cube.
In both problems, we can find certain group $G \le \perm_n$ acting on $\R^{k \times n}$ such that for any valid assignment $V \in \R^{k \times n}$ of problem, $g \cdot V$ is also valid for any $g \in G$.

In $9 \times 9$ Sudoku, we denote the first three rows of a Sudoku board by the first band, and the next three rows by the second band, and the last three rows by the third band.
In the same manner, we denote each three columns by a stack.
For Sudoku problem, one can observe that any row permutations within each band preserve the validity of the solutions.
Also, the permutations of bands preserve the validity.
We can represent this type of hierarchical actions by wreath product groups.
In this case, 3 bands are permuted in a higher level, and 3 rows in each band are permuted in a lower level, which corresponds to the group action of $\perm_3 \wr \perm_3$.
The same process can be applied to the stacks and the columns, which also corresponds to the group $\perm_3 \wr \perm_3$.
Furthermore, the permutations of number occurences in Sudoku (e.g., switching all the occurences of 3's and 9's) also preserves the validity of the solutions.
In this case, any permutations over $[9]$ are allowed, thus $\perm_9$ represents this permutation action.
Overall, we have three permutation groups, $G_1 = \perm_3 \wr \perm_3$, $G_2 = \perm_3 \wr \perm_3$, and $G_3 = \perm_9$.
These three groups are in different levels; $G_1$ is at the outermost level, and $G_3$ is at the innermost level.
Note that these three types of actions are commutative, i.e., the order does not matter if we apply the actions in different levels.
This relation forms a direct product between each level, and we conclude $G = (\perm_3 \wr \perm_3) \otimes (\perm_3 \wr \perm_3) \otimes \perm_9$.

For Rubik's cube, we first introduce some of the conventional notations.
We are given a $3 \times 3 \times 3$ Rubik's cube composed of 6 faces and 9 facelets in each face.
There are also 26 cubies in the Rubik's cube, and they are classified into 8 corners, 12 edges, and 6 centers.
To change the color state of Rubik's cube, we can move the cubies with 9 types of rotations, which are denoted by $U$, $D$, $F$, $B$, $L$, $R$, $M$, $E$, $S$.
We can represent each color state of the Rubik's cube by a function from $F = [6] \times [9]$ to $C = [6]$, where $F$ encodes the face and facelets, and $C$ encodes the colors.
The initial state of Rubik's cube is $s : F \to C$ satisfying $s(i, j) = i$ for all $(i, \, j) \in F$.
If there exists a sequence of rotations that transforms $s$ to the initial state, then $s$ is solvable.
Our completion problem of Rubik's cube is to find a color assignment such that the assignment is solvable.
For example, if a corner cubie contains 2 same colors, it is not solvable because these adjacent same colors cannot be separated by the 9 types of rotations.

We can observe the following group symmetries in the completion problem of Rubik's cube:
any solvable color state is still solvable after being transformed by the 9 types of rotations.
These 9 rotations generate a permutation group $\cR_{54}$ acting on the $6 \times 9$ facelets of a Rubik's cube, which is called the Rubik's cube group.
Furthermore, like Sudoku, we can also consider the permutations of color occurences.
In this case, if we assume the colors $0$ and $5$ are initially on the opposite sides,
then the solvability can be broken by switching the colors $0$ and $1$ since the colors $0$ and $5$ are then no longer on the opposite sides.
Considering this, computing all possible permutations acting on the color occurences is not trivial.
Instead, we can generate such group $\cR_6$ by 3 elements, each of which corresponds to the $90^\circ$ rotation in one of the three axes in 3-dimension.
Finally, these two groups $\cR_{54}$ and $\cR_6$ form different levels of actions, and actions from different levels commute each other.
Therefore, we conclude $G = \cR_{54} \otimes \cR_6$.

%% file: appendix-efficiency.tex
\section{Efficient Implementation of SymSATNet}

\algsetup{indent=0.6em}
\begin{figure}
\begin{minipage}{0.48\columnwidth}
\begin{algorithm}[H]
   \caption{Forward pass of SATNet}
   \label{alg:forward-SATNet}
\begin{algorithmic}[1]
   \STATE {\bfseries Input:} $V_{\cI}$
   \STATE {\bfseries init} random unit vectors $v_o$, $\forall o \in \cO$
   \STATE {\bfseries compute} $\Omega = VS^T$
   \REPEAT
   \FOR{$o \in \cO$}
   \STATE {\bfseries compute} $g_o = \Omega s_o - \Vert s_o \Vert^2 v_o$
   \STATE {\bfseries compute} $v_o = - g_o / \Vert g_o \Vert$
   \STATE {\bfseries update} $\Omega = \Omega + (v_o - v_o^{\text{prev}}) s_o^T$
   \ENDFOR
   \UNTIL{not converged}
   \STATE {\bfseries Output:} $V_{\cO}$
\end{algorithmic}
\end{algorithm}
\end{minipage}
\hfill
\begin{minipage}{0.48\columnwidth}
   \begin{algorithm}[H]
      \caption{Forward pass of SymSATNet}
      \label{alg:forward-SymSATNet}
   \begin{algorithmic}[1]
      \STATE {\bfseries Input:} $V_{\cI}$
      \STATE {\bfseries init} random unit vectors $v_o$, $\forall o \in \cO$
      \REPEAT
      \FOR{$o \in \cO$}
      \STATE {\bfseries compute} $g_o = V c_o - C_{o,o} v_o$
      \STATE {\bfseries compute} $v_o = - g_o / \Vert g_o \Vert$
      \STATE {\bfseries update} $V = V + (v_o - v_o^{\text{prev}}) \mathds{1}_o^T$
      \ENDFOR
      \UNTIL{not converged}
      \STATE {\bfseries Output:} $V_{\cO}$
   \end{algorithmic}
   \end{algorithm}
\vfill
\end{minipage}
\end{figure}
\algsetup{indent=1em}

\label{sec:algorithm-satnet}
Our SymSATNet includes forward-pass and backward-pass algorithms as described in Section~\ref{sec:SymSATNet}.
Here, we describe slight changes of the forward-pass and backward-pass algorithms in SymSATNet, and the improvement of efficiency obtained by these changes.
Algorithms~\ref{alg:forward-SATNet} and~\ref{alg:forward-SymSATNet} are the forward pass algorithms of SATNet and SymSATNet, and
Algorithms~\ref{alg:backward-SATNet} and~\ref{alg:backward-SymSATNet} are the backward pass algorithms of SATNet and SymSATNet.
In those algorithms, we let $c_o$ be the $o$-th column of the matrix $C$, and $v_o^{\text{prev}}$ and $u_o^{\text{prev}}$ be the previous $o$-th columns of $V$ and $U$ before the update.
Also, let $P_o = I_k - v_o v_o^T$, and $\mathds{1}_o$ be the n-dimensional one-hot vector whose only $o$-th element is $1$.
Note that our algorithms are only slightly different from the ones of original SATNet, and the only difference is that the inner products
$s_i^T s_j$ are substituted by $C_{i,j}$ which can be directly derived from $C = S^T S$.
This allows us to implement the forward-pass and the backward-pass algorithm of SymSATNet in the same manner as the original SATNet.

The above changes bring a small difference in the computational complexity.
In SATNet, new matrices $\Omega$ and $\Psi$ are required for the rank-$1$ update in each loop.
Recall that $V, U \in \R^{k \times n}$ and $S \in \R^{m \times n}$ imply $\Omega, \Psi \in \R^{k \times m}$. 
SATNet costs $O(nmk)$ twice every iteration, in the lines 6, 8 of Algorithm~\ref{alg:forward-SATNet} and in the lines 5, 7 of Algorithm~\ref{alg:backward-SATNet}. SymSATNet costs $O(n^2k)$ once every iteration, in the line 5 of Algorithm~\ref{alg:forward-SymSATNet} and in the line 5 of Algorithm~\ref{alg:backward-SymSATNet}.
Another slight difference is in the first line of the body of each for loop, where SATNet computes $\Vert s_o \Vert^2$, but SymSATNet uses the constant $C_{o, o}$.
If we denote the required number of iterations by $t$, then SATNet costs $O(nmk \cdot 2t)$, and SymSATNet costs $O(n^2k \cdot t)$ inside the loop.
These make one of the major differences between the runtime of SATNet and SymSATNet, since these are scaled by $t$, the number of iterations until convergence.
Outside the loop, SATNet additionally costs $O(nmk)$ in the line 3 of Algorithm~\ref{alg:forward-SATNet},
and SymSATNet costs $O(n^2d)$ in the forward computation of $C$ in \eqref{eqn:SymSATNet-equivariance}, and in the backward computation with the chain rule in \eqref{eqn:SymSATNet-chainrule}.
These costs are not significant because each of these occurs only once every gradient update in SATNet and SymSATNet.

We also inspected the values of $t$ in SATNet and SymSATNet in a training run for Sudoku and Rubik's cube problem in the configuration we used in our experiments.
For Sudoku, SATNet repeated the loop $t = 21.39$ times on average, in the range of $15 \le t \le 32$,
while SymSATNet finished the loop in $t = 19.10$ on average, in the scope of $3 \le t \le 26$.
For Rubik's cube, SATNet completes the loop in $t = 9.25$ on average, whose scope was $3 \le t \le 25$,
while SymSATNet iterates $t = 7.30$ times on average, with the range $2 \le t \le 26$.
These results show that the component $t$ is reduced in SymSATNet for those two problems, and bring further improvement of efficiency in practice.

%% file: appendix-hyperparameter.tex
\section{Hyperparameters}
In this section, we specify the hyperparameters to run SymFind algorithm and the validation steps in SymSATNet-Auto.
Each validation step requires a threshold of improvement of validation accuracy to measure the usefulness of each smaller part in a given group,
i.e., only the smaller parts showing improvement greater than the threshold were considered to be useful, and were combined to construct the whole group.
We determined this threshold by the number of corruption in the dataset; $0.05$ for the dataset with 0 and 1 corruption, $0.15$ for 2 corruption, and $0.2$ for 3 corruption.
Also, our SymFind algorithm receives $\lambda$ as an input, which determines the degree of sensitivity of SymFind.
In practice, this tolerance was implemented by two hyperparameters, say $\lambda_1$ and $\lambda_2$.
$\lambda_1$ was used to compute a threshold to decide whether a pair of entries have the same value within our tolerance so that they have to be clustered.
We computed this threshold by the multiplication of $\lambda_1$ and the norm of input matrix, and $\lambda_1$ was fixed by $0.1$ for both problems.
$\lambda_2$ played a role of a threshold to conclude that the input matrix truly has the symmetries under the discovered group by SymFind.
$\lambda_2$ was used in the last step of SymFind, where the largest group is selected among the candidates.
With the Reynolds operator, we computed the distance between the input matrix and the equivariant space under each candidate group, and filtered out the groups whose distance is greater than $\lambda_2$.
We determined $\lambda_2$ by the number of corruption in the dataset; $0.4$ for the dataset with 0 corruption, $0.5$ for 1 corruption, $0.55$ for 2 corruption, and $0.6$ for 3 corruption.

\algsetup{indent=0.4em}
\begin{figure}
    \begin{minipage}{0.50\columnwidth}
    \begin{algorithm}[H]
       \caption{Backward pass of SATNet}
       \label{alg:backward-SATNet}
    \begin{algorithmic}[1]
       \STATE {\bfseries Input:} $\{ \partial \ell / \partial v_o : o \in \cO \}$
       \STATE {\bfseries init} $U_{\cO} = 0$ and $\Psi = U_{\cO}S_{\cO}^T = 0$
       \REPEAT
       \FOR{$o \in \cO$}
       \STATE {\bfseries compute} $dg_o = \Psi s_o - \Vert s_o \Vert^2 u_o - \partial \ell / \partial v_o$
       \STATE {\bfseries compute} $u_o = - P_o dg_o / \Vert g_o \Vert$
       \STATE {\bfseries update} $\Psi = \Psi + (u_o - u_o^{\text{prev}}) s_o^T$
       \ENDFOR
       \UNTIL{not converged}
       \STATE {\bfseries Output:} $U_{\cO}$
    \end{algorithmic}
    \end{algorithm}
    \end{minipage}
    \hfill
    \begin{minipage}{0.48\columnwidth}
       \begin{algorithm}[H]
          \caption{Backward pass of SymSATNet}
          \label{alg:backward-SymSATNet}
       \begin{algorithmic}[1]
          \STATE {\bfseries Input:} $\{ \partial \ell / \partial v_o : o \in \cO \}$
          \STATE {\bfseries init} $U_{\cO} = 0$
          \REPEAT
          \FOR{$o \in \cO$}
          \STATE {\bfseries compute} $dg_o = U c_o - C_{o,o} u_o - \partial \ell / \partial v_o$
          \STATE {\bfseries compute} $u_o = - P_o dg_o / \Vert g_o \Vert$
          \STATE {\bfseries update} $U = U + (u_o - u_o^{\text{prev}}) \mathds{1}_o^T$
          \ENDFOR
          \UNTIL{not converged}
          \STATE {\bfseries Output:} $U_{\cO}$
       \end{algorithmic}
       \end{algorithm}
       \end{minipage}
    \end{figure}
\algsetup{indent=1em}

%% file: appendix-gpu.tex
\section{GPU Usage}

The implementation of SymSATNet is based on the original SATNet code, and thus the calculations in the forward-pass and the backward-pass algorithms of SymSATNet can be accelerated by GPU.
We used GeForce RTX 2080 Ti for every running of SATNet and SymSATNet.

%% file: appendix-ablation.tex
\section{Ablation Study of Validation Steps}

\begin{wrapfigure}{r}{0.48\columnwidth}
    \vskip -1em
    \centering
    \begin{minipage}{\linewidth}
    \centering
    \includegraphics[width=\columnwidth]{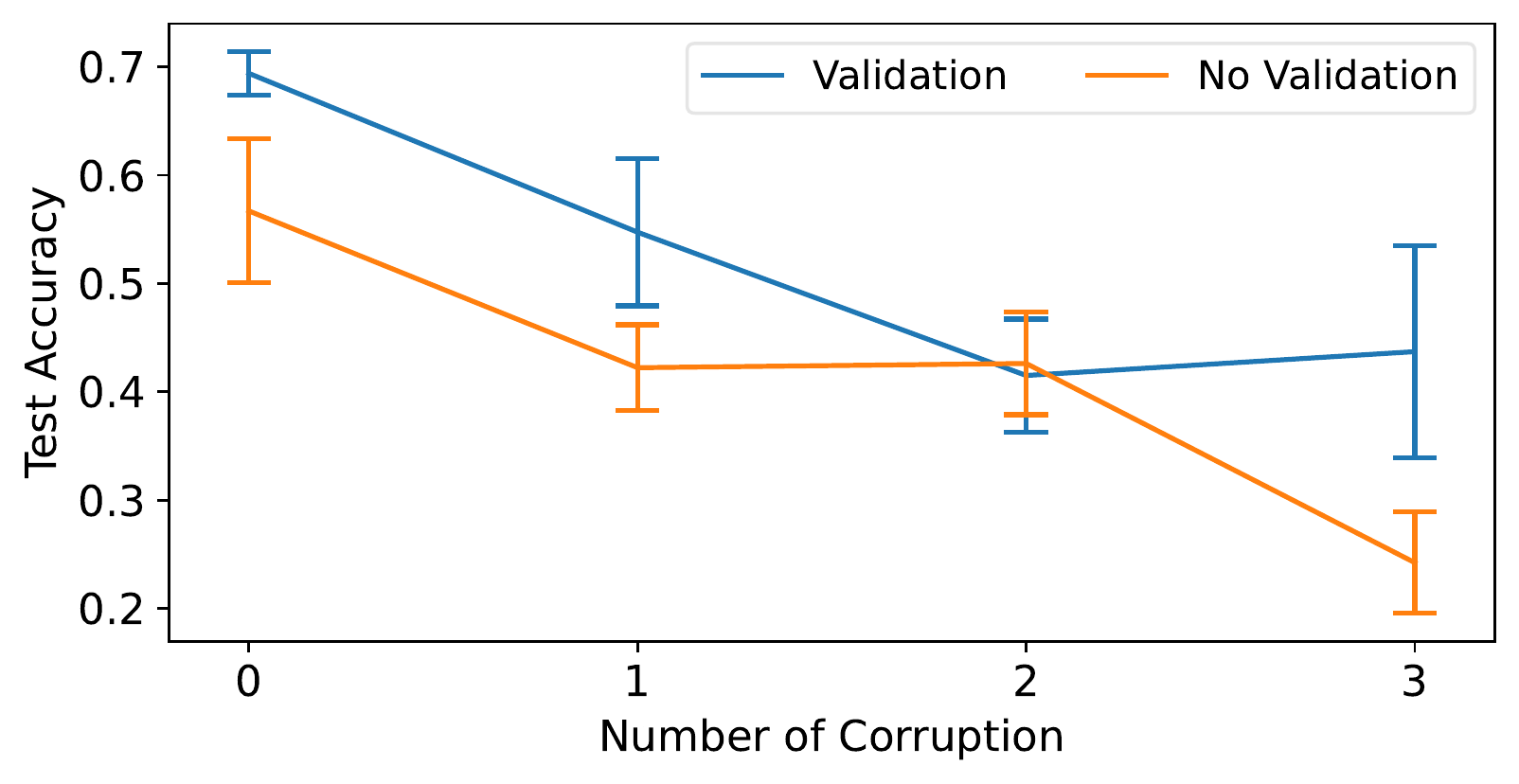}
    \end{minipage}
    \caption{Best test accuracies of SymSATNet-Auto with or without the validation step using noisy Rubik's cube datasets.}
    \label{fig:ablation}
    \vskip -1em
\end{wrapfigure}

In this section, we present the results of ablation study for the additional validation step in SymSATNet-Auto.
All the configurations are the same, except that the discovered group symmetries by $\symfind$ are directly exploited by SymSATNet without the subsequent validation step.
We used noisy Sudoku and Rubik's cube datasets, and repeated our experiment
with each dataset ten times. Here we report the average test accuracies with $95\%$ confidence interval.

In Sudoku, our $\symfind$ algorithm always succeeded in finding the full group symmetries, and these correct symmetries were preserved by the validation step
since they always significantly improved the validation accuracies after projecting the parameter $C$ of SATNet.
For this reason, omitting the validation step did not change the performance of SymSATNet-Auto. Thus, we do not plot the results for the Sudoku case.
For Rubik's cube, the results are shown in Figure~\ref{fig:ablation}.
Without the validation step, SymSATNet-Auto often overfit into the wrong group symmetries induced by the noise in the parameter matrix $C$ of SATNet, resulting in lower improvement of performance.
These results show that the additional validation step is useful not only when the training and the test examples share the same distribution,
but also when the distribution of training examples is the perturbation of that of test examples by noise.

%% file: appendix-emergence.tex
\section{Emergence of Symmetries in SATNet}

\begin{wrapfigure}{r}{0.48\columnwidth}
    \vskip -0.5em
    \centering
    \begin{subfigure}{0.48\columnwidth}
    \includegraphics[width=\linewidth]{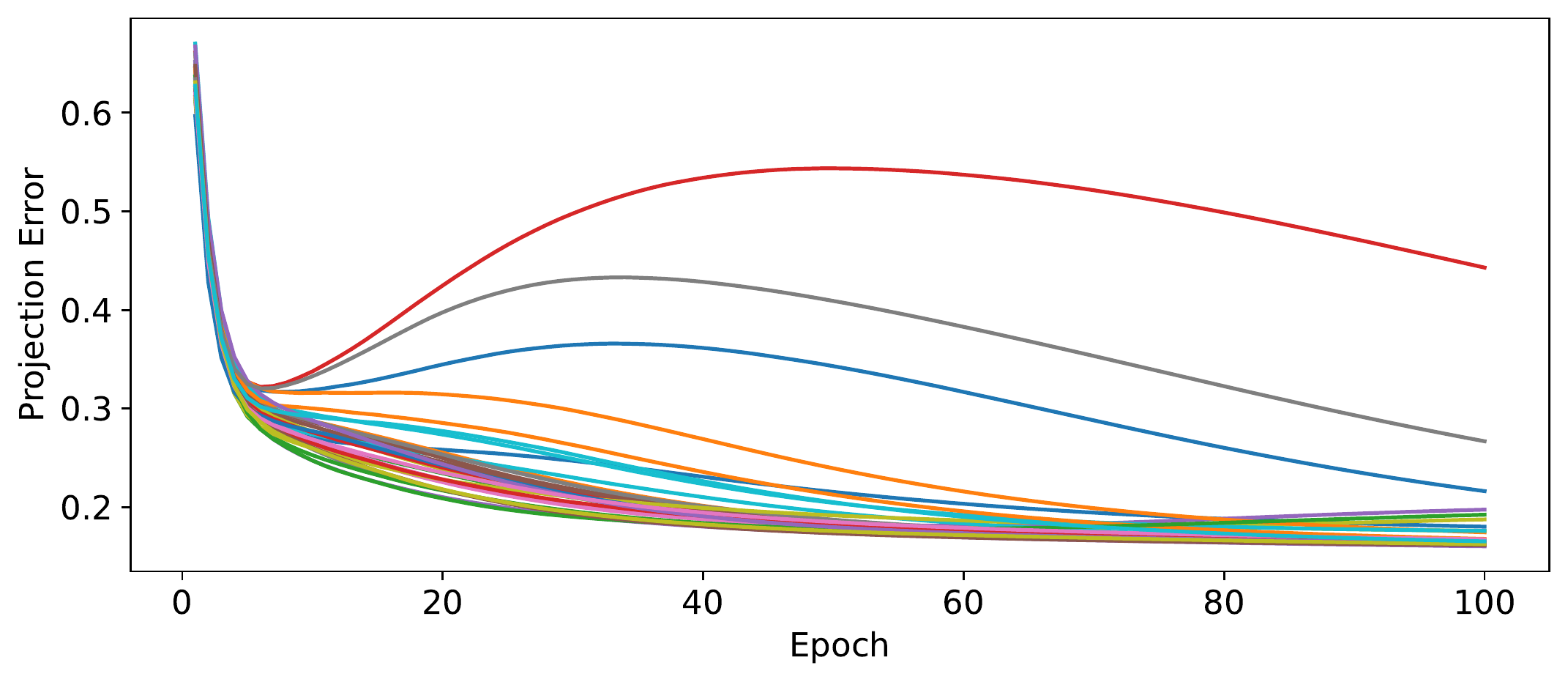}
    \subcaption{Sudoku}
    \end{subfigure}
    \vskip 0.5em
    \begin{subfigure}{0.48\columnwidth}
    \includegraphics[width=\linewidth]{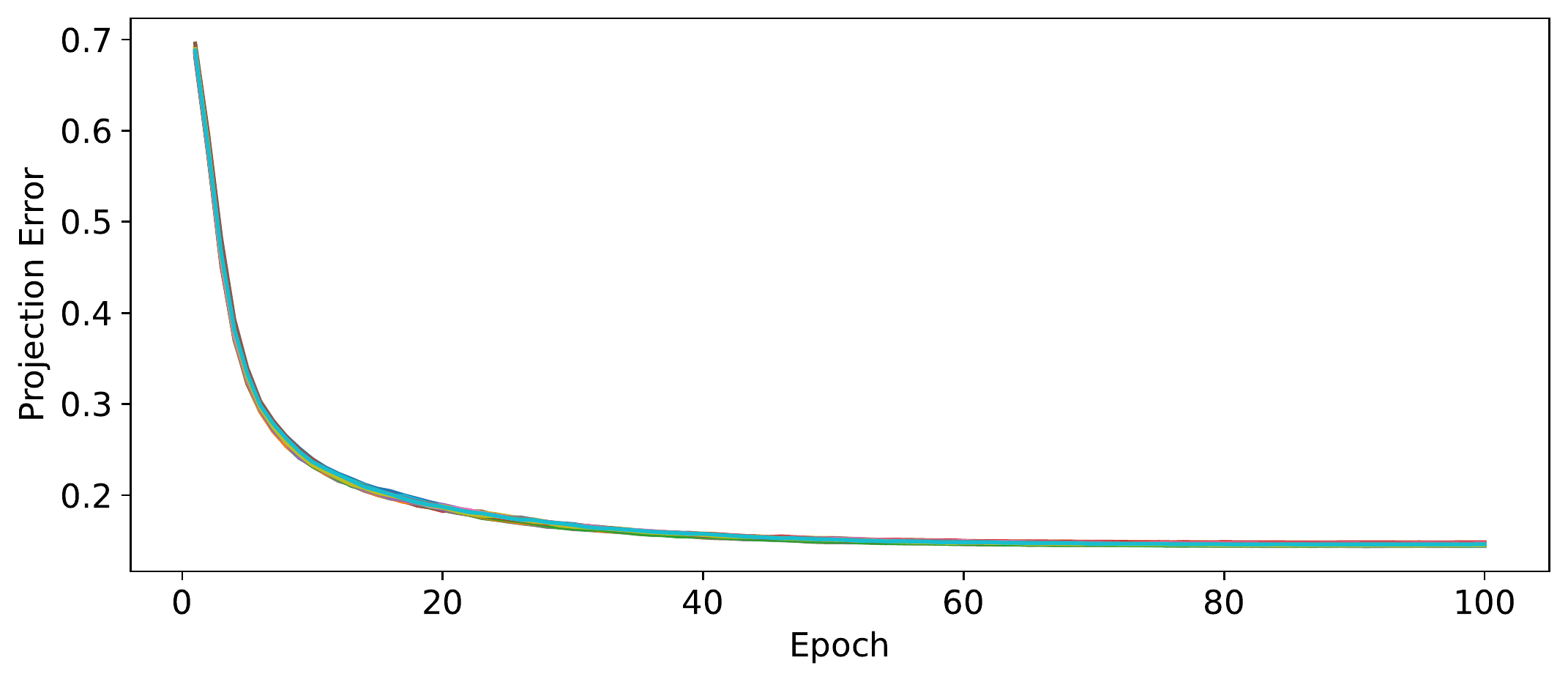}
    \subcaption{Rubik's cube}
    \end{subfigure}
    \caption{Projection errors of the SATNet's parameter $C$ during 100 epochs.
    We repeated 30 training trials for each problem to report the projection errors.}
    \label{fig:proj_err}
    \vskip -1em
\end{wrapfigure}

In this section, we report our experiment for detecting the appearance and disappearance of symmetries in the SATNet's parameter matrix $C$ throughout the training epochs.
We measured how close $C$ is to $\mathcal{E}(G)$, the space of equivariant matrices under $G$, by computing $\Vert \mathrm{prj}(G, C) - C \Vert_F$ (denoted by \textit{projection error} below) where $G$ is the ground-truth permutation group for a given learning problem (e.g., Sudoku or Rubik's cube).
We evaluated the projection errors throughout 100 training epochs.
Our experiment for each problem was repeated 30 times.

Figure~\ref{fig:proj_err} shows the results.
For Sudoku, in multiple cases out of 30 trials, the projection error hit the lowest point around the 5-10th epochs, and after these epochs, the projection error started to increase until certain epochs, so that the training ended with a high projection error.
We did not observe any clear difference in the training loss (or the test loss) between these cases with high projection errors and the other cases (which showed low projection errors).
This result suggests possible overfitting from the perspective of symmetry discovery in the original SATNet.
For Rubik’s cube, in all of our trials, the projection error always decreased throughout the epochs, and no sign of overfitting was detected.
Answering why this is the case is an interesting topic for future research.
Also, in the very beginning of the training for both problems, the projection errors were not sufficiently small.
By choosing proper stopping epochs in SymSATNet-Auto, we can avoid high projection errors, as we did (i.e., we picked the 10th epoch for Sudoku, and the 20th epoch for Rubik’s cube).
Finally, we point out that there may be factors other than the projection error that influence the performance of our symmetry-discovery algorithm. If found, 
those factors would become useful tools for analysing the theoretical guarantees of $\symfind$ algorithm, and we leave it to future work.

%% file: appendix-symsatnet-300aux.tex
\section{SymSATNet with Auxiliary Variables}

\begin{table}
    \centering
    \begin{small}
    \begin{sc}
    \caption{Best train and test accuracies during 100 epochs and average total train times ($10^2$ sec).
    Each experiment was repeated 10 times and its average and 95\% confidence interval are reported.
    Additional times for automatic symmetry detection ($\symfind$ and validation) are reported after $+$.}
    \label{table:accuracy_aux}
    \vskip 1em
    \begin{tabular}{lrrrrrr}
    \toprule
    Model & \multicolumn{3}{c}{Sudoku} & \multicolumn{3}{c}{Cube} \\
          & Train Acc. & Test Acc. & Time & Train Acc. & Test Acc. & Time \\
    \midrule \midrule
    SATNet\scriptsize-Plain   & 93.5\%                 & 88.1\%                 & 48.0                  & 99.4\%                 & 55.7\%                 & 1.8 \\
                              & \scriptsize $\pm$1.2\% & \scriptsize $\pm$1.3\% & \scriptsize $\pm$0.17 & \scriptsize $\pm$0.1\% & \scriptsize $\pm$0.5\% & \scriptsize $\pm$0.01 \\
    SATNet\scriptsize-300aux  & 99.8\%                 & 97.9\%                 & 90.3                  & 99.8\%                 & 56.5\%                 & 14.0 \\
                              & \scriptsize $\pm$0.0\% & \scriptsize $\pm$0.2\% & \scriptsize $\pm$0.68 & \scriptsize $\pm$0.0\% & \scriptsize $\pm$0.6\% & \scriptsize $\pm$0.12 \\
    \midrule
    SymSATNet\scriptsize-Plain  & 98.8\%                 & 99.2\%                 & 25.6                  & 67.1\%                 & 66.9\%                 & 1.1 \\
                                & \scriptsize $\pm$0.1\% & \scriptsize $\pm$0.2\% & \scriptsize $\pm$0.14 & \scriptsize $\pm$0.2\% & \scriptsize $\pm$0.9\% & \scriptsize $\pm$0.00 \\
    SymSATNet\scriptsize-300aux & 99.2\%                 & 99.3\%                 & 53.6                  & 69.6\%                 & 67.6\%                  & 9.5 \\
                                & \scriptsize $\pm$0.1\% & \scriptsize $\pm$0.1\% & \scriptsize $\pm$0.56 & \scriptsize $\pm$0.4\% & \scriptsize $\pm$0.5\% & \scriptsize $\pm$0.29 \\
    SymSATNet\scriptsize-Auto   & 99.3\%                 & 99.5\%                 & 22.7                  & 70.2\%                 & 68.1\%                 & 3.4 \\
                                & \scriptsize $\pm$0.1\% & \scriptsize $\pm$0.2\% & \scriptsize \makecell{$+$0.14 \\ $\pm$0.35} & \scriptsize $\pm$3.1\% & \scriptsize $\pm$2.0\% & \scriptsize \makecell{$+$0.66 \\ $\pm$0.19} \\
    \bottomrule
    \end{tabular}
    \end{sc}
    \end{small}
\end{table}

\begin{wrapfigure}{r}{0.48\columnwidth}
    \vskip -1em
    \begin{subfigure}{0.48\columnwidth}
    \includegraphics[width=0.95\linewidth]{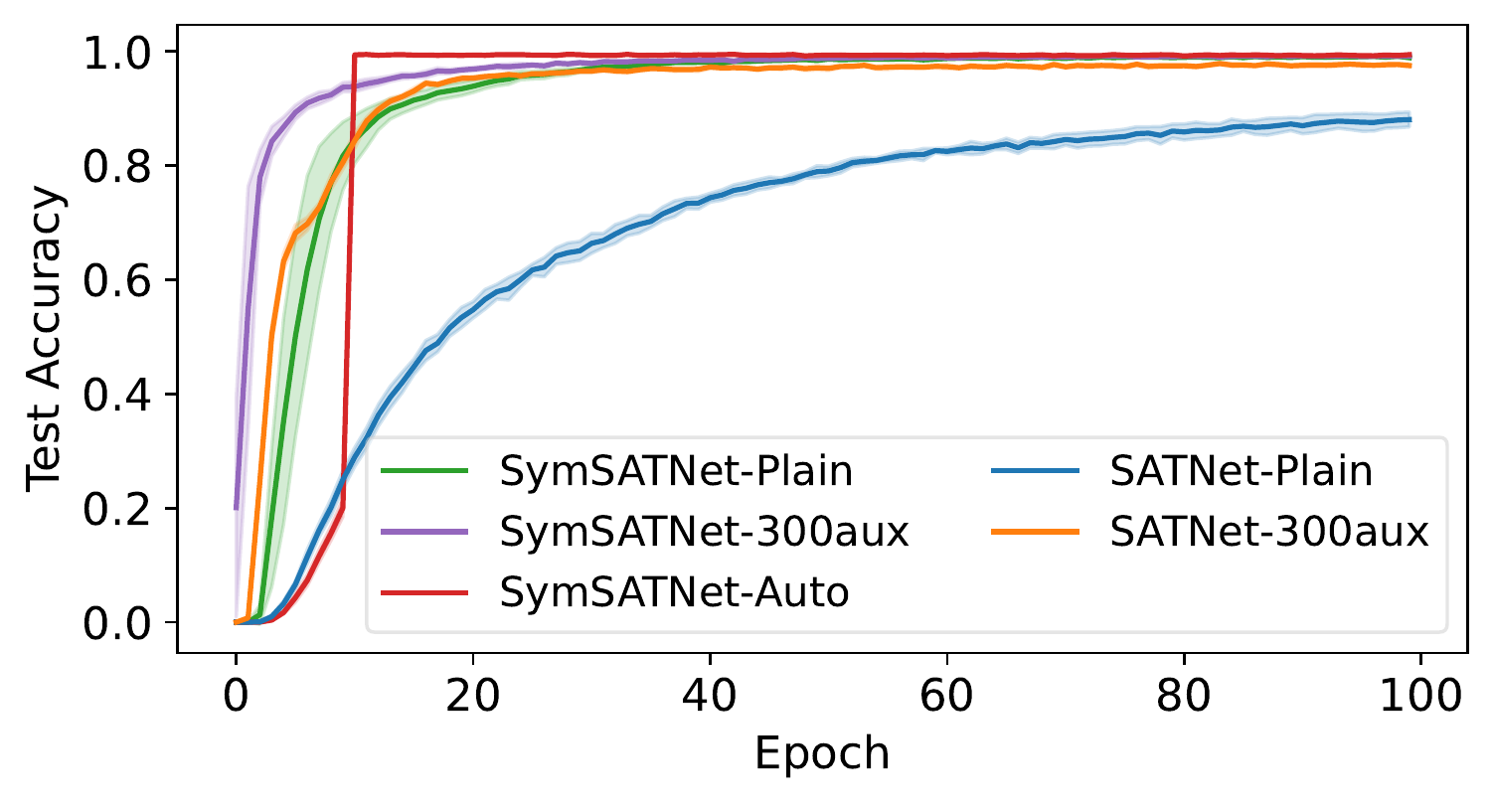}
    \vskip -0.5em
    \subcaption{Sudoku}
    \end{subfigure}
    \begin{subfigure}{0.48\columnwidth}
    \includegraphics[width=0.95\linewidth]{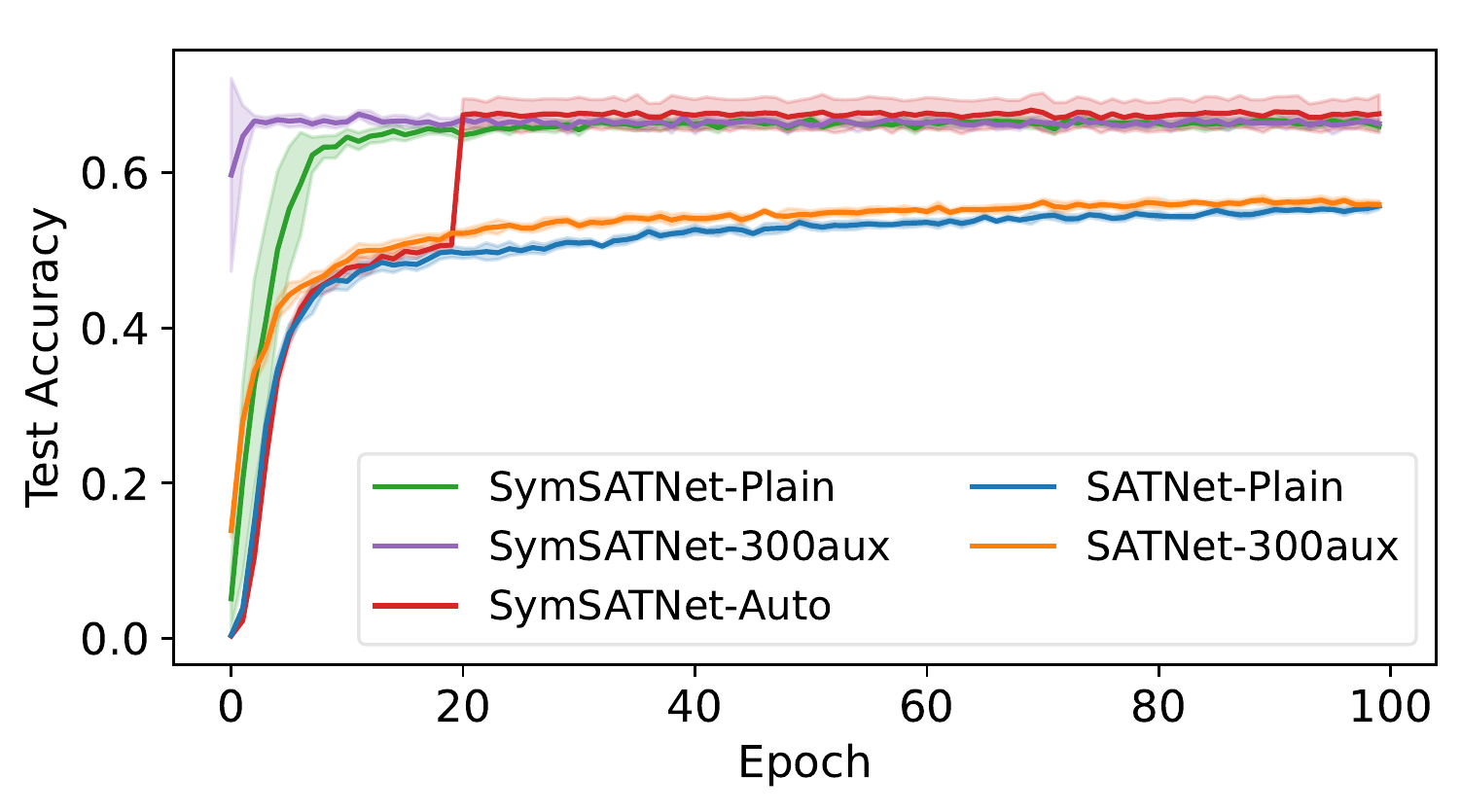}
    \vskip -0.5em
    \subcaption{Rubik's cube}
    \end{subfigure}
    \caption{Test accuracies throughout the first 100 training epochs. Each training run was repeated 10 times to compute the average and the 95\% confidence interval.}
    \label{fig:accuracy_aux}
    \vskip -1.5em
\end{wrapfigure}

By introducing auxiliary variables, the original SATNet can achieve higher expressive power and better performance while trading-off the runtime efficiency.
The similar process can be done in SymSATNet.
If a permutation group $G$ captures group symmetries of a problem to be solved without auxiliary variables,
then $G \oplus \mathcal{I}_m$ represents the extension of the same symmetries with the $m$ auxiliary variables.
The singleton group $\mathcal{I}_m$ here acts trivially (i.e., permutes nothing) on those auxiliary variables.

We now report the findings of our experiments with SymSATNet with 300 auxiliary variables (SymSATNet-300aux) that retains the above extension of group symmetries.
We compared the performance of SymSATNet-300aux with the other four models that we used before, on the 0-corrupted Sudoku and Rubik's cube datasets.
We denote SymSATNet without auxiliary variables by SymSATNet-Plain to distinguish it from SymSATNet-300aux.
We repeated the training for each dataset 10 times. Here we report the average and best test accuracies, the average training times, and $95\%$ confidence interval.
We trained SymSATNet-300aux with the learning rate $\eta = 4 \times 10^{-2}$, which is the same as SymSATNet-Plain.
All the experimental setups were the same as before.

Figure~\ref{fig:accuracy_aux} and Table~\ref{table:accuracy_aux} show the results.
In both tasks, SymSATNet-300aux outperformed all the other models except SymSATNet-Auto.
Also, for Rubik's cube, SymSATNet-300aux showed the fastest convergence in epoch among all the models.
Note that both the test accuracies and the training times of SymSATNet-300aux were remarkably improved over those of SATNet baselines.
These results demonstrate that SATNet with auxiliary variables still enjoyed the benefits by exploiting a minor extension of group symmetries,
although the benefits were not as significant as in the case of SymSATNet-Plain.

%% file: appendix-losses.tex
\section{Loss Curves}

\begin{figure}[bh]
    \begin{subfigure}{0.48\columnwidth}
        \includegraphics[width=0.95\linewidth]{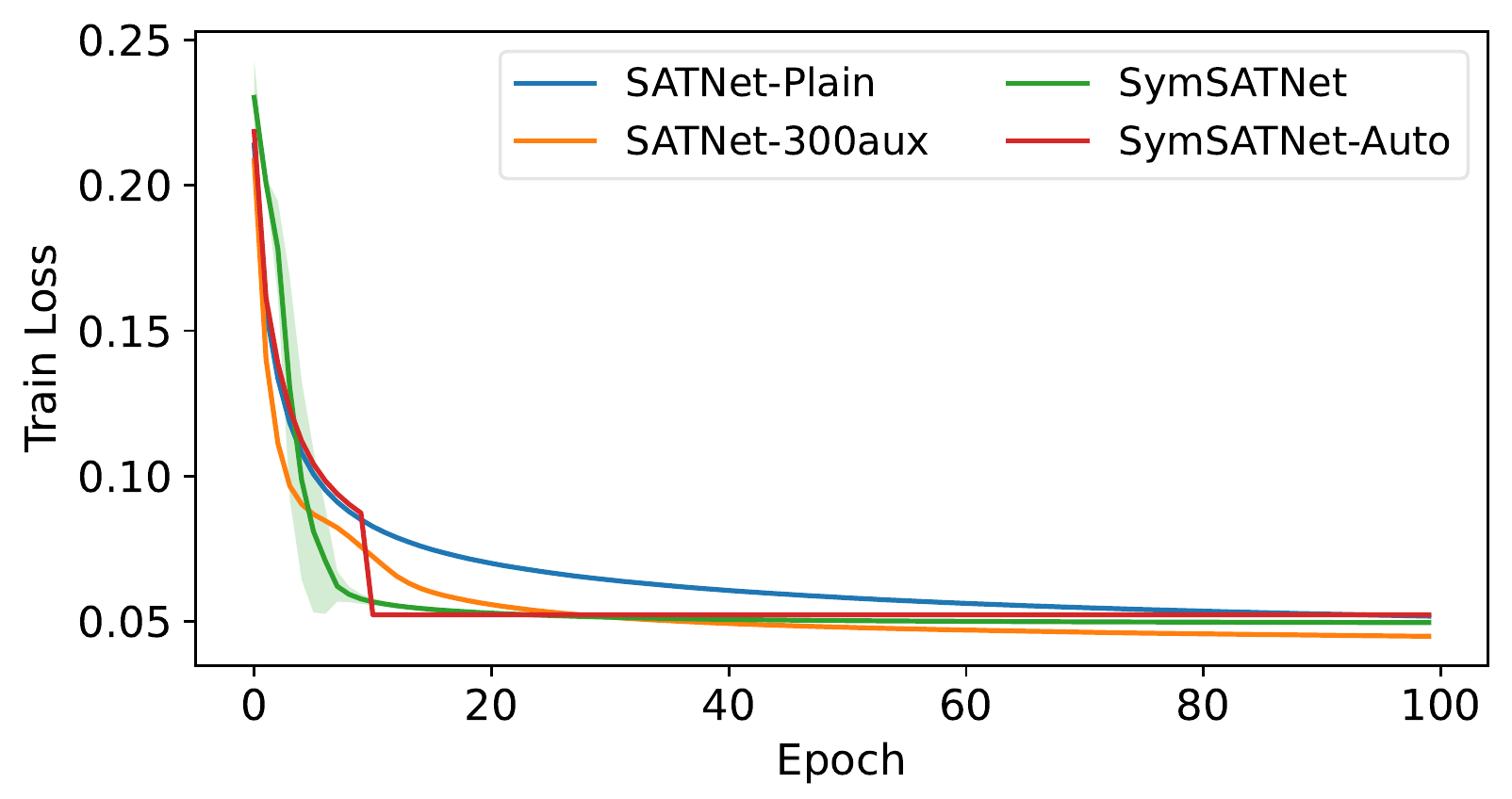}
        \subcaption{Training losses in Sudoku with 0 corruption}
        \label{fig:train-loss-sudoku-1}
    \end{subfigure}
    \begin{subfigure}{0.48\columnwidth}
        \includegraphics[width=0.95\linewidth]{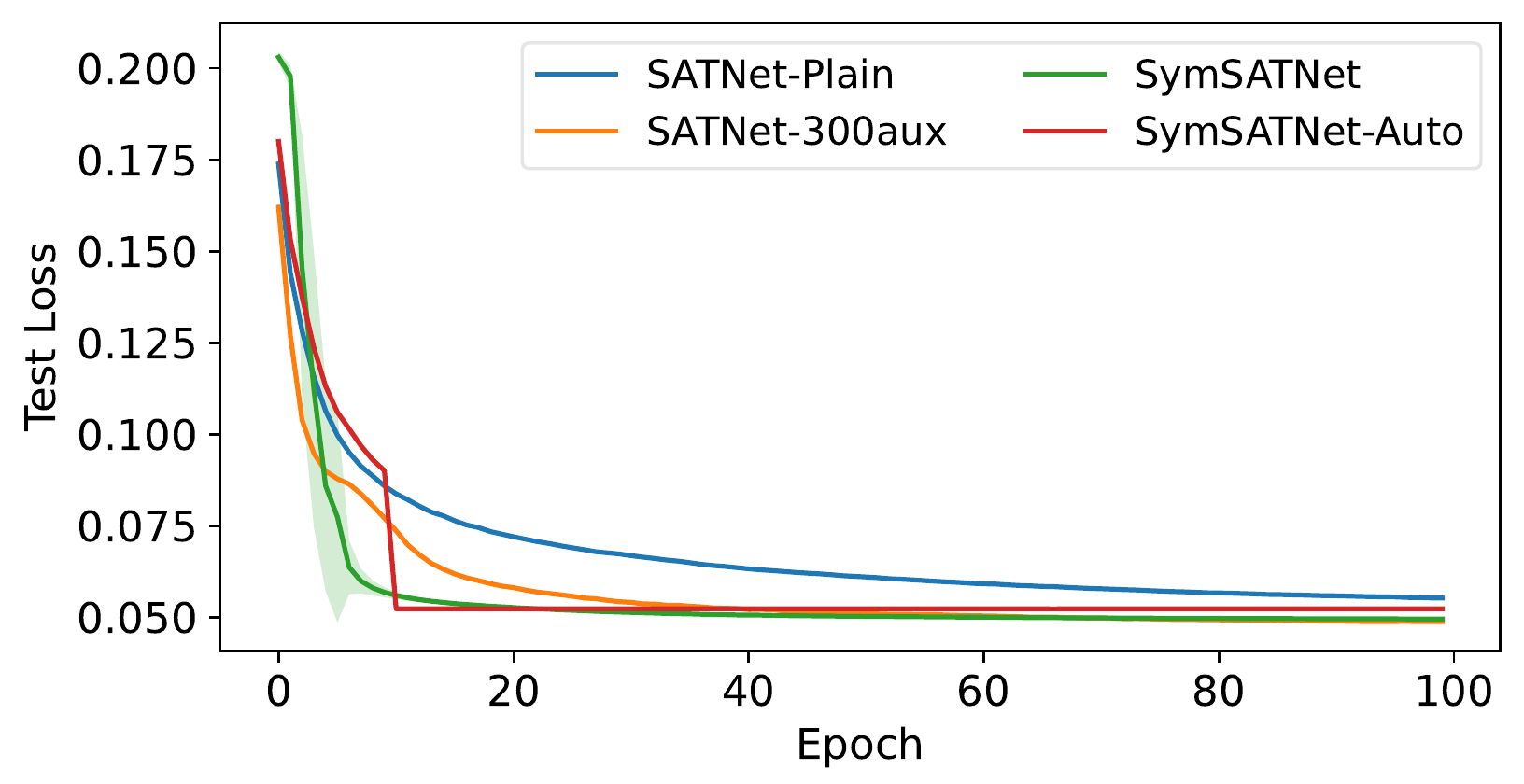}
        \subcaption{Test losses in Sudoku with 0 corruption}
        \label{fig:test-loss-sudoku-1}
    \end{subfigure}
    \begin{subfigure}{0.48\columnwidth}
        \includegraphics[width=0.95\linewidth]{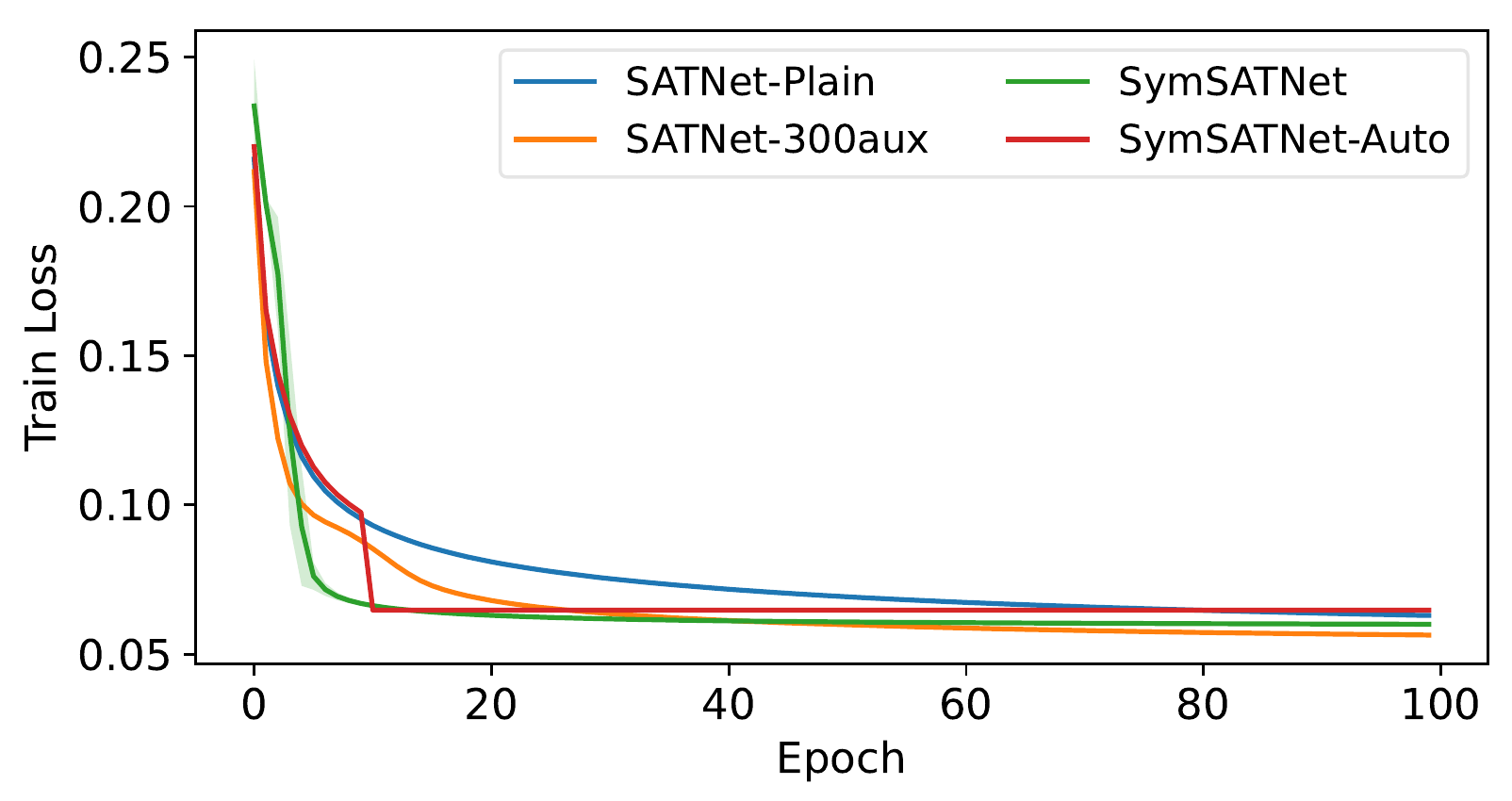}
        \subcaption{Training losses in Sudoku with 1 corruption}
        \label{fig:train-loss-sudoku-2}
    \end{subfigure}
    \begin{subfigure}{0.48\columnwidth}
        \includegraphics[width=0.95\linewidth]{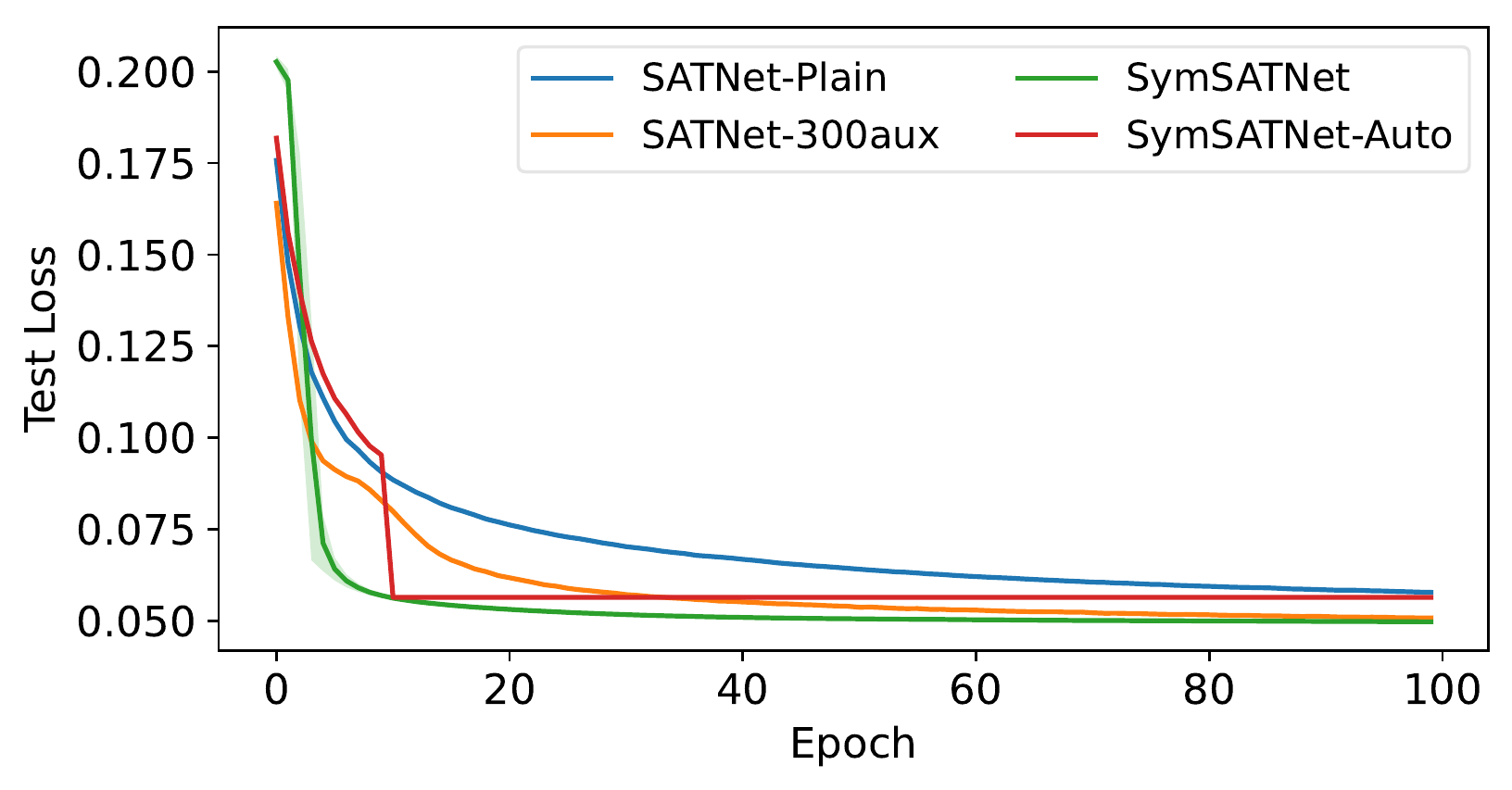}
        \subcaption{Test losses in Sudoku with 1 corruption}
        \label{fig:test-loss-sudoku-2}
    \end{subfigure}
    \begin{subfigure}{0.48\columnwidth}
        \includegraphics[width=0.95\linewidth]{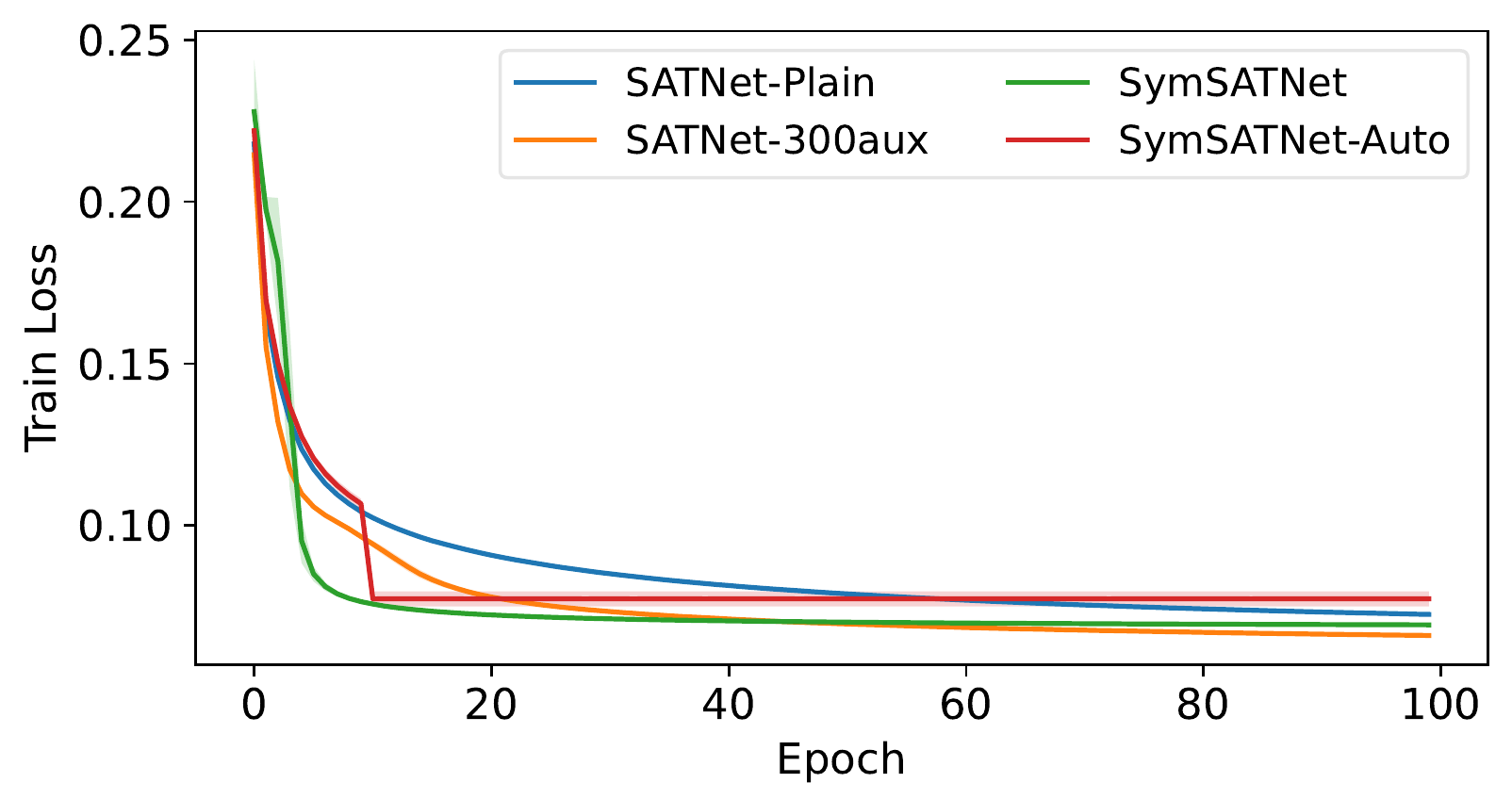}
        \subcaption{Training losses in Sudoku with 2 corruptions}
        \label{fig:train-loss-sudoku-3}
    \end{subfigure}
    \begin{subfigure}{0.48\columnwidth}
        \includegraphics[width=0.95\linewidth]{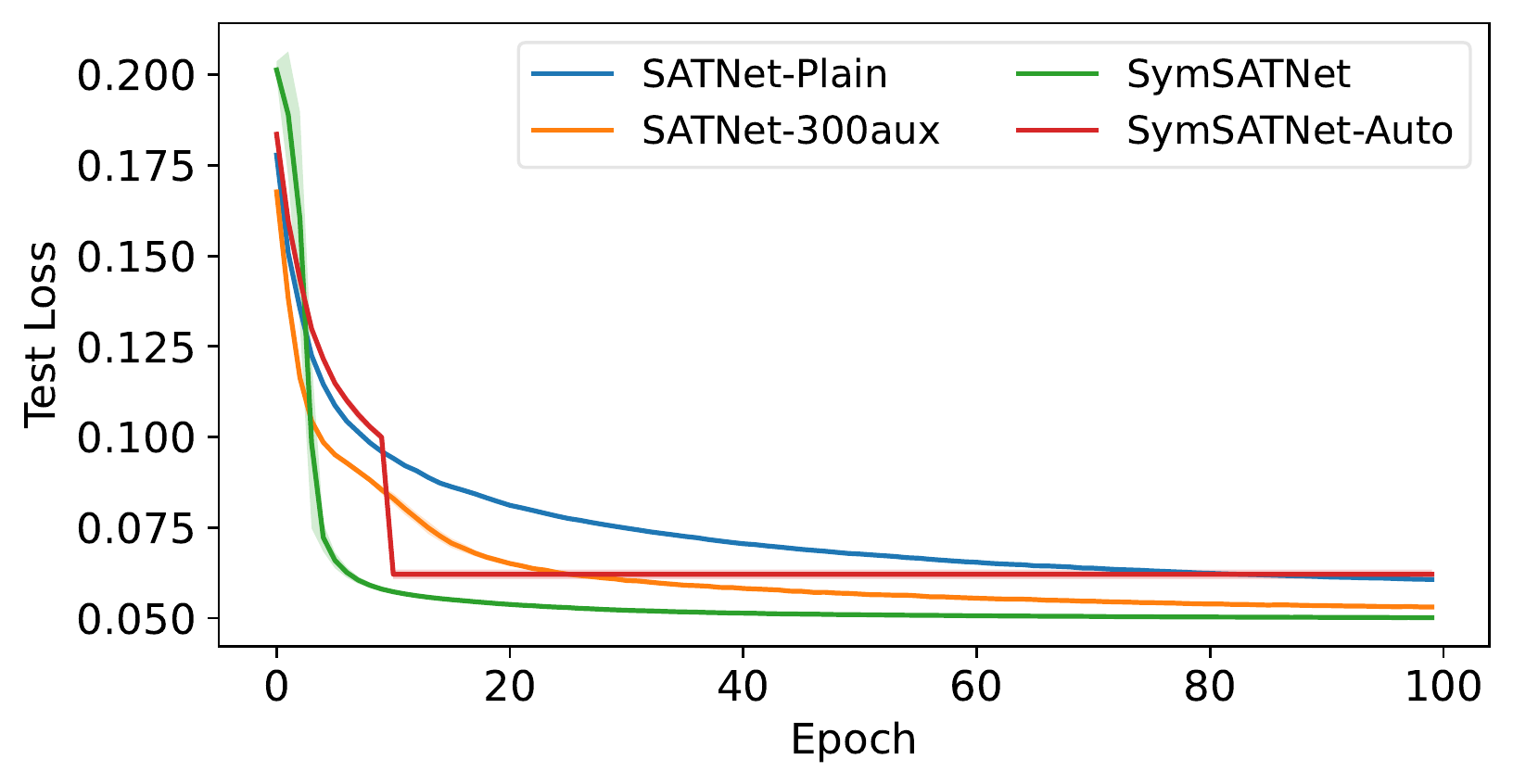}
        \subcaption{Test losses in Sudoku with 2 corruptions}
        \label{fig:test-loss-sudoku-3}
    \end{subfigure}
    \begin{subfigure}{0.48\columnwidth}
        \includegraphics[width=0.95\linewidth]{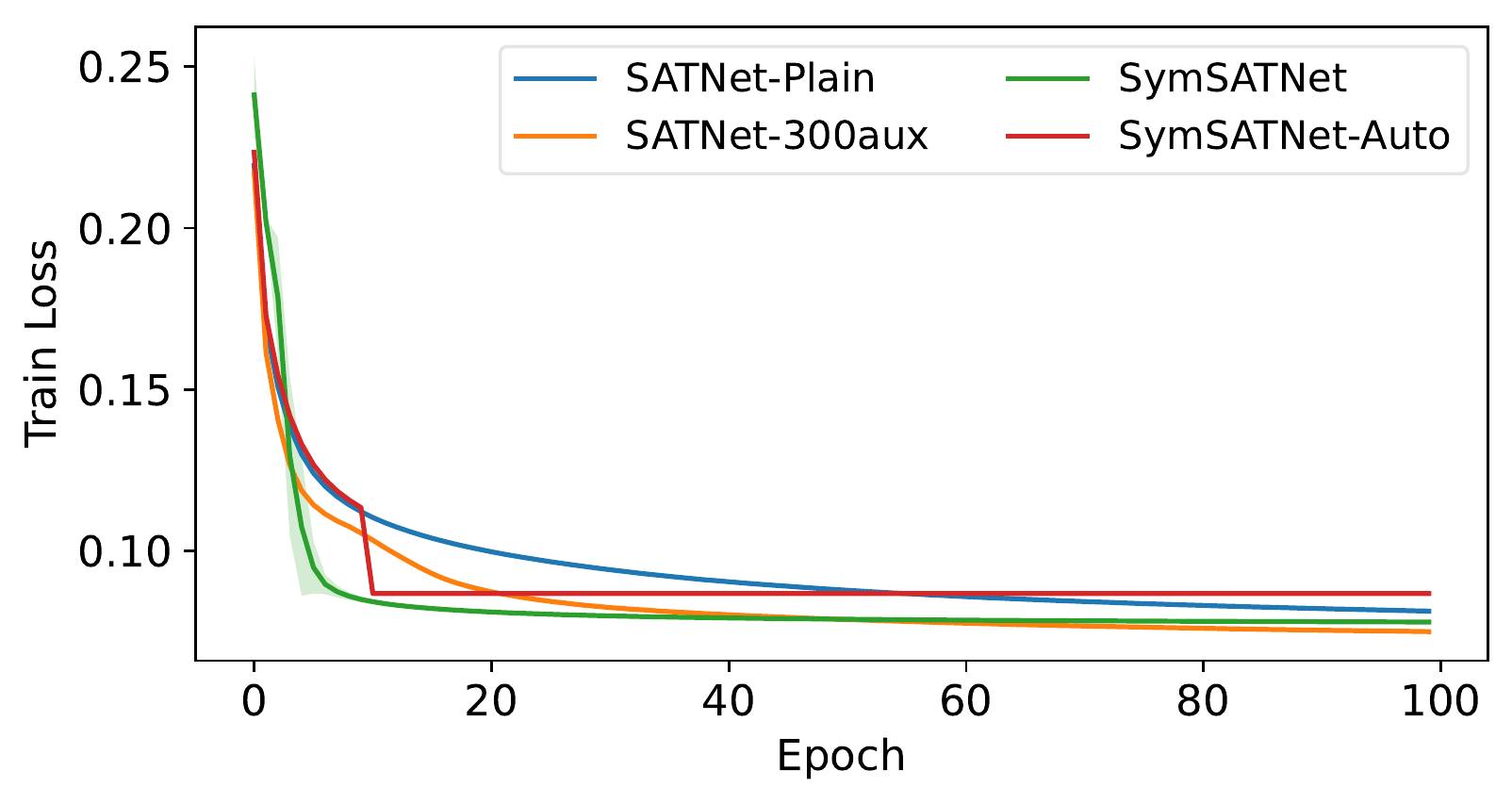}
        \subcaption{Training losses in Sudoku with 3 corruptions}
        \label{fig:train-loss-sudoku-4}
    \end{subfigure}
    \hskip 1.1em
    \begin{subfigure}{0.48\columnwidth}
        \includegraphics[width=0.95\linewidth]{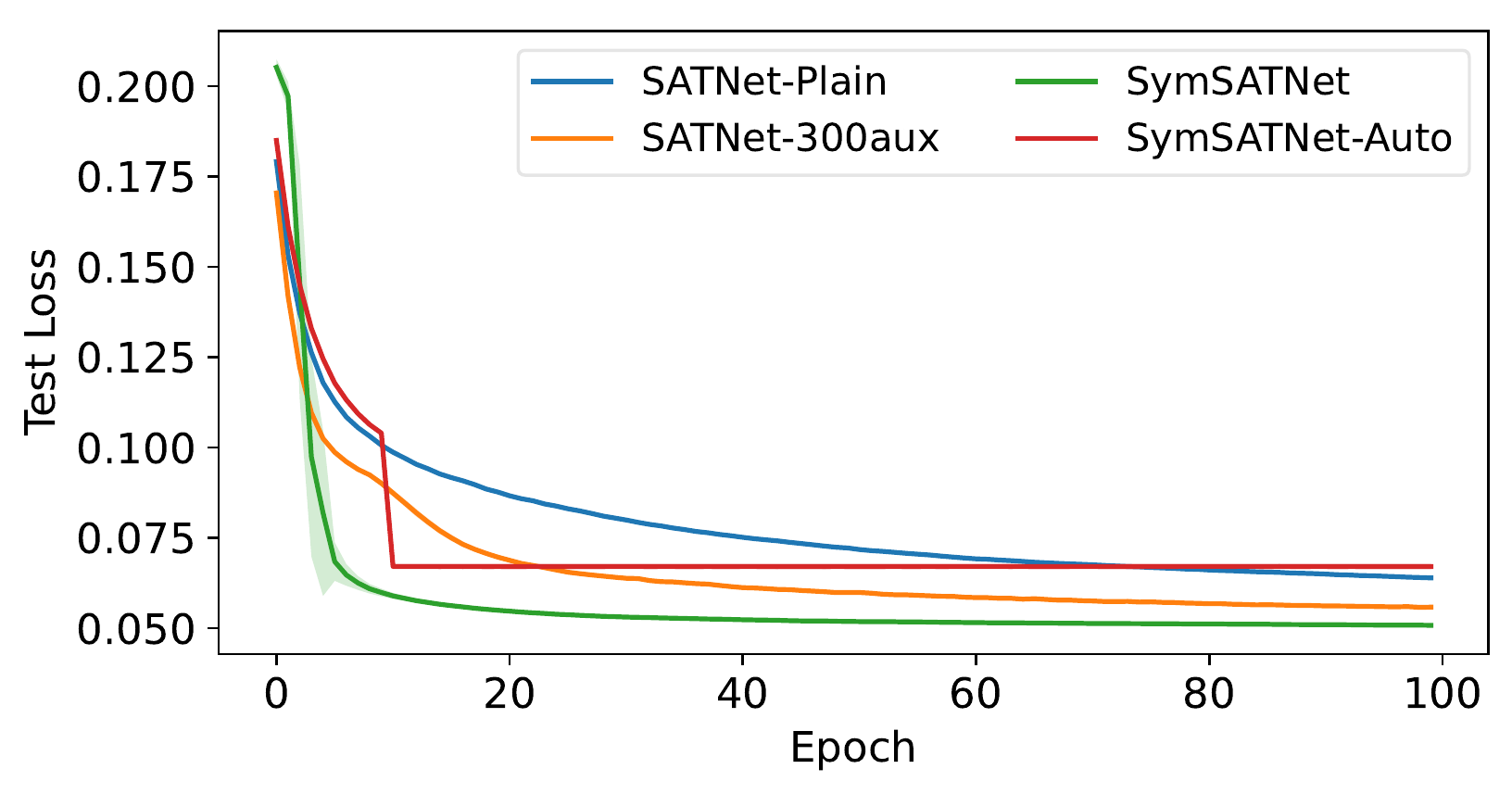}
        \subcaption{Test losses in Sudoku with 3 corruptions}
        \label{fig:test-loss-sudoku-4}
    \end{subfigure}
    \caption{Training and test loss curves for Sudoku. Each loss is averaged over 10 trials and each $95\%$ confidence interval is included.}
    \label{fig:losses-sudoku}
\end{figure}

\begin{figure}
    \begin{subfigure}{0.48\columnwidth}
        \includegraphics[width=0.95\linewidth]{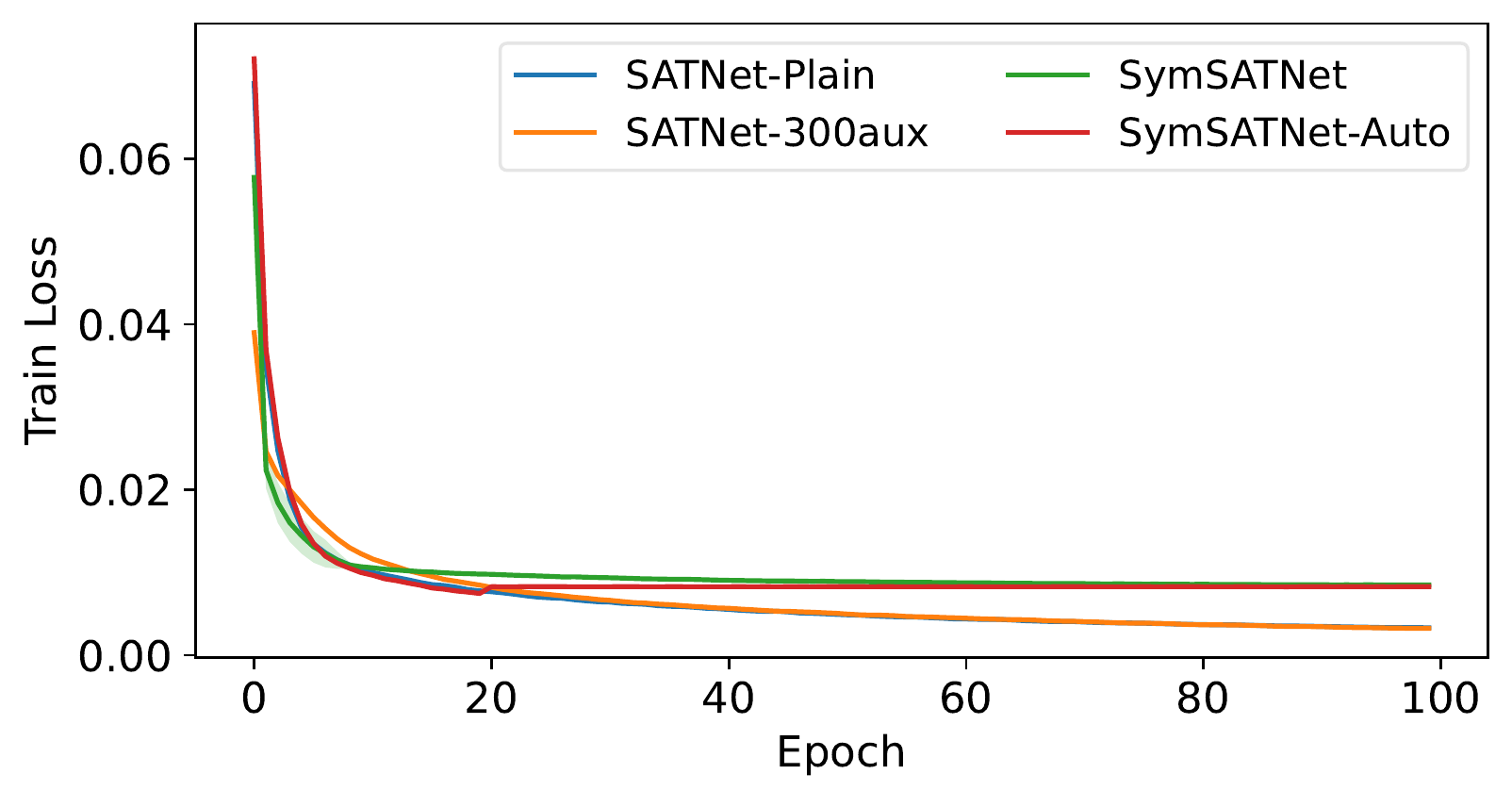}
        \subcaption{Training losses in Rubik's cube with 0 corruption}
        \label{fig:train-loss-cube-1}
    \end{subfigure}
    \begin{subfigure}{0.48\columnwidth}
        \includegraphics[width=0.95\linewidth]{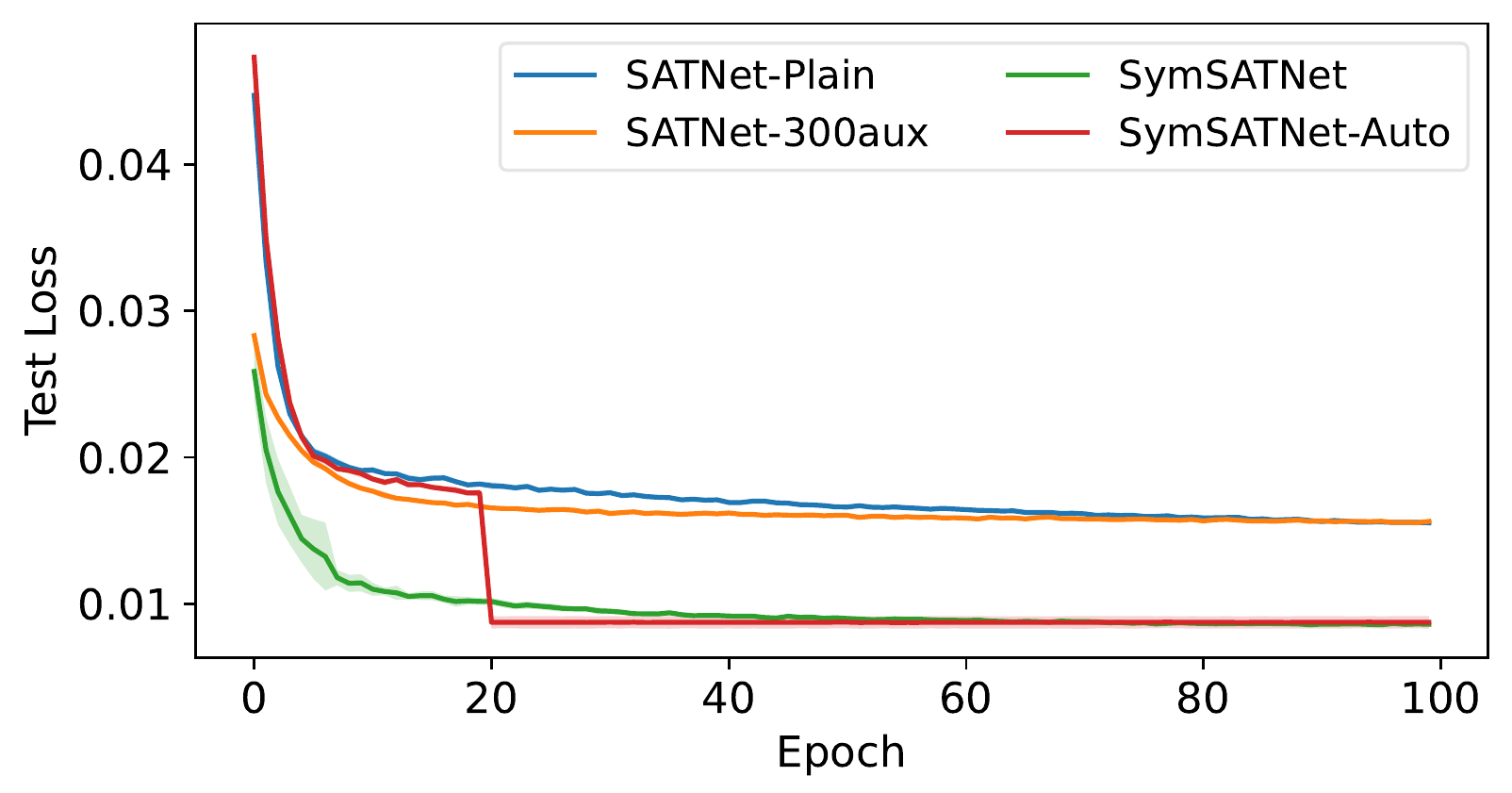}
        \subcaption{Test losses in Rubik's cube with 0 corruption}
        \label{fig:test-loss-cube-1}
    \end{subfigure}
    \begin{subfigure}{0.48\columnwidth}
        \includegraphics[width=0.95\linewidth]{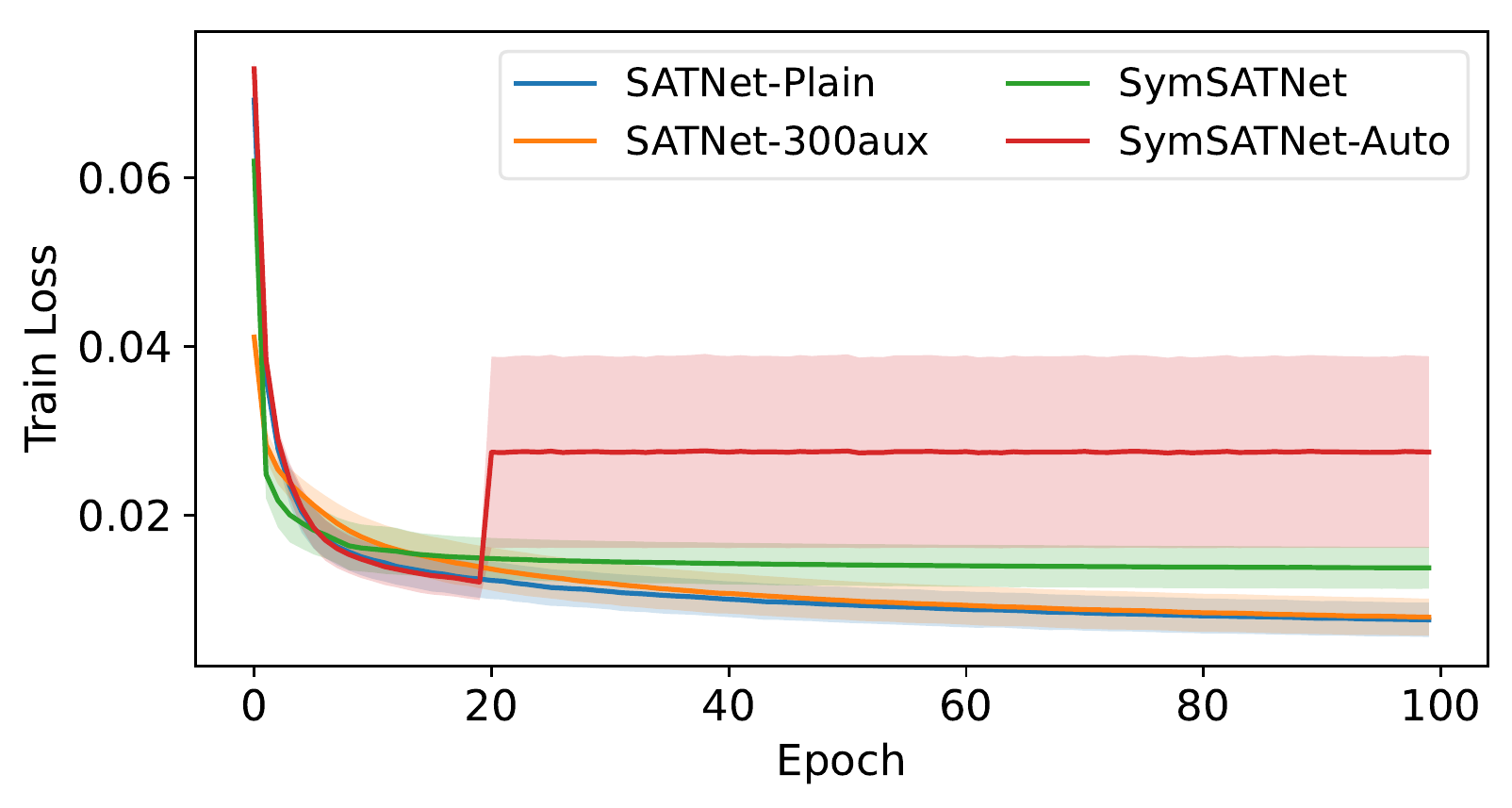}
        \subcaption{Training losses in Rubik's cube with 1 corruption}
        \label{fig:train-loss-cube-2}
    \end{subfigure}
    \begin{subfigure}{0.48\columnwidth}
        \includegraphics[width=0.95\linewidth]{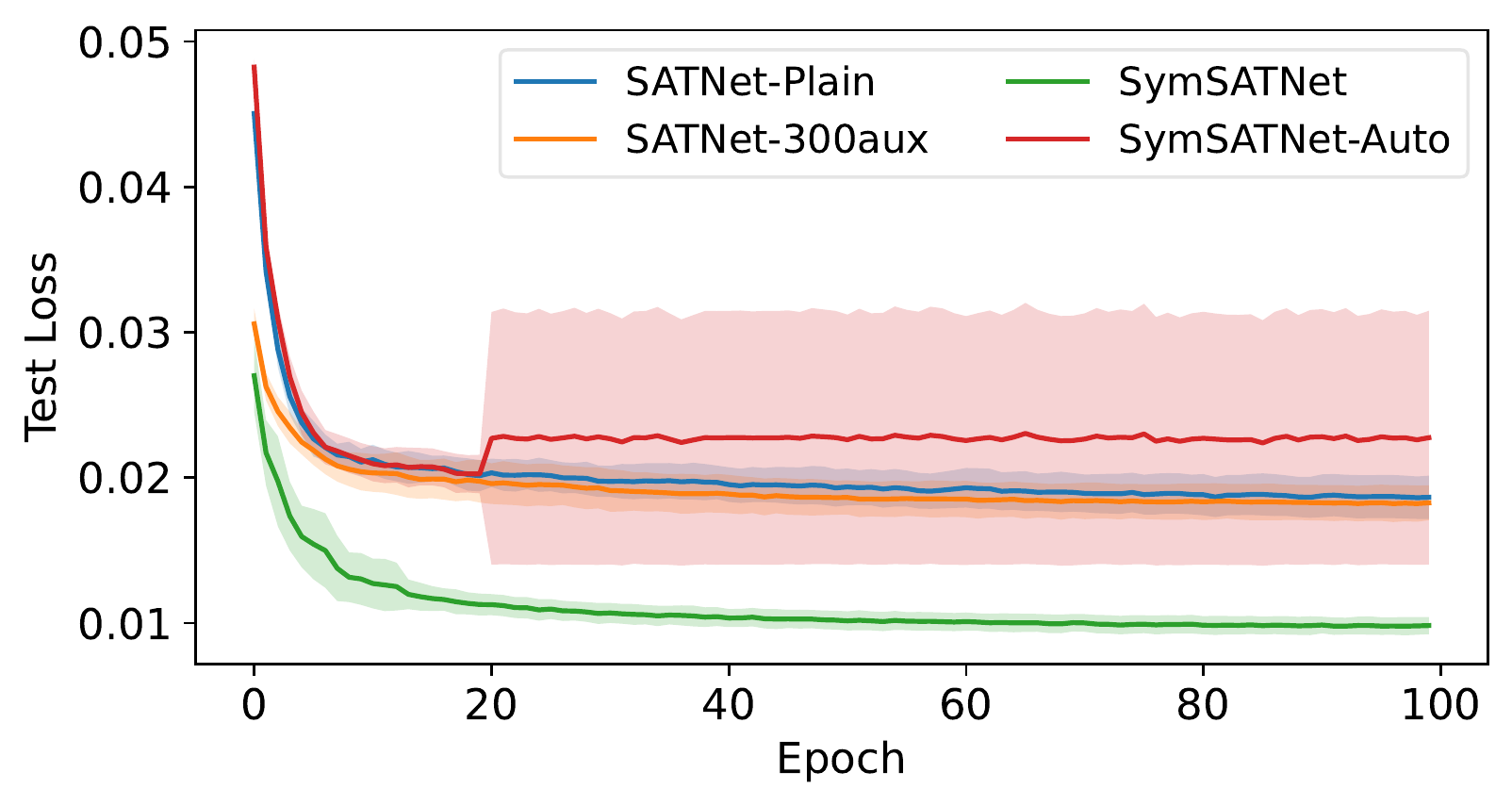}
        \subcaption{Test losses in Rubik's cube with 1 corruption}
        \label{fig:test-loss-cube-2}
    \end{subfigure}
    \begin{subfigure}{0.48\columnwidth}
        \includegraphics[width=0.95\linewidth]{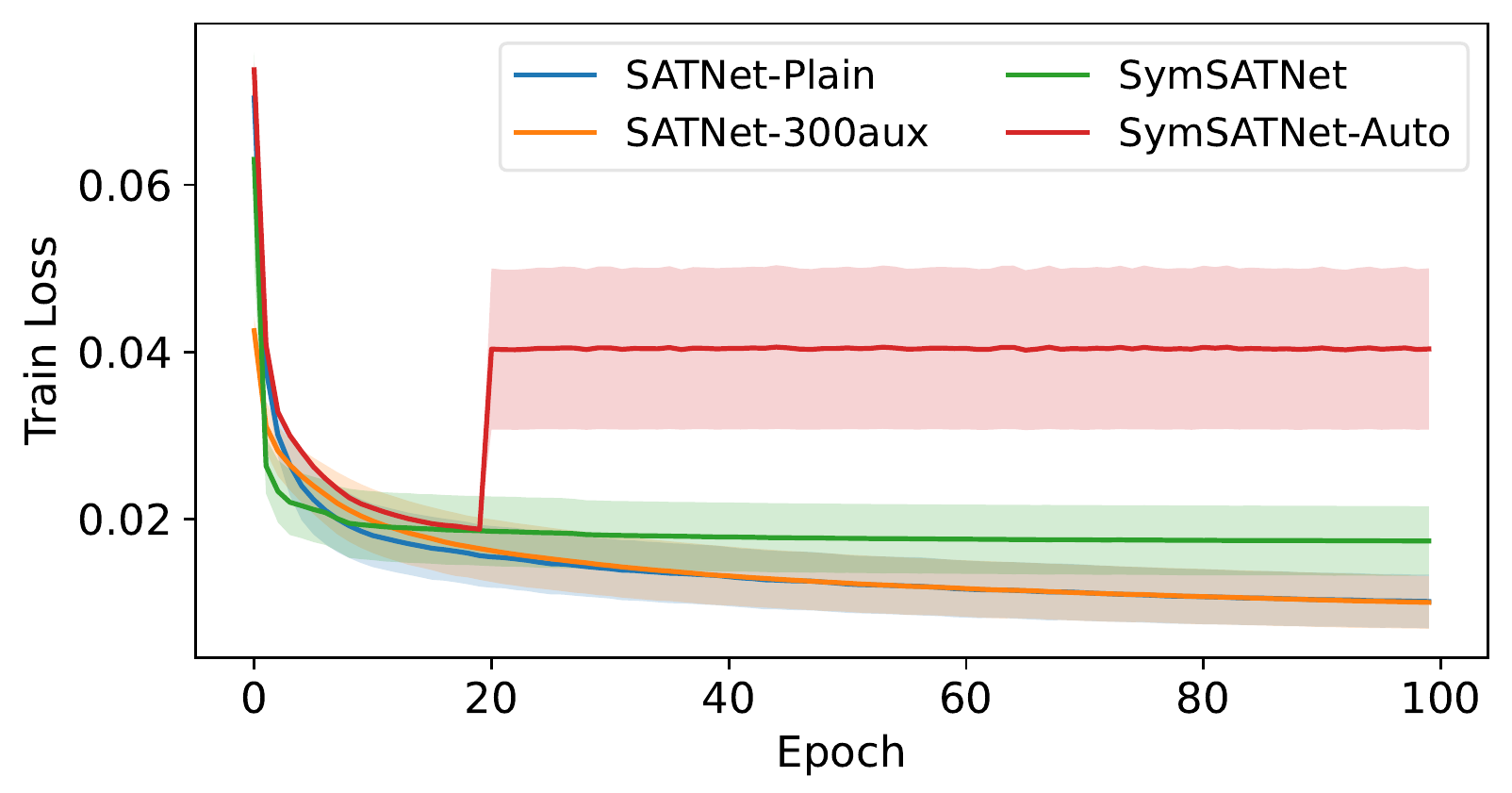}
        \subcaption{Training losses in Rubik's cube with 2 corruptions}
        \label{fig:train-loss-cube-3}
    \end{subfigure}
    \begin{subfigure}{0.48\columnwidth}
        \includegraphics[width=0.95\linewidth]{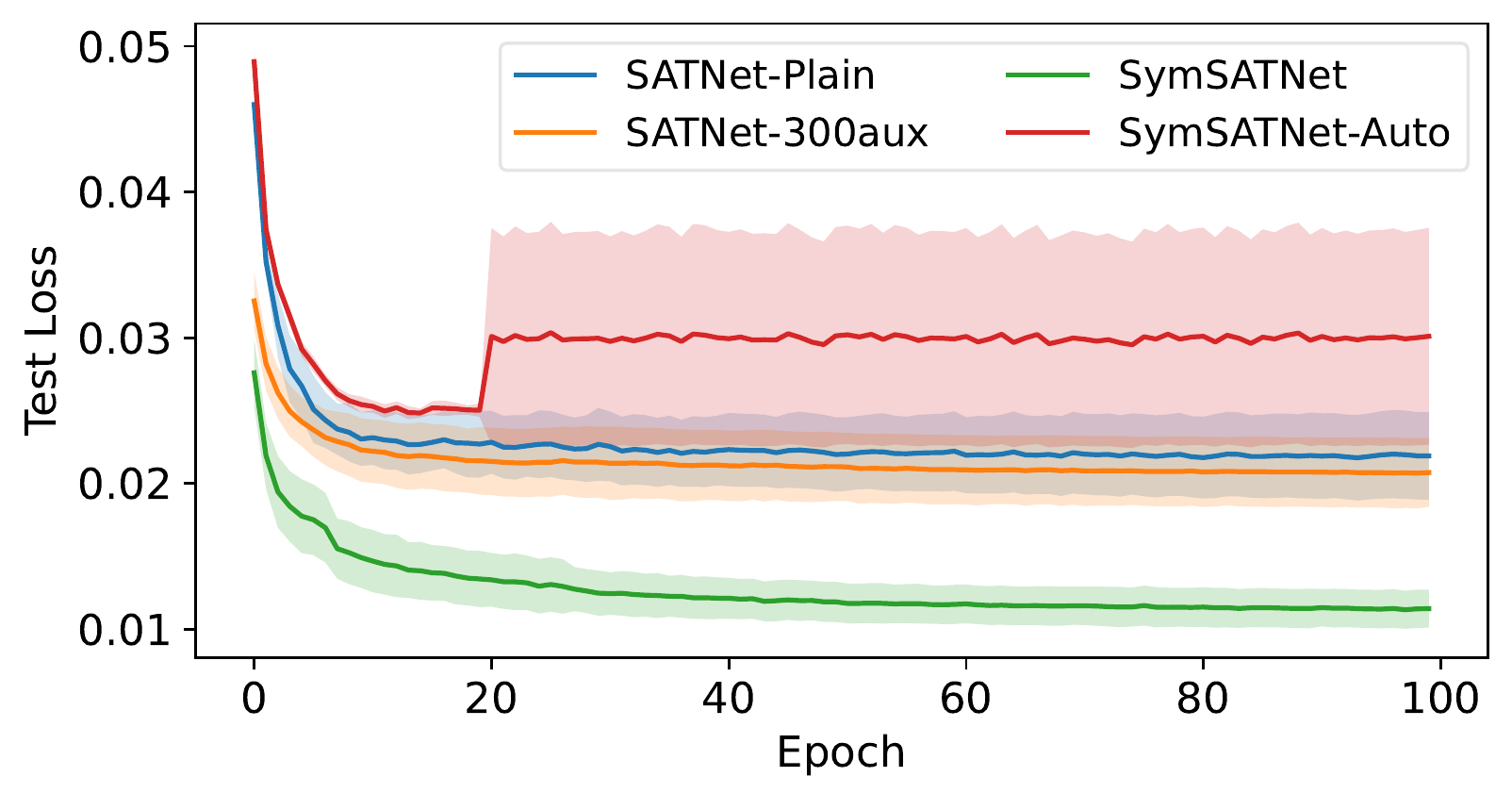}
        \subcaption{Test losses in Rubik's cube with 2 corruptions}
        \label{fig:test-loss-cube-3}
    \end{subfigure}
    \begin{subfigure}{0.48\columnwidth}
        \includegraphics[width=0.95\linewidth]{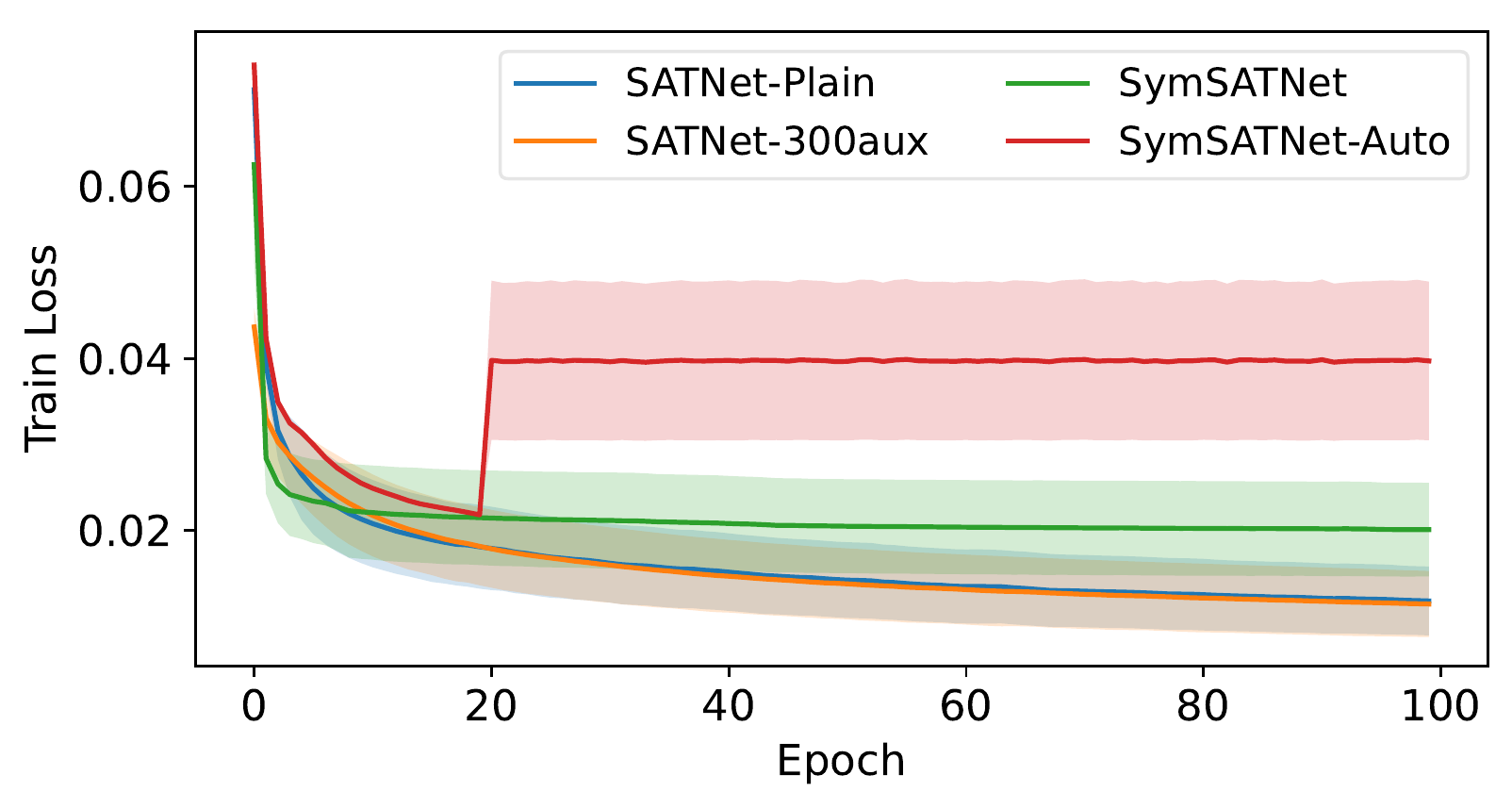}
        \subcaption{Training losses in Rubik's cube with 3 corruptions}
        \label{fig:train-loss-cube-4}
    \end{subfigure}
    \hskip 1.1em
    \begin{subfigure}{0.48\columnwidth}
        \includegraphics[width=0.95\linewidth]{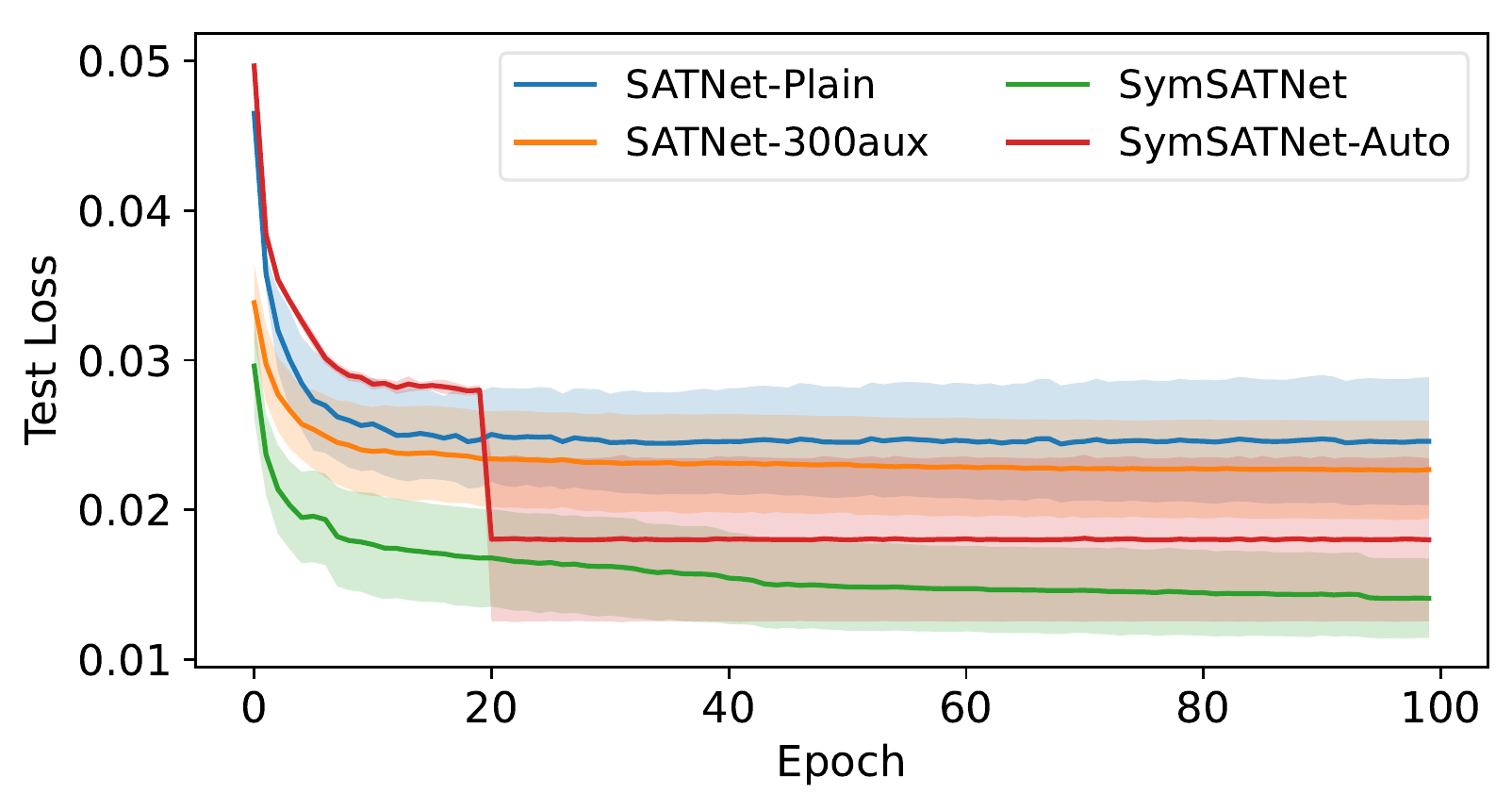}
        \subcaption{Test losses in Rubik's cube with 3 corruptions}
        \label{fig:test-loss-cube-4}
    \end{subfigure}
    \caption{Training and test loss curves for the Rubik's cube problem. Each loss is averaged over 10 trials and each $95\%$ confidence interval is included.}
    \label{fig:losses-cube}
\end{figure}